\documentclass{article}
\usepackage{subcaption}

\usepackage{PRIMEarxiv}

\usepackage[utf8]{inputenc} 
\usepackage[T1]{fontenc}    
\usepackage{amsfonts}       
\usepackage{nicefrac}       
\usepackage{microtype}      
\usepackage{lipsum}
\usepackage{fancyhdr}       
\usepackage{graphicx}       
\graphicspath{{media/}}     

\usepackage{graphicx}
\usepackage{stfloats}
\usepackage[round]{natbib}

\usepackage[colorlinks=true, linkcolor=red, citecolor=blue, urlcolor=green]{hyperref}      
\usepackage{booktabs}       
\usepackage{amsfonts}       
\usepackage{nicefrac}       
\usepackage{microtype}      
\usepackage{xcolor}         

\usepackage{graphicx}
\usepackage{amsmath}

\newtheorem{remark}{Remark}[section]
\usepackage{algorithmic}
\usepackage{float,psfrag,epsfig,color,xcolor,url}
\usepackage{diagbox}
\makeatother
\newtheorem{assumption}{Assumption}
\usepackage{algorithm}
\usepackage{bbm}
\usepackage{algorithmic}
\newtheorem{definition}{Definition}
\newtheorem{theorem}{Theorem}
\newtheorem{lemma}{Lemma}

\newtheorem{proof}{Proof}[section]
\usepackage{float}
\title{Improved Rates of Differentially Private Nonconvex-Strongly-Concave Minimax Optimization
}

\author{
  Ruijia Zhang$^*$  \\
  Johns Hopkins University \\
 {rzhan127@jh.edu} \\
   \And
  Mingxi Lei$^*$  \quad \quad \quad \quad \quad Meng Ding \\
  State University of New York at Buffalo \\
\texttt{\{mingxile,mengding\}}@buffalo.edu\\
  \And
  Zihang Xiang \\
  Provable Responsible AI and Data Analytics Lab\\ KAUST \\
  \texttt{zihang.xiang@kaust.edu.sa
} \\
\And
  Jinhui Xu \\
  State University of New York at Buffalo \\
\texttt{jinhui}@buffalo.edu\\
\And
  Di Wang \\
  Provable Responsible AI and Data Analytics Lab\\ KAUST \\
  \texttt{di.wang@kaust.edu.sa
}
}

\begin{document}

\maketitle
\def\thefootnote{*}\footnotetext{These authors contributed equally to this work. The work was done during Ruijia Zhang's internship at KAUST.}
\begin{abstract}
In this paper, we study the problem of (finite sum) minimax optimization in the Differential Privacy (DP) model. Unlike most of the previous studies on the (strongly) convex-concave settings or loss functions satisfying the Polyak-\L{ojasiewicz} condition, here we mainly focus on the nonconvex-strongly-concave one, which encapsulates many models in deep learning such as deep AUC maximization. Specifically, we first analyze a DP version of Stochastic Gradient Descent Ascent (SGDA) and show that it is possible to get a DP estimator whose $l_2$-norm of the gradient for the empirical risk function is upper bounded by $\tilde{O}(\frac{d^{1/4}}{({n\epsilon})^{1/2}})$, where $d$ is the model dimension and $n$ is the sample size. We then propose a new method with less gradient noise variance and improve the upper bound to $\tilde{O}(\frac{d^{1/3}}{(n\epsilon)^{2/3}})$, which matches the best-known result for DP Empirical Risk Minimization with non-convex loss. We also discussed several lower bounds of private minimax optimization. Finally, experiments on AUC maximization, generative adversarial networks, and temporal difference learning with real-world data support our theoretical analysis. 
\end{abstract}

\section{Introduction}
In recent years, minimax optimization has received great attention as it encompasses several basic machine learning and deep learning models such as generative adversarial networks (GANs) \citep{goodfellow2014generative,creswell2018generative}, deep AUC maximization \citep{yang2022auc}, distributionally robust optimization \citep{levy2020large}, and reinforcement learning \citep{sutton1988learning}, which have been widely used in different applications such as biomedicine and healthcare \citep{ling2022age,chen2022generative}.  The wide applications of minimax optimization also present privacy challenges in this problem as they always involve data with sensitive information. Differential Privacy (DP), introduced by \cite{dwork2006calibrating}, has gained widespread recognition as a method for preserving privacy by adding a controlled amount of random noise to the data or query responses, thereby effectively concealing the details of any individual. 

Recently, DP (finite sum) minimax optimization has been widely studied (see the related work section 
for details). However, compared to DP Empirical Risk Minimization~\citep{wang2017differentially,wang2021convergence,wang2019sparse}, DP Minimax optimization is still in its early stages of development. Specifically, most of the previous work focuses on the case where the loss is either (strongly)-convex-(strongly)-concave \citep{yang2022differentially,zhang2022bring,boob2024optimal,bassily2023differentially,gonzalez2024mirror,zhou2024differentially} or non-convex but satisfying the Polyak-\L{ojasiewicz} (PL) condition \citep{yang2022differentially}. However, compared to these settings, non-convex minimax optimization is more widespread in deep neural networks, and all these methods are based on stability analysis and are hard to extend to the non-convex minimax problem. Thus, there is still lacking understanding when the loss is nonconvex, which motivates the study in this paper. 

Recently, \cite{zhao2023differentially} presented the first study on DP temporal difference learning, which can be formalized as a specific nonconvex-strong-concave minimax problem. However, several challenges remain: First, compared to the utility metrics of DP Empirical Risk Minimization, which always use first order or second order gradient of the objective function~\citep{wang2019differentially,wang2019differentially12,wang2021escaping}, the metric in \cite{zhao2023differentially} cannot directly measure the stationariness of a model in general, indicating that it is hard to be explained whether the private model is good or not. Moreover,  their utility metric has not been widely used in other related work for both minimax optimization and reinforcement learning, making it hard to compare with the non-private case and hard to use in general minimax optimization problems (see Theorem 5.2 in \cite{zhao2023differentially} for details). Second, although in the ideal case \cite{zhao2023differentially} shows that their utility will be close to the $l_2$-norm gradient of the objective function, they show a utility bound of $\tilde{O}(\frac{d^{1/8}}{(n\epsilon)^{1/4}})$, where $d=\max \{d_1, d_2\}$ with $d_1$ and $d_2$ are model dimensions and $n$ is the sample size. It still has a gap with the best-known result $\tilde{O}(\frac{d^{1/3}}{(n\epsilon)^{2/3}})$ for DP Empirical Risk Minimization with non-convex loss~\citep{murata2023diff2,tran2022momentum}. Finally, their approach is only tailored for temporal difference learning, and it is unknown whether it can be extended to general minimax problems. 

To address the aforementioned issues, this paper revisits the DP minimax optimization problem in the nonconvex-strong-concave (NC-SC) setting, offering a more general and enhanced analysis. Our contributions can be summarized as follows:
\begin{enumerate}
    \item When the loss function is Lipschitz and smooth, we first show that by modifying the classical Stochastic Gradient Descent Ascent (SGDA) algorithm, it is possible to get an $(\epsilon,\delta)$-DP model whose $l_2$-norm of the gradient for the empirical risk function is upper bounded by $\tilde{O}(\frac{d^{1/4}}{{(n\epsilon)^{1/2}}})$. 

     \item The primary weakness of DP-SGDA is that it relies on using noise of the same scale to ensure differential privacy, which results in excessive variance and an unsatisfactory utility bound. To address this issue, we leverage the gradient difference between the current and previous models to adjust the noise scale. This approach allows us to add less noise as the iterations progress since the gradient difference tends to diminish. Specifically, we propose a novel method called PrivateDiff Minimax and demonstrate that its output can achieve an upper bound of $\tilde{O}(\frac{d^{1/3}}{(n\epsilon)^{2/3}})$, which matches the best-known result for DP Empirical Risk Minimization with non-convex loss.

     \item We also provide a preliminary study on the lower bounds of private minimax optimization. Specifically, for finite sum minimax problems, we show that there exists an instance such that for any $(\epsilon, \delta)$-DP model, its $l_2$-norm gradient is lower bounded by $\Omega(\frac{\sqrt{d}}{n\epsilon})$. Moreover, for the group distributional robust optimization problem, its utility is lower bounded by $\Omega(\frac{d\sqrt{d}}{n\epsilon})$. 
     
     \item Finally, we conduct experiments on AUC maximization, generative
adversarial networks,  and temporal difference learning with real-world data. Our results demonstrate that our method, PrivateDiff Minimax, outperforms other approaches across various datasets and privacy budgets, providing empirical support for our theoretical analysis.
\end{enumerate}

\section{Related Work}\label{sec:related}
\paragraph{DP Minimax Optimization.} \cite{yang2022differentially} provides the first study on DP stochastic minimax optimization. Specifically, for the convex-(strongly)-concave case, they provide upper bounds in terms of weak primal-dual population risk, which match the optimal rates for DP Stochastic Convex Optimization~\citep{su2024faster,su2023differentially,hu2022high,huai2020pairwise,wang2020differentially,DBLP:conf/ijcai/XueYH021,DBLP:conf/ijcai/TaoW0W22}. They further consider the NC-SC case where the loss satisfies the PL condition. However, as their analysis is based on algorithmic stability, it is difficult to extend to general  NC-SC loss, which is studied in this paper.  \cite{zhang2022bring} also studies the convex-(strongly)-concave case and provides a linear-time algorithm, which can also achieve optimal rates.  \cite{boob2024optimal} considers both convex-concave minimax optimization and stochastic variational inequality, it provides both strong and weak primal-dual population risks. Recently, \cite{bassily2023differentially} justifies that the (strong) primal-dual gap is a more meaningful and challenging efficiency estimate for DP convex-concave 
 minimax optimization. Very recently, \cite{gonzalez2024mirror} considers the convex-concave case where the constrain sets are polyhedral; it provides utility bounds that are independent of the polynomial of the model dimension. \cite{zhou2024differentially} considers the DP worst-group risk minimization with convex loss, which is a specific instance of minimax optimization, and provides both upper and lower bounds of the problem. 

To the best of our knowledge, \cite{zhao2023differentially} is the only paper that studies the general NC-SC case of stochastic minimax optimization.  However, as mentioned previously, their utility has been only used in reinforcement learning rather than in other minimax optimization problems. In our paper, we consider the gradient norm as the utility, which is more natural and has been widely used in both non-private studies and the DP nonconvex case \citep{wang2019differentially,xiao2023theory,wang2019differentially12,wang2023efficient,murata2023diff2,tran2022momentum}.  

\paragraph{Nonconvex Minimax Optimization.} As there is a long list of work on minimax optimization, here we only focus on the ones that consider the NC-SC setting. Previous work mainly focuses on improving the gradient complexity or number of loops  \citep{nouiehed2019solving,lin2020gradient,lin2020near,lu2020hybrid,zhang2022sapd+,boct2023alternating,sharma2022federated,guo2021novel,yan2020optimal,xu2023unified,luo2020stochastic}. For example, \cite{lin2020gradient} shows the local convergence of SGDA w.r.t. the gradient norm if the stepsizes are chosen appropriately, which motivates our first algorithm DP-SGDA. \cite{luo2020stochastic} provides a variance reduction-based approach to accelerate SGDA further.  It is notable that our second method is quite different from all these non-private methods.  Specifically, our approach is still based on SGDA. However, we use the gradient difference between the current and previous models to reduce the variance of added noise. This makes us add less noise as the iteration increases since the gradient difference tends to be zero.
 Thus, even from the optimization point of view, our method is still of interest.

\section{Preliminaries}
\subsection{Differential Privacy}
	\begin{definition}[Differential Privacy \citep{dwork2006calibrating}]\label{def:3.1}
	Given a data universe $\mathcal{Z}$, we say that two datasets $D,D'\subseteq \mathcal{Z}$ are neighbors if they differ by only one entry, which is denoted as $D\sim D'$. A randomized algorithm $\mathcal{A}$ is $(\epsilon,\delta)$-differentially private (DP) if for all neighboring datasets $D,D'$ and for all events $E$ in the output space of $\mathcal{A}$, the following holds
	$$\mathbb{P}(\mathcal{A}(D)\in E)\leq e^{\epsilon} \mathbb{P}(\mathcal{A}(D')\in E)+\delta.$$ 
 If $\delta=0$, we call algorithm $\mathcal{A}$ is $\epsilon$-DP. 
\end{definition} 


In this paper, we focus on $(\epsilon, \delta)$-DP and mainly use the Gaussian mechanism and moment accountant~\citep{abadi2016deep} to guarantee the DP property. 

\begin{definition}[$l_2$-sensitivity]
    Given a function $q: \mathcal{Z} \to \mathbb{R}^d$, we say $q$ has $\Delta_2(q)$ $l_2$-sensitivity if for any neighboring datasets $D, D^{\prime}$ we have $
\|q(D)-q(D^{\prime})\|_2 \leq \Delta_2(q).$
\end{definition}
\begin{definition}[Gaussian Mechanism]
	Given any function $q: \mathcal{Z} \rightarrow \mathbb{R}^d$, the Gaussian mechanism is defined as  $q(D)+\xi$ where $\xi\sim \mathcal{N}(0,\frac{8\Delta^2_2(q)\log(1.25/\delta)}{\epsilon^2}\mathbb{I}_d)$,  Gaussian mechanism preserves $(\epsilon, \delta)$-DP for $0<\epsilon, \delta\leq 1$.
 \end{definition} 

\begin{definition}
    For an (randomized) algorithm $\mathcal{A}$ and neighboring datasets $D, D^{\prime}$, the $\lambda$-th moment is given as
$$
\alpha_{\mathcal{A}}(\lambda, D, D^{\prime})=\log \mathbb{E}_{O \sim \mathcal{A}(D)}[(\frac{\mathbb{P}[\mathcal{A}(D)=O]}{\mathbb{P}[\mathcal{A}(D^{\prime})=O]})^\lambda] .
$$
The moment accountant is then defined as
$$
\alpha_\mathcal{A}(\lambda)=\sup _{D, D^{\prime}} \alpha_\mathcal{A}(\lambda, D, D^{\prime}) .
$$
\end{definition}

\begin{lemma}\citep{abadi2016deep}
\label{Lemma Abadi}
 Consider a sequence of mechanisms $\{\mathcal{A}_t\}_{t \in[T]}$ and the composite mechanism $\mathcal{A}=(\mathcal{A}_1, \cdots, \mathcal{A}_T)$. We have the following properties: \\
(a) [Composability] For any $\lambda$,
$$
\alpha_{\mathcal{A}}(\lambda)=\sum_{t=1}^T \alpha_{\mathcal{A}_{t}}(\lambda). 
$$\\
(b)\label{Tail Bound} [Tail bound] For any $\epsilon$, the mechanism $\mathcal{A}$ is $(\epsilon, \delta)$ differentially private for
$$
\delta=\min _\lambda \alpha_{\mathcal{A}}(\lambda)-\lambda \epsilon. 
$$
\end{lemma}



\begin{lemma}[Privacy Amplification via Subsampling \citep{balle2018privacy}]
\label{amplification}
    Consider a sequence of mechanisms $\mathcal{A}_t=q_t(D_t)+\xi_t$ where $\xi_t \sim \mathcal{N}(0, \sigma^2 \mathbb{I})$. Here each function $q_t: \mathcal{Z} \rightarrow \mathbb{R}^d$ has $l_2$-sensitivity of 1. And each $D_t$ is a subsample of size $m$ obtained by uniform sampling without replacement from space $\mathcal{Z}$, i.e. $D_t \sim(\operatorname{Unif}(D))^m$. Then we have 
$$
\alpha_{\mathcal{A}_t}(\lambda) \leq \frac{m^2 n \lambda(\lambda+1)}{n^2(n-m) \sigma^2}+\mathcal{O}(\frac{m^3 \lambda^3}{n^3 \sigma^3}). 
$$
\end{lemma}

\subsection{Minimax Optimization}
Given a dataset $D=\{z_1, \cdots, z_n\}\in \mathcal{Z}^n$ and a loss function $f: \mathcal{X}\times \mathcal{Y} \times \mathcal{Z}\mapsto \mathbb{R}$, a (finite sum) minimax optimization problem aims to optimize the following empirical risk function: 
\begin{equation}
\label{emperical minimax}
    \min _{x \in \mathcal{X}} \max _{y \in \mathcal{Y}} \hat{L}(x, y;D):=\frac{1}{n} \sum_{i=1}^n f(x, y ; z_i), 
\end{equation}  
 where $\mathcal{X}$ and $\mathcal{Y}$ are the constrained sets.  If each $z_i$ is i.i.d. sampled from some underlying distribution $\mathcal{Z}$, then we further aim to optimize the population risk: 
\begin{equation}
\label{Population minimax}
\min_{x \in \mathcal{X}} \max _{y \in \mathcal{Y}} L(x,y; D):=\mathbb{E}_{\mathcal{Z}}[L(x,y ;z)].
\end{equation}
In this paper, we mainly focus on the empirical risk function. 

Recall that the minimax problem \eqref{emperical minimax} is equivalent to minimizing the function $\Phi(\cdot)=\max _{y \in \mathcal{Y}} \hat{L}(\cdot, y)$. For nonconvex strongly concave minimax problems in which $\hat{L}(x, \cdot)$ is strongly concave in $y$ for each $x \in \mathcal{X}$, the maximization problem $\max _{y \in \mathcal{Y}} \hat{L}(x, y)$ can be solved efficiently and provides useful information about $\Phi$. However, it is  NP-hard to find the global minimum of $\Phi$ in general when $\Phi$ is nonconvex, which is considered in our paper. In this work, we hope to find an approximate first-order stationary point instead, which has been widely adopted in previous literature \citep{lin2020gradient}. 

\begin{definition}
\label{stationary}
    A point $\mathrm{x}$ is an $\epsilon$-stationary point $(\epsilon \geq 0)$ of a differentiable function $\Phi$ if $\|\nabla \Phi(x)\| \leq \epsilon$. If $\epsilon=0$, then $x$ is a stationary point.
\end{definition}

Note that there are also other metrics for stationary points \citep{lu2020hybrid,nouiehed2019solving}; however, these notions are weaker than $\|\nabla \Phi(\cdot)\|$. From the above definitions, it is clear that DP minimax optimization aims to develop an $(\epsilon, \delta)$-DP algorithm whose output $(x^{\text{priv}}, y^{\text{priv}})$ makes $\|\nabla \Phi(x^{\text{priv}})\|$ be as small as possible. In this paper, we focus on the nonconvex-strongly-convex (NC-SC) setting and we impose the following assumptions. 

\begin{definition}
    A function $g$ is $G$-Lipschitz if for $\forall x, x^{\prime} \in \mathcal{X}$, we have $\|g(x)-g(x^{\prime})\| \leq G\|x-x^{\prime}\|$.
\end{definition}

\begin{definition}
    A function $g$ is $l$-smooth if for $\forall x, x^{\prime}  \in \mathcal{X}$, we have $\|\nabla g(x)-\nabla g(x^{\prime})\| \leq l\|x-x^{\prime}\|$.
\end{definition}

\begin{definition}
    A function $g$ is $\mu$-strongly convex if for $\forall x, x^{\prime} \in \mathcal{X}$, we have $\langle \nabla g(x)-\nabla g(x'), x-x'\rangle \geq \mu \|x-x'\|_2^2$. A function $g$ is $\mu$-strongly concave if $-g$ is $\mu$-strongly convex. 
\end{definition}

\begin{assumption}
\label{NCSC} For any fixed $x\in \mathcal{X}$, $\hat{L}(x, \cdot; D)$ is $\mu$-strongly concave in $y$. Moreover, we assume $\mathcal{X}=\mathbb{R}^{d_1}$ and $\mathcal{Y} \subseteq \mathbb{R}^{d_2}$ is a convex and bounded set with diameter $\Lambda$ (we denote $d=\max\{d_1, d_2\}$). We also assume $f(\cdot, \cdot; z_i)\leq M$. 
\end{assumption}

\begin{assumption}
\label{lipschitz cts}
     There exist $G_{x}, G_{y}$ such that, for any $x \in \mathcal{X}, y\in \mathcal{Y}$, 
     function $f(\cdot, y ; z_i)$ is $G_{x}$-Lipschitz and  
     function $f( x, \cdot ; z_i)$ is $G_{y}$-Lipschitz. 
Denote $G=\max \{G_{\mathbf{w}}, G_{\mathbf{v}}\}$.
\end{assumption}

\begin{assumption}
\label{smooth}
  There exists a constant  $l_x$ and $l_y$ such that for any $x \in \mathcal{X}, y\in \mathcal{Y}$,   function $\hat{L}(\cdot, y ; D)$ is $l_{x}$-smooth and  
     function $\hat{L}( x, \cdot ; D)$ is $l_{y}$-smooth.  
  Denote $l=\max \{l_{x}, l_{y}\}$.
\end{assumption}

\begin{assumption}\label{stochastic gradient variance} 
    For randomly drawn $j\in [n]$, the gradients $ \nabla_{x} f(x, y; z_j)$ and $ \nabla_{y} f(x, y; z_j)$ have
bounded variances $B_x$ and $B_y$ respectively. Let $\mathcal{B}=\max\{B_x, B_y\}$. 
\end{assumption}

We present a technical lemma on the structure of function $\Phi$, which is essential for the convergence analysis.
\begin{lemma}[\cite{lin2020gradient}]\label{lemma:y*}
    Under Assumption \ref{NCSC} and \ref{smooth}, $\Phi(\cdot)=\max _{y \in \mathcal{Y}} \hat{L}(\cdot, y)$ is $(l+\kappa l)$-smooth, where $\kappa=\frac{l}{\mu}$ is the condition number. Moreover, for any $x\in \mathcal{X}$, $\nabla \Phi(x)=\nabla_{x} \hat{L}(x, y^{\star}(x))$, where $y^{\star}(x)=\arg\max_{y\in \mathcal{Y}} \hat{L}(x, y)$ and $y^{\star}(\cdot)$  is $\kappa$-Lipschitz.
\end{lemma}

\section{A Preliminary Exploration}
\subsection{An Upper Bound via DP-SGDA}
In the non-private case, a natural approach to solving the minimax problem is the gradient descent ascent (GDA). However, when privacy is a concern, directly applying GDA can lead to significant privacy risks. To address this, we explore a differentially private version of stochastic GDA (DP-SGDA) in this section, providing a preliminary analysis of our problem while ensuring privacy is maintained.

DP-SGDA (Algorithm \ref{alg:DP-SGDA}) was proposed by \cite{yang2022differentially}. Their analysis relies on the algorithm's stability within the convex-concave setting, which can not extend to our nonconvex-strongly-concave (NC-SC) case. In the following, we provide a more general analysis tailored for the NC-SC setting.

\begin{algorithm}[htb]
\caption{ Differentially Private Stochastic Gradient Descent Ascent (DP-SGDA)}
\label{alg:DP-SGDA}
\begin{algorithmic}[1] 
\REQUIRE Dataset $D$,  privacy budget $\epsilon, \delta$, iteration number $T$, learning rates $\{\eta_{x}, \eta_{y}\}$, initialization $(x_0, y_0)$, clipping thresholds $C_1, C_2$. 
   \FOR{ $t=0,1, \cdots, T$}
            \STATE Draw a collection of i.i.d. data samples $\{z_{t}^j\}_{j=1}^m$ uniformly without replacement.
            \STATE Sample independent noises $\xi_t \sim \mathcal{N}(0, \sigma_{x}^2 I_{d_1})$ and $\zeta_t \sim \mathcal{N}(0, \sigma_{y}^2 I_{d_2})$.   
            \STATE Update $x_{t+1}$:\\
            ${x}_{t+1}={x}_t-\eta_{x}( \frac{1}{m} \sum_{j=1}^m \text{Clipping}(\nabla_{x} f({x}_t, {y}_t ; {z}_{t}^j), C_1)$\\
            $\quad\quad\quad+\xi_t)$. \\
            \STATE Update $y_{t+1}$:\\
            ${y}_{t+1}$$=\Pi_{\mathcal{Y}}({y}_t+\eta_{y}(\frac{1}{m} \sum_{j=1}^m \text{Clipping}(\nabla_{y} f({x}_t, {y}_t ; {z}_{t}^j),$\\
            $\quad\quad\quad C_2) +\zeta_t).$

    \ENDFOR
    \RETURN $(x^{\text{priv}},y^{\text{priv}})\in \{(x_0, y_0), \cdots, (x_T, y_T)\}$ where the tuple is uniformly
sampled. 
\end{algorithmic}
\end{algorithm}

\begin{algorithm}[!htb]
\caption{Clipping $(x, C)$}
\label{alg:ClippedMean}
\begin{algorithmic}[1] 
\REQUIRE $x$ and clipping threshold $C>0$. \\ 
\STATE  $\hat{x}=\min \{\frac{C}{\|x\|_2}, 1\} x $
    
\RETURN $\hat{x}.$ 
\end{algorithmic}
\end{algorithm}

\begin{theorem}
\label{DP SGDA privacy thm}
    There exist constants $c_1, c_2$ and $c_3>0$ such that given the mini-batch size $m$ and total iterations $T$, for any $\epsilon<c_1 m^2 T / n^2$ and $0<\delta<1$, Algorithm 1 is $(\epsilon, \delta)$-DP if we set
\begin{equation}
    \sigma_{x}=\frac{c_2 C_1 \sqrt{T \log (1 / \delta)}}{n \epsilon}, \sigma_{y}=\frac{c_3 C_2 \sqrt{T \log (1 / \delta)}}{n \epsilon}.
\end{equation}
\end{theorem}
In practice, a set of parameters applicable to Theorem \ref{DP SGDA privacy thm} is provided by \cite{yang2022differentially, abadi2016deep} to ensure the privacy guarantee. By setting $\epsilon \leq 1, \delta \leq 1 / n^2$ and $m=$ $\max (1, n \sqrt{\epsilon /(4 T)})$, the explicit values for the variances are given as $\sigma_{x}=\frac{8 \sqrt{T \log (1 / \delta)}}{n \epsilon}, \sigma_{y}=\frac{8\sqrt{T \log (1 / \delta)}}{n \epsilon}$.

Next, we show an improved utility bound of Algorithm \ref{alg:DP-SGDA}. 
\begin{theorem}
\label{thm: DP-SGDA utility}
    Suppose Assumptions \ref{NCSC}-\ref{stochastic gradient variance} hold. If we choose parameters satisfying: iterations $T =\Theta(\frac{n\epsilon}{\sqrt{d\log({1}/{\delta})}}),$ clipping thresholds $ C_1\geq G_{x}, C_2\geq G_{y}$, step sizes $ \eta_{x} =O(\frac{1}{l \kappa^2})$, $\eta_{y}=O(\frac{1}{l})$ and batch size $m=O(\frac{n\epsilon}{\sqrt{d\log({1}/{\delta})}})$,  then the output of DP-SGDA satisfies 
    \begin{equation}
        \mathbb{E}\|\nabla \Phi (x^{\text{priv}})\| \leq O \left(\frac{(d\log ({1}/{\delta}))^{1/4}}{\sqrt{n \epsilon}}\right), 
    \end{equation}
    where $O$ hides other terms related to $G, \ell, 
    \mathcal{B}, \mu$ and $\kappa$.
\end{theorem}
\noindent \textbf{Technical Overview} Although the idea of DP-SGDA is natural, our utility analysis is highly non-trivial. Specifically, our proof needs to set a pair of stepsizes $(\eta_{x}, \eta_{y})$, which updates $\{y_t\}_{t \geq 1}$ significantly faster than that of $\{x_t\}_{t \geq 1}$. Recall that $y^{\star}(\cdot)$ is $\kappa$-Lipschitz in Lemma \ref{lemma:y*}:
$$
\|y^{\star}(x_1)-y^{\star}(x_2)\| \leq \kappa\|x_1-x_2\| .
$$
Consequently, if $\{x_t\}_{t \geq 1}$ changes slowly, it follows that its corresponding sequence ${y^*(x_t)}$ also evolves gradually. Therefore, This allows us to perform gradient descent analysis on the strongly concave function $\hat{L}(x_t,\cdot; D)$, albeit it is changing slowly. Additionally, by defining the error as $\theta_t=\|y^{\star}(x_t)-y_t\|^2$, we can first apply the descent lemma to $\Phi(x)$. Then, by performing telescoping, we obtain the following inequality:
$$
 \mathbb{E}\Phi(x_{T+1})-\Phi(x_0) 
\leq -\Omega(\eta_{x})(\sum_{t=0}^T\mathbb{E} \|\nabla \Phi(x_t)\|^2)+O(\eta_x)+ O(\eta_x) [\sum_{t=0}^T \mathbb{E}\|\xi_t\|_2^2+\mathbb{E}\|\zeta_t\|_2^2] +O(\frac{T\eta_x}{m}). $$
Thus, $\sum_{t=0}^T\|\nabla \Phi(x_t)\|^2$ can be upper bounded by the last term on the right-hand side, which is the desired utility bound.
\begin{remark}
    Note that when there is no variable $y$, then DP-SGDA will be reduced to DP-SGD in \cite{wang2017differentially}. Moreover, the bound $\tilde{O}(\frac{\sqrt[4]{d}}{ \sqrt{n \epsilon}})$ aligns with the bounds provided in previous work on DP Empirical Risk Minimization with non-convex loss, such as \cite{wang2017differentially,wang2023efficient}. While \cite{yang2022differentially} considered DP-SGDA for non-convex loss under the PL condition, our approach differs in the choice of stepsize: we use a constant stepsize throughout all iterations, whereas \cite{yang2022differentially} requires the stepsize to decay with respect to the iteration number.
\end{remark}
\subsection{Lower Bounds of the DP-Minimax Problem}
We now show a lower bound $\Omega(\frac{d\log ({1}/{\delta})}{n^2\epsilon^2})$ for the utility under differential privacy in the case where $\mathcal{X}=\mathbb{R}^{d_1}$ and $\mathcal{Y}$ is a bounded convex set.  Our lower bound matches the current best-known lower bound for DP-ERM with non-convex loss \citep{arora2023faster} and holds even for convex functions. 
\begin{theorem}\label{thm_low:1}
    Given $n, \epsilon=O(1)$, $2^{-\Omega(n)}\leq \delta\leq 1/n^{1+\Omega(1)}$, there exists a convex set $\mathcal{Y}\subseteq {R}^{d_2}$, a loss function $\hat{L}: \mathbb{R}^{d_1} \times \mathcal{Y} \times \mathcal{Z}^{n} \mapsto \mathbb{R}$ satisfying Assumption \ref{NCSC}-\ref{smooth} with $\mu, G, l=O(1)$ and a dataset $D$ of $n$ samples such as for any $(\epsilon, \delta)$-DP algorithm with output $(x^{\text{priv}},y^{\text{priv}})$ satisfies 
    \begin{equation*}
        \|\nabla \Phi(x^{\text{priv}})\| \geq\Omega (\min\{1, \frac{\sqrt{d\log({1}/{\delta})}}{n\epsilon }\}). 
    \end{equation*}
\end{theorem}
It is notable that this result implies the same lower bound (up to logarithmic factors) for the population gradient using the technique in \cite{bassily2019private}. Furthermore, the aforementioned lower bound applies specifically to minimax problems in finite-sum form, as described in \eqref{emperical minimax}. However, different lower bounds may be derived for specific problems that cannot be expressed in this form. For instance, consider the (regularized) worst-group risk minimization problem:
\begin{equation}\label{worst_ERM}
     \min _{x \in \mathbb{R}^{d_2}} \max _{y \in \Delta_{d_2}} \hat{L}(x, y;D)= \sum_{i=1}^{d_2} y_i \hat{L}_{D_i}(x)-\frac{1}{2}\|y\|_2^2,
\end{equation}
where $\Delta_{d_2}=\{y\in [0, 1]^{d_2}: \|y\|_1=1\}$, $D=\bigcup D_i$, $D_i\bigcap D_j=\emptyset$ for $i\neq j$, and $\hat{L}_{D_i}(x)=\frac{1}{|D_i|}\sum_{z\in D_i} \hat{L}(x; z)$.
\begin{theorem}\label{thm_low:2}
       Given $n, \epsilon=O(1)$, $2^{-\Omega(n)}\leq \delta\leq 1/n^{1+\Omega(1)}$, there exists a convex set $\mathcal{Y}\subseteq {R}^{d_2}$, a Lipschitz and smooth loss function $\hat{L}: \mathbb{R}^{d_1} \times \mathcal{Z} \mapsto \mathbb{R}$ and a dataset $S$ of $n$ samples such as for any $(\epsilon, \delta)$-DP algorithm with output $(x^{\text{priv}},y^{\text{priv}})$ satisfies  
       \begin{equation*}
           \|\nabla \Phi(x^{\text{priv}})\| \geq \Omega (\min\{1, \frac{ d\sqrt{d\log ({1}/{\delta})} }{n\epsilon}\} ). 
       \end{equation*}
\end{theorem}

\begin{algorithm*}[h]
\caption{PrivateDiff Minimax }
\label{alg:Framwork}
\begin{algorithmic}[1] 
\REQUIRE  
   Initial Point $x_0, \tilde{y}_0$, dateset $D$, learning rates $\eta_x$,  noise variance $ \sigma_{x_{1}}^2, \sigma_{x_{2}}^2, \sigma_y^2$, clipping radius $C_0, C_1, C_2$ and $C_3$, iteration number $T_1, T_2$ and $R$, batch size $m$.

   \FOR{ $r=0,1, 2, 3 \ldots, R$}
   \STATE Draw a collection of i.i.d. data samples $\{z_{r}^j\}_{j=1}^m$ uniformly without replacement.
    \STATE $y_r=\tilde{y}_r$
            \STATE $y_{r+1}=$ Mini-batch SGA($\hat{L}(x_r,y_r;D), T_2, C_0$) 
            \STATE \text {if }r \% T=0 \text { then } 
            \STATE $\quad \mathbf{d_r}$= $\frac{1}{m} \sum_{j=1}^m \text{Clipping}(\nabla_{x} f(x_{r}, y_{r+1}; z_{r}^j),C_1)$. \\
            $\quad$Set $\sigma_x=\sigma_{x_1}, C=C_1$ and $\widetilde{v}_{r}=0$; \\

            \STATE else:\\
            $\quad \mathbf{d_r}=$ $\frac{1}{m} \sum_{j=1}^m $Clipping $(  \nabla_{x} f(x_{r}, y_{r+1}; z_{r}^j)-\nabla_{x} f(x_{r-1}, y_{r}; z_{r-1}^j),C_{2, r})$\\
            $\quad$Set $\sigma_x=\sigma_{x_2} \text{and } C=C_{2, r}=C_2\|x_{r}-x_{r-1}\|+ C_3$.\\
            end if

            

            

            \STATE Set $v_{r+1}=\mathbf{d_r}+\widetilde{v}_{r}$ and $\widetilde{v}_{r+1}=v_{r+1}+\xi_{x_{r+1}}$, where $\xi_{x_{r+1}} \sim N(0, \sigma_{x}^2 C^2 I_{d_1})$.
            
            \STATE $x_{r+1}=x_{r}-\eta_x \widetilde{v}_{r+1}$.
            \STATE $\tilde{y}_{r+1}=y_{r+1}+\zeta$, where $\zeta \sim \mathcal{N}(0, \sigma_y^2  I_{d_2})$.
       
    \ENDFOR
    
\RETURN $(x^{\text{priv}},y^{\text{priv}})\in \{(x_1, \tilde{y}_1), \cdots, (x_R, \tilde{y}_R)\}$ where the tuple is uniformly
sampled. 
\end{algorithmic}
\end{algorithm*}

\begin{algorithm}[htb]
    \caption{Mini-batch Stochastic Gradient Ascent (Mini-batch SGA)}
    \label{alg:DPGDSC}
    \begin{algorithmic}[1]
 \REQUIRE ~~ 
Fixed $x$, step size $\eta_{y_i}$, initial point $y^\prime_0=y$,  number of iterations $T_{2}$, clipping threshold $C_0$.\\
\FOR{ $i=0,1, 2, 3 \ldots, T_{2}$}
\STATE Draw a collection of i.i.d. data samples $\{z_{i}^j\}_{j=1}^m$ uniformly without replacement.
\STATE Update $y^{\prime}_{i+1}$ as   $ y^{\prime}_{i+1}=\Pi_{\mathcal{Y}}(y^{\prime}_{i}+\frac{\eta_{y_i}}{m} \sum_{j=1}^m \text{Clipping}(\nabla_{y} f({x}, {y}^\prime_{i}; z_{i}^j), C_0))$. 
\ENDFOR
\STATE Return $y^{\prime}_{T_2}$.  
    \end{algorithmic}
\end{algorithm}

\section{Improved Rate via PrivateDiff Minimax }
As discussed in the previous section, there remains a gap of $\widetilde{O}(\frac{d^{1/4}}{\sqrt{n \epsilon}})$ between the upper and lower bounds. In this section, we aim to bridge this gap. Specifically, our goal is to develop a method that achieves a rate of $\widetilde{O}(\frac{d^{1 / 3}}{(n \epsilon )^{2 / 3}})$.

Our key observation is that the utility of Algorithm \ref{alg:DP-SGDA} heavily depends on the noise variance we add in each iteration. Notably, the scale of its noise variance is proportional to the $l_2$-norm sensitivity of the gradient, which is upper bounded by the smoothness constant of the function. Thus, by using the composition theorem, adding the same scale of noise in each iteration in Algorithm \ref{alg:DP-SGDA} can guarantee DP. From the utility side, this is fine for variable $y$ as $\hat{L}(x, \cdot; D)$ is strongly concave, and it is known that DP-SGD with the same scale of noise in each iteration can achieve the optimal rate~\citep{bassily2019private}. However, such a strategy is only sub-optimal for variable $x$, which corresponds to a nonconvex loss $\hat{L}(\cdot, y; D)$. As a result, we propose an algorithm called PrivateDiff Minimax, which focuses on improving the performance for $x$. 

\noindent \textbf{Main Idea:} In essence, PrivateDiff Minimax updates variable $y$ and variable $x$ alternatively within each iteration.  Suppose in the $r$-th iteration, we have $(x_r, \tilde{y}_{r})$ after update. For variable $y$, due to the strong convexity on the maximization side, we can directly update it at the beginning of each iteration and get a temporary $\tilde{y}_{r+1}$. Subsequently, our algorithm involves building a private estimator $\widetilde{v}_r$ to approximate the $\nabla_{x} \hat{L}(x_r, \tilde{y}_{r+1};D)$. Generally speaking, $\widetilde{v}_r$ accumulates stochastic gradient differences between two consecutive iterations. 
In detail, we begin with the following equation:
\[\nabla_{x} \hat{L}(x_r, \tilde{y}_{r+1}; D) =\nabla_{x} \hat{L} (x_r, \tilde{y}_{r+1}; D)-\nabla_{x} \hat{L
} (x_{r-1}, \tilde{y}_{r}; D)+\nabla _{x}\hat{L}(x_{r-1}, \tilde{y}_{r}; D ). \]


We use a stochastic gradient ascent algorithm to update $\tilde{y}_r$ to $\tilde{y}_{r+1}$. In doing so, we can approximate 
 $\nabla_{x} \hat{L}(x_r, \tilde{y}_{r+1} ;D )$ and $\nabla_{x} \hat{L
} (x_{r-1}, \tilde{y}_{r}; D)$ by $\nabla \Phi(x_{r})$ and $\nabla \Phi(x_{r-1})$ respectively. This approximation is accurate up to some controlled error term by Lemma \ref{lemma:y*} and the convergence rate of SGDA. Moreover, since  $\Phi(\cdot )$ is smooth, indicating that $|\nabla \Phi(x_{r})-\nabla \Phi(x_{r-1})|\leq O(1) \|x_{r}-x_{r-1}\|$. This means that we can add a noise whose variance $\xi_{x_r}$ is proportional to $\|x_{r}-x_{r-1}\|$ to ensure the differential privacy. In total, we have 
\begin{equation}\label{eq:6}
    \widetilde{v}_r \approx (\nabla \Phi(x_{r})-\nabla \Phi(x_{r-1})+\xi_{x_r}  )+\widetilde{v}_{r-1}.
\end{equation}
with  initial $\widetilde{v}_0:=0$. Previously, the $l_2-$sensitivity of the private estimator in Algorithm \ref{alg:DP-SGDA} is bounded by the whole gradient's Lipschitz constant. It is now bounded by the distance of $x_r$ and $x_{r-1}$. Therefore, it can be much smaller than the gradient's $l_2$-sensitivity when $x_r$ and $x_{r-1}$ are near enough. Hence, our algorithm's gradient differences accumulation design breeds the ability to add smaller noise variance while preserving privacy. 

\noindent \textbf{Algorithm Layouts:} Algorithm \ref{alg:Framwork} is the detailed implementation of our above idea.  In each iteration, we first leverage a clipped version of Mini-batch SGA for strongly concave loss function $\hat{L}(x_{r},\cdot;D)$ with initialization $y_r$ to get a non-private version of $y_{r+1}$ (step 3). 
We then need to construct a private estimator $ \widetilde{v}_{r+1} $  to approximate the gradient $\nabla_{x} \hat{L}(x_r, y_{r+1};D)$. To get this, our framework restarts every $T$ round, where the length of $T$ is carefully controlled in pursuit of optimal utility bound in our analysis. Specifically, for every $T$ iteration, we will calculate a subsampled gradient $\mathbf{d_r}$, which is a base state analogous to the initial differential private gradient estimator $\tilde{v}_1=\mathbf{d_0}=\frac{1}{m} \sum_{j=1}^m \text{Clipping}(\nabla_{x} f(x_{0}, y_{1}; z_{0}^j),C_1)$. Note that such a restart mechanism is essential as it can significantly reduce the noise we add since we just need to add noise with whose scale depends on the Lipschitz constant every $T$ iterations.

If $r\% T\neq 0$, we then leverage \eqref{eq:6} to recursively update $v_{r+1}$ via adding $v_{r}$ with the gradient difference $\mathbf{d_r}=\frac{1}{m} \sum_{j=1}^m $Clipping $(  \nabla_{x} f(x_{r}, y_{r+1}; z_{r}^j)-\nabla_{x} f(x_{r-1}, y_{r}; z_{r-1}^j),C_{2, r})$(step 7). Subsequently, We add noise to $v_{r+1}$ to ensure DP. Note that when  $r\% T\neq 0$ the noise scale only depends on $\|x_r-x_{r-1}\|$ and a small constant $C_3$ that corresponds to the convergence rate of SGA. 

The private estimator $\widetilde{v}_r$ is then utilized by performing gradient descent on $x_{r}$ to get new $x_{r+1}$. After that, we perform output perturbation on $y_r$ to get the final private version $\tilde{y}_{r+1}$. In the following, we provide privacy and utility guarantees.

\begin{theorem}
\label{thm:DP guarantee}
 Under Assumption \ref{NCSC},  there exist constants $c_4, c_5, c_6$ and $c_7>0$ so that given the mini-batch size $m$, restart interval $T$ and total iterations $R$, for any $\epsilon<c_4 m^2 T / n^2$, Algorithm \ref{alg:Framwork},  is $(\epsilon, \delta)$-differentially private for any $\delta>0$ if we choose

\begin{equation*}
\sigma_{{x_1}}=\frac{c_5 \sqrt{\frac{R}{T} \log (1 / \delta)}}{n\epsilon}, \sigma_{{x_2}}=\frac{c_6 \sqrt{R\log (1 / \delta)}}{n\epsilon}\text{\quad and \quad} \sigma_y=c_7\frac{(2C_0^2+\beta  M)\sqrt{R \log (1 / \delta)}}{n \epsilon}.
\end{equation*}
   
\end{theorem}

\begin{remark}
\label{DP parameter remark}
   We give a set of parameters applicable to Theorem \ref{thm:DP guarantee} here in practice. By setting $\epsilon \leq 1, \delta \leq 1 / n^2$ and $m=$ $\max (1, n \sqrt{\epsilon /(8T)})$, then explicit values for the variances are: $\sigma_{x_1}=\frac{4\sqrt{\frac{R}{T
   }\log (1 / \delta)}}{n \epsilon}, \sigma_{x_2}=\frac{4\sqrt{R\log (1 / \delta)}}{n \epsilon},\sigma_{y}=\frac{4(2C_0^2+\beta M) \sqrt{R\log (1 / \delta)}}{\mu n\epsilon}$ .
\end{remark}
Since our mechanism can reduce the noise scale and thus the variance of the private gradient estimator $\widetilde{v}_r$. Therefore, we expect a better utility bound than the standard DP-SGDA, which is formally stated as the following. 

\begin{theorem}
\label{thm: PrivateDiff utility}
     Let $\varepsilon \in (0,\frac{1}{e})$ and suppose  Assumptions \ref{NCSC}-\ref{stochastic gradient variance} hold. In Algorithm \ref{alg:Framwork} and under the choices of $\sigma_{x_1}^2$, $\sigma_{x_2}^2$, and $\sigma_y$ in Theorem \ref{thm:DP guarantee}, if we further set $C_0\geq G_y$, $C_1\geq G_x$, $C_2 \geq l+\kappa l$ and  $C_3= \tilde{O}(\frac{1}{\sqrt{T_2}})$; the stepsizes $\eta_x=O(\min \{\frac{1}{l+\kappa l}, \frac{1}{\sqrt{T} l\sigma_{x_2} \sqrt{d}} \}),\eta_{y_i}=\frac{1}{\mu i}$; the restart interval $T=\Theta\left((\frac{ \sqrt{d}}{n\epsilon})^{2/3} R \right)$, total number of rounds $R=\widetilde{\Theta}\left(\max\{ \frac{1}{\varepsilon_{\mathrm{opt}}}, (\frac{ d}{n^2\epsilon^2 \varepsilon_{\mathrm{opt}}^2}\}\right)$ with $\varepsilon_{\mathrm{opt}}:=O(\frac{d^{\frac{2}{3}}}{(n\epsilon)^{\frac{4}{3}}})$, number of iterations of Mini-batch SGA $T_2= O \left(\max\{\frac{(n\epsilon)^{{4}/{3}}}{d^{{2}/{3}}},TR
\cdot \frac{d^{1/3}}{(n\epsilon)^{2/3}}\}\right)$ and the batch size $m=O(\frac{(n\epsilon)^{4/3}}{{d^{2/3}}})$, with probability at least $1-\vartheta$, the utility bound of  PrivateDiff Minimax satisfies 
$$ 
\mathbb{E}\|\nabla \Phi(x^{\text{priv}})\| \leq \widetilde{O}(\frac{(d\log \frac{1}{\delta})^{{1}/{3}}}{(n \epsilon )^{\frac{2}{3}}}).
$$
\end{theorem}
 The obtained utility is significantly better than the best-known utility bound $\widetilde{O}(d^\frac{1}{4}/\sqrt{n\epsilon})$ when $n\geq \Omega(\sqrt{d} /(G \epsilon))$.  Note that by some appropriate choice of the thresholds, one can show that the clipping has no effect.  Moreover, we can see there are two terms in $C_{2,r}$ where the first term corresponds to the upper bound of $|\nabla \Phi(x_{r})-\nabla \Phi(x_{r-1})|$ and the second one is the convergence error caused by $y_{r+1}$. Thus, when $T_2$ is large enough, the noise $\sigma_{x_2}$ could be very small if $x_r$ is close enough to $x_{r-1}$. 

\begin{figure*}[t]
    \centering
    \begin{subfigure}{0.247\textwidth}
        \centering
        \includegraphics[width=\linewidth]{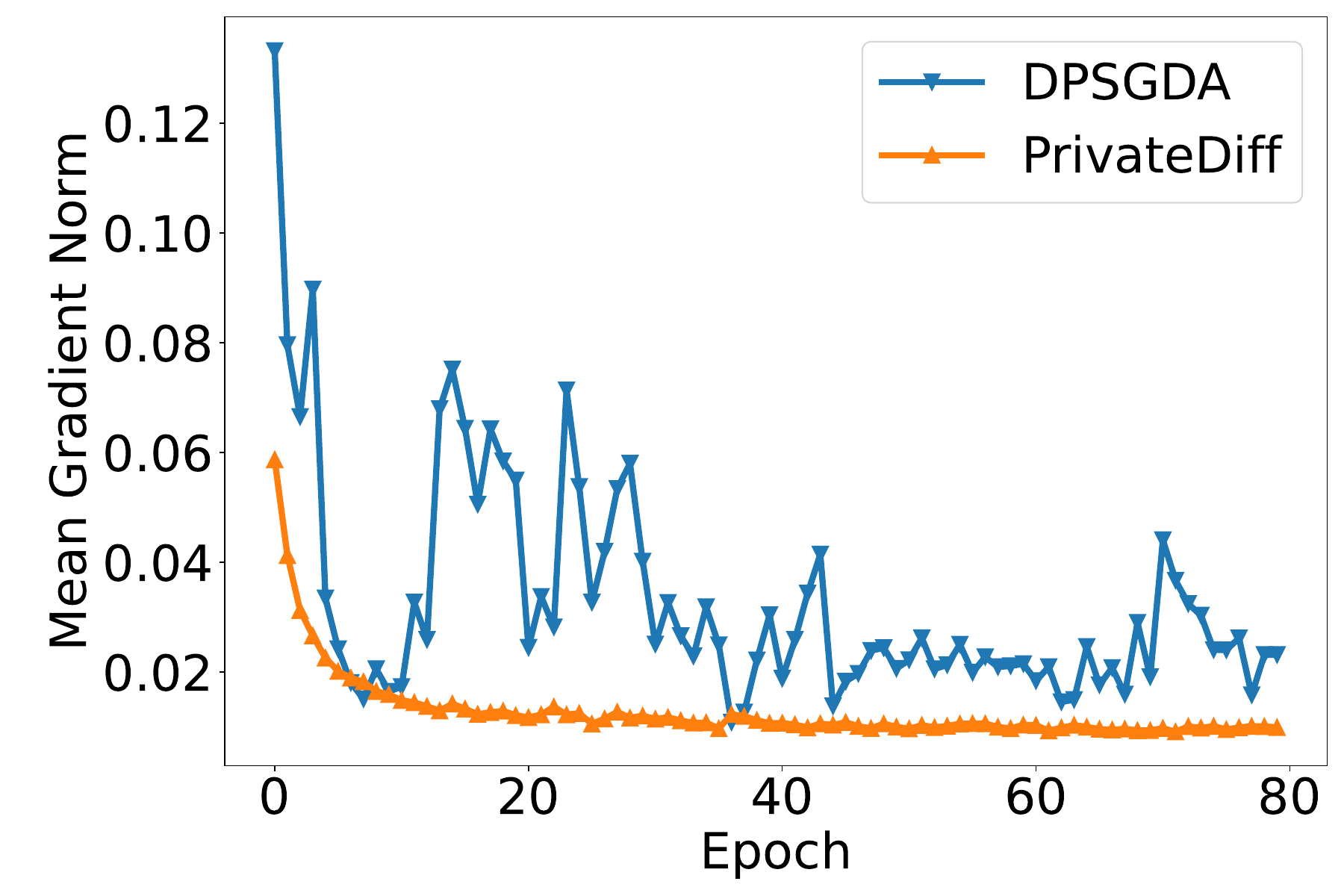}
        \caption{Gradient Norm}
        \label{fig:norm}
    \end{subfigure}%
    \begin{subfigure}{0.247\textwidth}
        \centering
        \includegraphics[width=\linewidth]{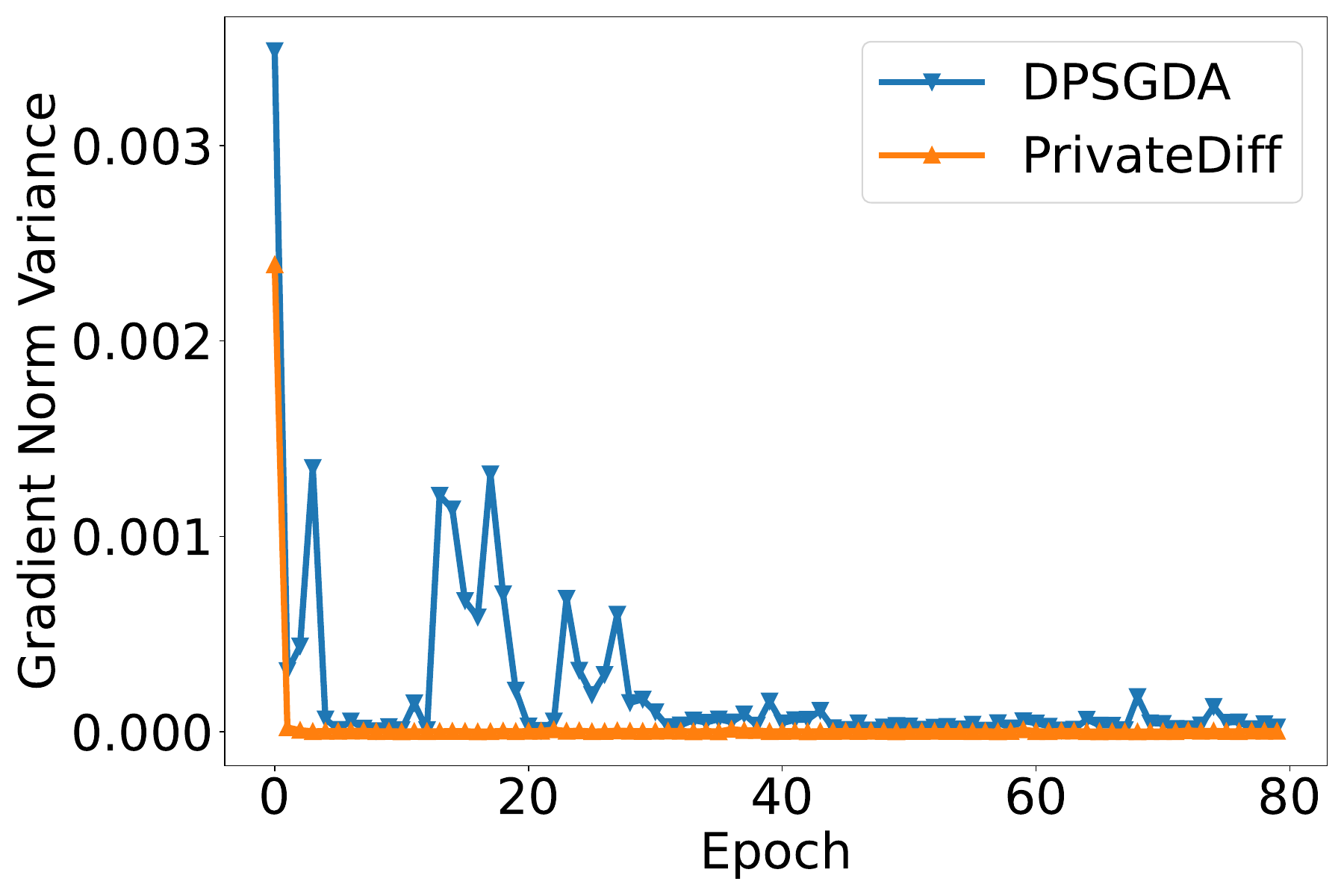}
        \caption{Gradient Variance}
        \label{fig:variance}
    \end{subfigure}
    \begin{subfigure}{0.247\textwidth}
        \centering
        \includegraphics[width=\linewidth]{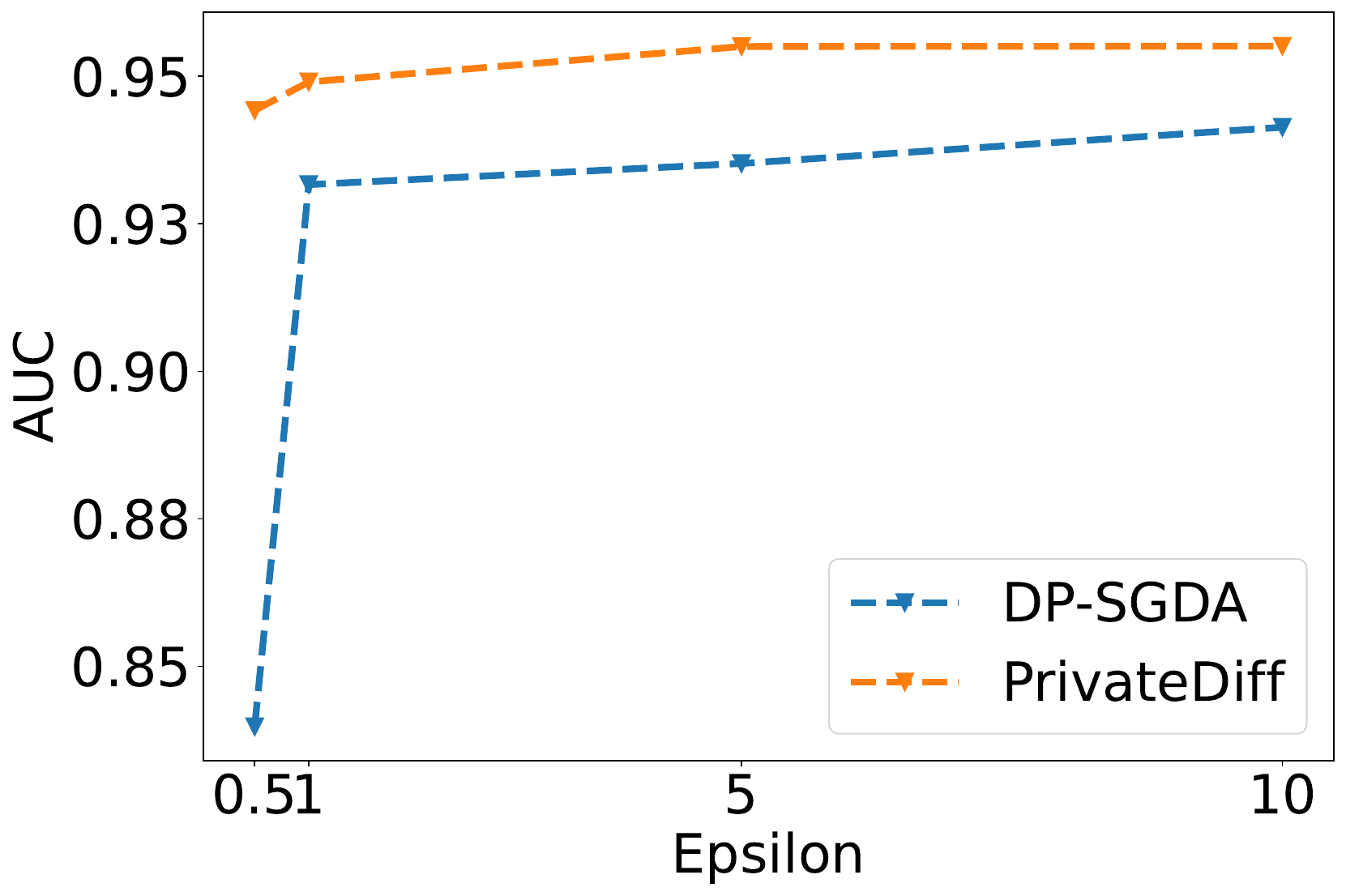}
        \caption{Fashion-MNIST}
        \label{fig:epfmnist}
    \end{subfigure}%
    \begin{subfigure}{0.247\textwidth}
        \centering
        \includegraphics[width=\linewidth]{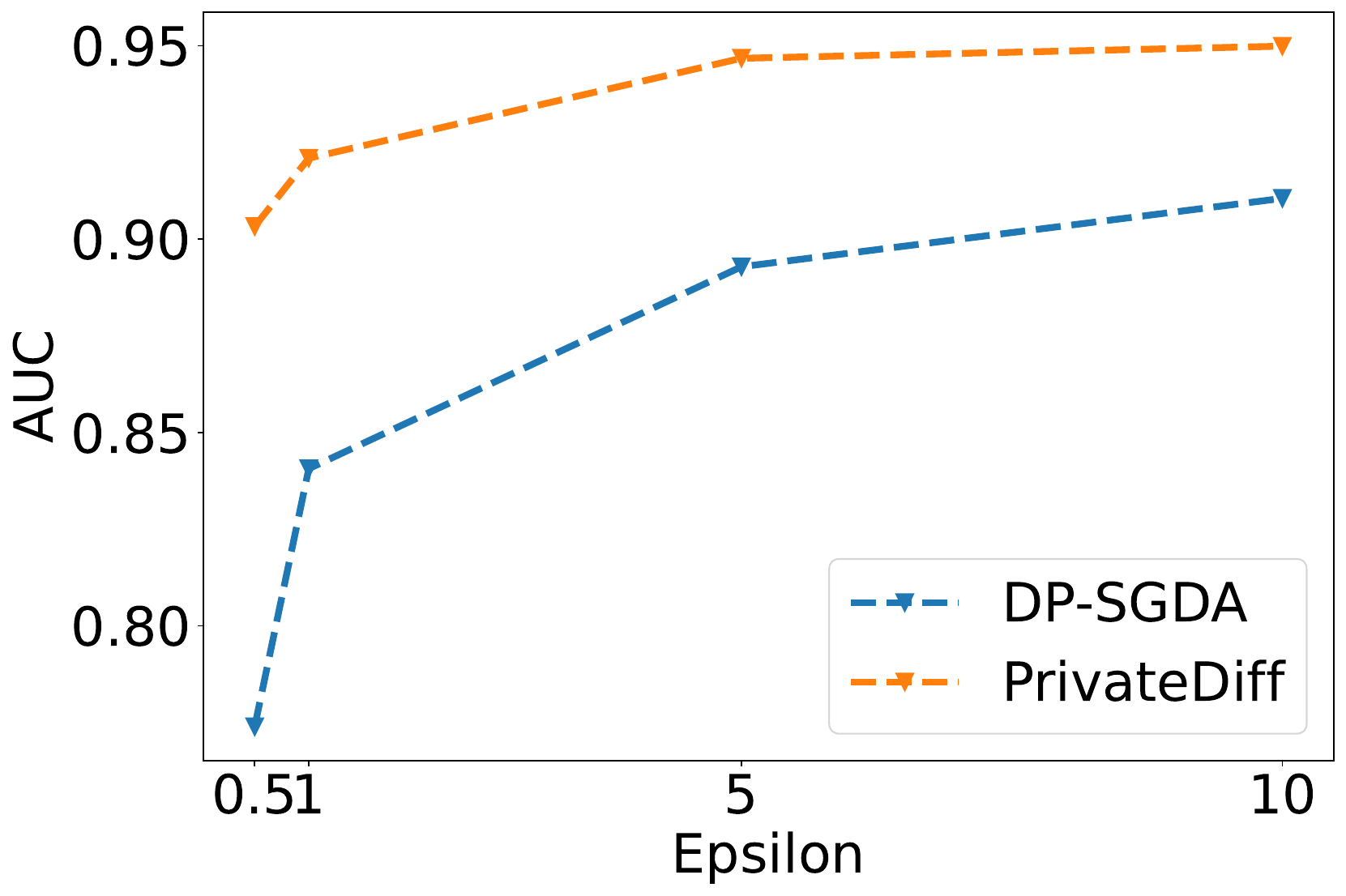}
        \caption{MNIST}
        \label{fig:epmnist}
    \end{subfigure}
    \caption{Comparison of Gradient Norm, Gradient Variance, and AUC Performance between DP-SGDA and PrivateDiff.}
\end{figure*}

\section{Experiments}
In this section, we evaluate the effectiveness of our proposed PrivateDiff Minimax method. Due to space constraints, we focus on the AUC maximization experiment here. Additional experiments, including reinforcement learning and generative adversarial networks, are provided in the appendix.

\begin{table*}[ht] 
\setlength{\fboxsep}{0pt}
\centering
\setlength{\tabcolsep}{3pt} 
\setlength\arrayrulewidth{0.6pt}

\resizebox{\textwidth}{!}{%
{\begin{tabular}{lcccccccc}
\toprule
{Dataset} &  \multicolumn{2}{c}{Fashion-MNIST} & \multicolumn{2}{c}{MNIST} & \multicolumn{2}{c}{Imbalanced Fashion-MNIST}   & \multicolumn{2}{c}{Imbalanced MNIST} \\ \cmidrule{2-9}
 & DP-SGDA $\uparrow$ & PrivateDiff $\uparrow$  & DP-SGDA $\uparrow$ & PrivateDiff $\uparrow$  & DP-SGDA $\uparrow$ & PrivateDiff $\uparrow$  & DP-SGDA $\uparrow$ & PrivateDiff $\uparrow$  \\
\midrule
\multicolumn{1}{c|}{Non-private}  & 0.9661  & 0.9659                  & 0.9901  & 0.9901          & 0.9567  & 0.9569                             & 0.9588  & 0.9593 \\
\multicolumn{1}{c|}{$\epsilon=0.5$}    & 0.9203  & 0.9569                  & 0.8837  & 0.9608          & 0.8398  & 0.9442                             & 0.7739  & 0.9033    \\
\multicolumn{1}{c|}{$\epsilon=1$}   & 0.9403  & 0.9609                  & 0.9022  & 0.9729          & 0.9317  & 0.9491                             & 0.8406  & 0.9209     \\
\multicolumn{1}{c|}{$\epsilon=5$}  & 0.9412  & 0.9657                  & 0.9544  & 0.9860          & 0.9352  & 0.9551                             & 0.8928  & 0.9467  \\
\multicolumn{1}{c|}{$\epsilon=10$}   & 0.9426  & 0.9660                  & 0.9532  & 0.9878          & 0.9414  & 0.9551                             & 0.9105  & 0.9499    \\

\bottomrule
\end{tabular}}
}

\caption{Comparison of AUC performance in DP-SGDA and PrivateDiff Minimax on various datasets.}
\label{tab:main}
\end{table*}

\noindent \textbf{Experimental Setup} 
We first conduct experiments on the problem of the Area under the curve (AUC) maximization with the least squares loss \citep{Yuan_2021_ICCV} to evaluate the DP-SGDA and PrivateDiff (Minimax) algorithms. AUC, ranging from 0 to 1, is a widely used metric to evaluate the performance of binary classification models. It is particularly valuable in situations where the class distribution is imbalanced because it captures the trade-offs between true positive and false positive rates. A good classifier should achieve AUC scores close to one. Maximizing AUC was demonstrated to be equivalent to a minimax problem. More detailed introductions to AUC are included in the appendix.

Our experiments are based on two common datasets, MNIST and FashionMNIST, which are transformed into binary classes by randomly partitioning the data into two groups. Following this, we create imbalanced conditions, setting an imbalance ratio of 0.1 for training, where minority classes are underrepresented, and 0.5 for testing. We chose to evaluate an imbalanced dataset because the evaluation metric, AUC scores, is particularly well-suited for assessing small or imbalanced datasets, providing a clearer indication of the algorithm's performance.

We set privacy budget $\epsilon=\{0.5,1,5,10\}$ and $\delta=\frac{1}{n^{1.1}}$.  A two-layer multilayer perceptron is used, consisting of 256 and 128 neurons, respectively. For other hyperparameters, we either used a grid search to select the best one or followed our previous theorems. 


\noindent \textbf{General AUC Performance vs Privacy}
Table~\ref{tab:main} demonstrates that PrivateDiff Minimax consistently achieves higher AUC scores than DP-SGDA across all dataset and privacy budget combinations.
It shows that PrivateDiff consistently outperforms DP-SGDA across various datasets (Fashion-MNIST, MNIST, Imbalanced Fashion-MNIST, and Imbalanced MNIST). The performance gap is most significant at lower privacy budgets ( $\epsilon=0.5$ and 1), particularly in the MNIST and Imbalanced MNIST datasets. As the privacy budget increases, the gap narrows, but PrivateDiff still maintains a higher AUC across all scenarios, demonstrating its robustness and effectiveness in preserving utility under strong privacy constraints.


We also compare the performance of DP-SGDA and PrivateDiff across various privacy budgets $(\epsilon)$ on the Fashion-MNIST and MNIST datasets. The results in Figures \ref{fig:epfmnist} and \ref{fig:epmnist} highlight the following observations:$\mathbf{1)}$ Performance Across Datasets: On both the Fashion-MNIST and MNIST datasets, PrivateDiff consistently outperforms DP-SGDA across all values of $\epsilon$. This suggests that PrivateDiff is more robust in maintaining a higher AUROC score, indicating better classification performance even under stronger privacy constraints. $\mathbf{2)}$ Impact of Epsilon on AUROC: As $\epsilon$ increases, the AUROC for both DP-SGDA and PrivateDiff improves, reflecting the typical trade-off between privacy and utility in differential privacy frameworks. With higher $\epsilon$, the privacy guarantee becomes weaker, allowing the models to achieve higher AUROC values. $\mathbf{3)}$ Comparison of Improvements: The relative improvement in AUROC with increasing $\epsilon$ is more pronounced for DP-SGDA, particularly in the MNIST dataset (Figure \ref{fig:epmnist}). This might suggest that DP-SGDA's performance is more sensitive to changes in the privacy budget than that of PrivateDiff.

\noindent \textbf{Robustness of PrivateDiff} PrivateDiff consistently maintains lower gradient norm variance throughout the training process, as seen in Figure \ref{fig:variance}. This reduced variance indicates a more consistent optimization trajectory, minimizing the stochastic fluctuations and contributing to a more robust training process. In contrast, DP-SGDA shows higher variance early in the training process, which indicates initial instability. An increase in variance leads to more unstable updates, which may result in overshooting or oscillating around the optimal solution. Note that a similar phenomenon has also appeared at DP Empirical Risk Minimization with non-convex loss \citep{wang2019differentially}. 

Moreover, Figure \ref{fig:norm} illustrates that PrivateDiff achieves a stable decrease in the mean gradient norm over epochs, exhibiting fewer fluctuations compared to DP-SGDA. The steady reduction in mean gradient norm and low variance associated with PrivateDiff suggest a more reliable convergence behavior, crucial for steadily approaching the optimal solution without divergence or instability. Conversely, DP-SGDA's convergence is less reliable due to its higher variance and instability, which can lead to convergence to suboptimal solutions. 
These observations align with our theoretical conclusions that PrivateDiff can effectively reduce variance and offer a more stable and consistent optimization process.

\section{Conclusions}
We studied the finite sum minimax optimization problem in the Differential Privacy (DP) model where the loss function is nonconvex-(strongly)-concave. Specifically, we first analyzed DP-SGDA, which was studied previously only for convex-concave or the loss satisfying the PL condition. We then discussed several lower bounds. To further fill in the gap between lower and upper bounds, we then proposed a novel variance reduction-based algorithm. Experiments on AUC maximization, generative adversarial networks and temporal difference learning supported our theoretical analysis.  

\section*{Acknowledgments}
{
Di Wang and Zihang Xiang are supported in part by the funding BAS/1/1689-01-01, URF/1/4663-01-01,  REI/1/5232-01-01,  REI/1/5332-01-01,  and URF/1/5508-01-01  from KAUST, and funding from KAUST - Center of Excellence for Generative AI, under award number 5940.}

\bibliography{References}


\section{Proofs of Theorems}
\subsection{Proof of Theorem \ref{DP SGDA privacy thm}:} Refer to the proof in Appendix B of \cite{yang2022differentially}.

\subsection{Proof of Theorem \ref{thm: DP-SGDA utility}:}

Recall that 
$$\min _x \max _y \hat{L}(x, y; D)=\frac{1}{n} \sum_{i=1}^n f(x_i, y_i; z_t^j).$$
We first give some auxilliary lemmas for the proof. 
\begin{lemma}\label{auxilliary lemma 1}{For DP-SGDA, the iterates ${x_t}$ satisfy the following inequality:}
\begin{equation}
\label{Phi_descent_Lemma}
\begin{aligned}
\mathbb{E}[\Phi(x_t)] \leq & \mathbb{E}[\Phi(x_{t-1})]+[2(l+\kappa l) \eta_{x}^2-\frac{\eta_x}{2}]\|\nabla \Phi(x_{t-1})\|_2^2+(l+\kappa l)\eta_{x}^2(\frac{\mathcal{B}^2}{m}+\mathbb{E}\|\xi_{t-1}\|_2^2)\\
&\quad +[2(l+\kappa l) \eta_x{ }^2+\frac{\eta_x}{2}]\|\nabla \Phi(x_{t-1})-\nabla_{x}\hat{L}(x_{t-1}, y_{t-1})\|^2.
\end{aligned}
\end{equation}
\end{lemma}

\begin{proof}[{\bf Proof of Lemma~\ref{auxilliary lemma 1}}]
   Since  $\Phi(x)=\max _y \hat{L}(x, \cdot; D)$ is $(l+\kappa l)$-smooth with $\kappa=\frac{l}{\mu}$,  we have:
\begin{equation}
\label{phi smooth}
   \Phi(x_t) \leq \Phi(x_{t-1})+\nabla \Phi(x_{t-1})^{\top}(x_t-x_{t-1})+\frac{l+\kappa l}{2}\|x_t-x_{t-1}\|_2^2. 
\end{equation}
When $C_1 \geq G_x$, we have the following update rule of variable $x$ in Algorithm \ref{alg:DP-SGDA}
\begin{equation}
\label{x_update in DP-SGDA}
    x_t-x_{t-1}=-\eta_x(\frac{1}{m} \sum_{i=1}^m \nabla_{x} f(x_{t-1}, y_{t-1} ; z_t^{j})+\xi_t).
\end{equation}
Therefore, we plug \eqref{x_update in DP-SGDA} into \eqref{phi smooth} and we get:
\begin{equation}
\label{Phi_descent}
    \begin{aligned}
        \begin{aligned}
\Phi(x_t) &\leq \Phi(x_{t-1})+\nabla \Phi(x_{t-1})^{\top}[-\eta_x(\frac{1}{m} \sum_{i=1}^m \nabla_{x} f(x_{t-1}, y_{t-1}; z_t^j)+\xi_t)]\\
& \quad +\frac{l+\kappa l}{2}\|-\eta_x(\frac{1}{m} \sum_{i=1}^m \nabla_{x} f(x_{t-1}, y_{t-1}; z_t^{j})+\xi_t)\|_2^2 .
\end{aligned}
    \end{aligned}
\end{equation}

Take square and expectation on both sides of \eqref{x_update in DP-SGDA}:
\begin{equation}
\label{xt-xt-1}
    \begin{aligned}
\mathbb{E} \|x_t-x_{t-1}\|_2^2 & =\eta_x^2 \mathbb{E} \|\frac{1}{m} \sum_{i=1}^m \nabla_{x} f(x_{t-1}, y_{t-1} ; z_t^{j})+\xi_{t-1}\|_2^2 \\
& \leq 2\eta_x^2[\mathbb{E}\|\frac{1}{m} \sum_{i=1}^n \nabla_{x} f(x_{t-1}, y_{t-1} ; z_t^{j})\|_2^2+\mathbb{E}\|\xi_{t-1}\|_2^2] \\
& \stackrel{(a)}\leq 2\eta_x^2[\|\nabla_{x}\hat{L}(x_{t-1}, y_{t-1})\|_2^2+\frac{\mathcal{B}^2}{m}+\mathbb{E} \|\xi_{t-1}\|_2^2] \\
& =2\eta_x^2 \|\nabla_{x}\hat{L}(x_{t-1}, y_{t-1})\|_2^2+2\eta_x^2(\frac{\mathcal{B}^2}{m}+\mathbb{E}\|\xi_{t-1}\|_2^2).
\end{aligned}
\end{equation}

$(a)$ is derived from the bounded variance of stochastic gradients in Assumption \ref{stochastic gradient variance}.

We restate the above result here:
\begin{equation}
\label{updates of variable x}
   \mathbb{E}\|x_t-x_{t-1}\|_2^2\leq2\eta_x^2 \|\nabla_{x}\hat{L}(x_{t-1}, y_{t-1})\|_2^2+2\eta_x^2(\frac{\mathcal{B}^2}{m}+\mathbb{E} \|\xi_{t-1}\|_2^2).
\end{equation}

We then take expectation on both sides of \eqref{Phi_descent}, conditioned on $(\mathrm{x}_{t-1}, y_{t-1})$. It yields that
$$
\begin{aligned}
\mathbb{E}[\Phi(x_t) \mid x_{t-1}, y_{t-1}]&=\Phi(x_{t-1})-\eta_x \nabla \Phi(x_{t-1})^{\top} \nabla_{x}\hat{L}(x_{t-1}, y_{t-1})+\frac{l+\kappa l}{2}\eta_x^2\mathbb{E} \|\frac{1}{m} \sum_{i=1}^m \nabla_{x} f(x_{t-1}, y_{t-1} ; z_t^j)+\xi_{t-1}\|_2^2 \\
& =\Phi(x_{t-1})-\eta_x\|\nabla \Phi(x_{t-1})\|_2^2+\eta_x \nabla \Phi(x_{t-1})^{\top}(\nabla \Phi(x_{t-1})-\nabla_{x}\hat{L}(x_{t-1}, y_{t-1})) \\
& \quad +(l+\kappa l) \eta_x^2[\|\nabla_{x}\hat{L}(x_{t-1}, y_{t-1})\|_2^2+\frac{\mathcal{B}^2}{m}+\mathbb{E}\|\xi_{t-1}\|_2^2] \\
& \stackrel{(a)}\leq \Phi(x_{t-1})-\eta_x\|\nabla \Phi(x_{t-1})\|_2^2+\eta_x \frac{\|\nabla \Phi(x_{t-1})-\nabla_{x}\hat{L}(x_{t-1}, y_{t-1})\|_2^2+\|\nabla \Phi(x_{t-1})\|_2^2}{2} \\
& \quad+(l+\kappa l) \eta_x^2[2\|\nabla_{x}\hat{L}(x_{t-1}, y_{t-1})-\nabla \Phi(x_{t-1})\|_2^2+2\|\nabla \Phi(x_{t-1})\|_2^2+\frac{\mathcal{B}^2}{m}+\mathbb{E}\|\xi_{t-1}\|_2^2].\\
\end{aligned}
$$
$(a)$ results from two important observations. One is using Young's inequality:
\begin{equation}
\label{Young's inequality}
    \nabla \Phi(x_{t-1})^{\top}(\nabla \Phi(x_{t-1})-\nabla_{x}\hat{L}(x_{t-1}, y_{t-1})) 
\leq  \frac{\|\nabla \Phi(x_{t-1})-\nabla_{x}\hat{L}(x_{t-1}, y_{t-1})\|_2^2+\|\nabla \Phi(x_{t-1})\|_2^2}{2}.
\end{equation}

The other is from the Cauchy-Schwartz inequality:
\begin{equation}
    \label{Cauchy_Schwartz}
    \|\nabla_{x} \hat{L}(x_{t-1}, y_{t-1})\|_2^2 \leq 2\|\nabla_{x}\hat{L}(x_{t-1}, y_{t-1})-\nabla \Phi(x_{t-1})\|_2^2+2\|\nabla \Phi(x_{t-1})\|_2^2.
\end{equation}

Above all, we derive our lemma:
\begin{equation*}
\begin{aligned}
\mathbb{E}[\Phi(x_t)] \leq & \mathbb{E}[\Phi(x_{t-1})]+[2(l+\kappa l) \eta_{x}^2-\frac{\eta_x}{2}]\|\nabla \Phi(x_{t-1})\|_2^2+(l+\kappa l)\eta_{x}^2(\frac{\mathcal{B}^2}{m}+\mathbb{E}\|\xi_{t-1}\|_2^2)\\
&\quad +[2(l+\kappa l) \eta_x{ }^2+\frac{\eta_x}{2}]\|\nabla \Phi(x_{t-1})-\nabla_{x}\hat{L}(x_{t-1}, y_{t-1})\|^2.
\end{aligned}
\end{equation*}
\end{proof}

\begin{lemma}{For DP-SGDA, let $\theta_t=\mathbb{E}[\|y^{\star}(x_t)-y_t\|^2]$, we have the following statement:
}
\label{auxilliary lemma 2}
\begin{equation}
\label{Lemma:theta_t}
    \theta_t \leq (1-\frac{1}{2 k}+4 k^3 \eta_x^2 l^2) \theta_{t-1}+4 k^3 \eta_x^2\|\nabla \Phi(x_{t-1})\|_2^2+( 2 k^3 \eta_x^2+\frac{2}{l^2}) \frac{\mathcal{B}^2}{m}+2 k^3 \eta_x^2 \mathbb{E}\|\xi_{t-1}\|_2^2+\frac{2}{l^2} \mathbb{E}\|\zeta_{t-1}\|_2^2. 
\end{equation}
    
\end{lemma}

\begin{proof}[{\bf Proof of Lemma \ref{auxilliary lemma 2}}] 
    By Young's inequality, we have
$$
\begin{aligned}
\theta_t & \leq(1+\frac{1}{2(\kappa-1)}) \mathbb{E}[\|y^{\star}(x_{t-1})-y_t\|^2]+(1+2(\kappa-1)) \mathbb{E}[\|y^{\star}(x_t)-y^{\star}(x_{t-1})\|^2] \\
& \leq(\frac{2 \kappa-1}{2 \kappa-2}) \mathbb{E}[\|y^{\star}(x_{t-1})-y_t\|^2]+2 \kappa \mathbb{E}[\|y^{\star}(x_t)-y^{\star}(x_{t-1})\|^2] \\
& \stackrel{(a)}{\leq}(1-\frac{1}{2 \kappa}) \theta_{t-1}+2 \kappa \mathbb{E}[\|y^{\star}(x_t)-y^{\star}(x_{t-1})\|^2]+\frac{2}{l^2 }(\frac{\mathcal{B}^2}{m}+\mathbb{E}\|\zeta_{t-1}\|_2^2) .
\end{aligned}
$$

$(a)$ is derived as follows:

\begin{equation}
\label{first term in theta_t}
    \begin{aligned}
\mathbb{E}\|y^*(x_{t-1})-y_t\|_2^2&\stackrel{(b)}=\theta_{t-1}+\eta_y^2\|\nabla_{y}\hat{L}(x_{t-1}, y_{t-1})\|_2^2+\eta_y^2(\frac{\mathcal{B}^2}{m}+\mathbb{E}\|\zeta_{t-1}\|_2^2) \\
&\quad -2 \eta_y\langle y^*(x_{t-1})-y_{t-1}, \nabla_{y} L( x_{t-1}, y_{t-1})\rangle \\
& \stackrel{(c)}{\leq} (1-\frac{1}{\kappa} )\theta_{t-1}+\frac{2}{l^2 }(\frac{\mathcal{B}^2}{m}+\mathbb{E}\|\zeta_{t-1}\|_2^2).
\end{aligned}
\end{equation}

By the the following update rule of variable $y$ in Algorithm \ref{alg:DP-SGDA} and $C_2 \geq G_y$ in Theorem \ref{DP SGDA privacy thm}, 
\begin{equation}
\label{y_update in DP-SGDA}
    \|y_t-y_{t-1}\|_2\leq\|-\eta_y(\frac{1}{m} \sum_{i=1}^m \nabla_{y} f(x_{t-1}, y_{t-1} ; z_t^{j})+\zeta_t)\|_2.
\end{equation}

We decompose $\mathbb{E}[\|y^{\star}(x_{t-1})-y_t\|^2]$ into:

\begin{equation}\label{y^*-y}
\mathbb{E}[\|y^{\star}(x_{t-1})-y_{t-1}+y_{t-1}-y_t\|^2]= \mathbb{E}[\|y^{\star}(x_{t-1})-y_{t-1}\|^2]+\mathbb{E}[\|y_{t-1}-y_t\|^2]+2\mathbb{E}[(y^{\star}(x_{t-1})-y_{t-1})^{T}(y_{t-1}-y_t)].
\end{equation}

Plug \eqref{y_update in DP-SGDA}
into \eqref{y^*-y} and then yield $(b)$.

We show $(c)$ by using the fact that $\hat{L}(x, y)$ is $\mu$-strongly concave in $y \text{ and is } l$-smooth; 

We can see that $-\hat{L}(x, y)$ is $l-$smooth as well, then we have:
\begin{equation}
\label{-L smooth}
   -\hat{L}(x_{t-1}, y_t) \leq -\hat{L}(x_{t-1}, y_{t-1})+\langle-\nabla_{y}\hat{L}(x_{t-1}, y_{t-1}), y_t-y_{t-1}\rangle+\frac{l}{2}\|y_t-y_{t-1}\|^2.  
\end{equation}
By taking expectation on both hand sides of \eqref{-L smooth}, we yield that:
\begin{equation}
\label{E[-L]}
    \mathbb{E}[-\hat{L}(x_{t-1}, y_t)] \leq -\mathbb{E}[\hat{L}(x_t-1, y_{t-1})]-\eta_y \mathbb{E} \| \nabla_{y} (x_{t-1}, y_{t-1}) \|_2^2 +\frac{l}{2}\eta_y^2(\frac{\mathcal{B}^2}{m}+\mathbb{E}\|\zeta_{t-1}\|_2^2)+\frac{l}{2} \eta_y^2 \mathbb{E}\|\nabla_{y}(x_{t-1}, y_{t-1})\|_2^2.
\end{equation}
Take $\eta_y=\frac{1}{l}$, \eqref{E[-L]} becomes:
\begin{equation}
    -\mathbb{E}[\hat{L}(x_{t-1}, y_t)] \leq-\mathbb{E}[\hat{L}(x_{t-1}, y_{t-1})]-\frac{1}{2 l} \mathbb{E}\|\nabla_{y} \hat{L}(x_t-1, y_{t-1})\|_2^2+\frac{1}{2l}(\frac{\mathcal{B}^2}{m}+\mathbb{E}\|\zeta_{t-1}\|_2^2).
\end{equation}

By shifting the gradient term to the left-hand side and doing some simple algebra, we get:
\begin{equation}
    \mathbb{E}\|\nabla_{y}\hat{L}(x_{t-1}, y_{t-1})\|_2^2 \leq 2l \mathbb{E}[\hat{L}(x_{t-1}, y_t)-\hat{L}(x_{t-1}, y_{t-1})]+(\frac{\mathcal{B}^2}{m}+\mathbb{E}\|\zeta_{t-1}\|_2^2).
\end{equation}
From the definition of $y^*(x_t)$, we have the following inequality: 

\begin{equation}
\label{y_gradient_bound}
   \mathbb{E}\|\nabla_{y}\hat{L}(x_{t-1}, y_{t-1})\|_2^2 \leq \mathbb{E}[L(x_{t-1}, y^*(x_{t-1}))-\hat{L}(x_{t-1}, y_{t-1})]+(\frac{\mathcal{B}^2}{m}+\mathbb{E}\|\zeta_{t-1}\|_2^2).
\end{equation}

 Also, we notice that $L(x_{t-1}, y^*(x_{t-1}))$ is strongly concave in $y$:
 \begin{equation}
     L(x_{t-1}, y^*(x_{t-1})) \leq\hat{L}(x_{t-1}, y_{t-1})+\langle \nabla_{y} L(x_{t-1}, y_{t-1}\rangle, y^*(x_{t-1})-y_{t-1}\rangle -\frac{\mu}{2}\|y^*(x_{t-1})-y_{t-1}\|_2^2.
 \end{equation}

Therefore,
\begin{equation}
\label{mu_convexity cross term}
    \langle-\nabla_{y}\hat{L}(x_{t-1}, y_{t-1}), y^*-y_{t-1}) \leq\hat{L}(x_{t-1}, y_{t-1})-L(x_{t-1}, y^*(x_{t-1})) - \frac{\mu}{2}\|y^*(x_{t-1})-y_{t-1}\|_2^2.
\end{equation}

Combining \eqref{y_gradient_bound}, \eqref{mu_convexity cross term} with \eqref{first term in theta_t}, we yield the inequality $(c)$

Thus, we arrive at the final stage of our lemma:
\begin{equation}
\label{final theta_t proof}
    \theta_t \leq(1-\frac{1}{2 \kappa}) \theta_{t-1}+2 \kappa \mathbb{E}[\|y^{\star}(x_t)-y^{\star}(x_{t-1})\|^2]+\frac{2}{l^2 }(\frac{\mathcal{B}^2}{m}+\mathbb{E}\|\zeta_{t-1}\|_2^2).
\end{equation}

Since $y^{\star}(\cdot)$ is $\kappa$-Lipschitz by Lemma \ref{lemma:y*}, $\|y^{\star}(x_t)-y^{\star}(x_{t-1})\| \leq \kappa\|x_t-x_{t-1}\|$. Furthermore, we apply \eqref{Cauchy_Schwartz} to \eqref{updates of variable x}:
\begin{equation}
\label{after_Cauchy_x_update}
    \mathbb{E}[\|x_t-x_{t-1}\|^2] \leq 2 \eta_{x}^2 l^2 \theta_{t-1}+2 \eta_{x}^2 \mathbb{E}[\|\nabla \Phi(x_{t-1})\|^2]+\eta_x^2(\frac{\mathcal{B}^2}{m}+\mathbb{E}\|\xi_{t-1}\|_2^2) .
\end{equation}

We combine \eqref{after_Cauchy_x_update} and \eqref{final theta_t proof} together and get the final proof of the lemma:
\begin{equation}
  \begin{aligned}
\theta_t & \leq 4 k^3 \eta_x^2 l^2\|y_{t-1}-y^*(x_{t-1})\|_2^2+4 k^3 \eta_x^2\|\nabla \Phi(x_{t-1})\|_2^2 +(1-\frac{1}{2 k}) \theta_{t-1}+(2 k^3 \eta_x^2+\frac{2}{l^2}) \frac{\mathcal{B}^2}{m}\\
&\quad+2 k^3 \eta_x^2\mathbb{E}\|\xi_{t-1}\|_2^2+\frac{2}{l^2}\mathbb{E}\|\zeta_{t-1}\|_2^2 \\
& =(1-\frac{1}{2 k}+4 k^3 \eta_x^2 l^2) \theta_{t-1}+4 k^3 \eta_x^2\|\nabla \Phi(x_{t-1})\|_2^2+( 2 k^3 \eta_x^2+\frac{2}{l^2}) \frac{\mathcal{B}^2}{m}+2 k^3 \eta_x^2 \mathbb{E}\|\xi_{t-1}\|_2^2+\frac{2}{l^2} \mathbb{E}\|\zeta_{t-1}\|_2^2. \\
\end{aligned}  
\end{equation}

\end{proof}

\begin{lemma}{$\text { For DP-SGDA, let } \theta_t=\mathbb{E}[\|y^{\star}(x_t)-y_t\|^2],$}
\label{Lemma: numerical one}
\begin{equation}
\label{numerical condition}
   \mathbb{E}[\Phi(x_t)]\leq \mathbb{E}[\Phi(x_{t-1})]-\frac{7}{16} \eta_x \mathbb{E}\|\nabla \Phi(x_{t-1})\|_2^2+\frac{9}{16} \eta_x l^2 \theta_{t-1}+(l+ kl) \eta_x^2(\frac{\mathcal{B}^2}{m}+\mathbb{E}\|\xi_{t-1}\|_2^2).
\end{equation}
\end{lemma}

\begin{proof}[{\bf Proof of Lemma \ref{Lemma: numerical one}}]
    We set $\eta_{x}=1 / 16(\kappa+1)^2 l$  and hence

\begin{equation}
    \frac{7 \eta_{\mathrm{x}}}{16} \leq \frac{\eta_{\mathrm{x}}}{2}-2 \eta_{\mathrm{x}}^2 \kappa l \leq \frac{\eta_{\mathrm{x}}}{2}+2 \eta_{\mathrm{x}}^2 \kappa l \leq \frac{9 \eta_{\mathrm{x}}}{16}.
\end{equation}
Since $\nabla \Phi(x_{t-1})=\nabla_{x} \hat{L}(x_{t-1}, y^{\star}(x_{t-1}))$, we have
\begin{equation}
    \|\nabla \Phi(x_{t-1})-\nabla_{x} \hat{L}(x_{t-1}, y_{t-1})\|^2 \leq l^2\|y^{\star}(x_{t-1})-y_{t-1}\|^2=l^2 \theta_{t-1}.
\end{equation}

Recall \eqref{Phi_descent_Lemma} in Lemma \ref{auxilliary lemma 1}, we incorporate \eqref{numerical condition} to get the desired lemma.

\end{proof}

\textbf{Proof of Theorem \ref{thm: DP-SGDA utility}: } 
\begin{proof}
    We define $$\gamma=1-1 / 2 \kappa+4 \kappa^3 l
^2 \eta_{x}^2,$$ and perform \eqref{Lemma:theta_t} in our Lemma \ref{auxilliary lemma 2} recursively. The following inequality is given: 
\begin{equation}
\label{theta_recursive}
    \theta_t \leq \gamma^t \theta_0+4 k^3 \eta_x^2 \sum_{j=0}^{t-1} \gamma^{t-1-j}\|\nabla \Phi(x_j)\|_2^2+[(2k^3 \eta_x^2+\frac{2}{l^2}) \frac{\mathcal{B}^2}{m}+2 k^3 \eta_x^2 \mathbb{E}\|\xi_{t-1}\|_2^2+\frac{2}{l^2} \mathbb{E}\|\zeta_{t-1}\|_2^2](\sum_{j=0}^{t-1} \gamma^{t-1-j}).
\end{equation}
Recall that $\theta_0 \leq \Lambda^2$,  \eqref{theta_recursive} becomes:
\begin{equation}
\label{new_theta_recursive}
    \theta_t \leq \gamma^t \Lambda^2+4 k^3 \eta_x^2 \sum_{j=0}^{t-1} \gamma^{t-1-j}\|\nabla \Phi(x_j)\|_2^2+[(2k^3 \eta_x^2+\frac{2}{l^2}) \frac{\mathcal{B}^2}{m}+2 k^3 \eta_x^2 \mathbb{E}\|\xi_{t-1}\|_2^2+\frac{2}{l^2} \mathbb{E}\|\zeta_{t-1}\|_2^2](\sum_{j=0}^{t-1} \gamma^{t-1-j}).
\end{equation}

We plug \eqref{new_theta_recursive} into \eqref{numerical condition} in Lemma \ref{Lemma: numerical one} and have the following:

\begin{equation}
\label{Phi_numerical_descent}
    \begin{aligned}
\mathbb{E}[\Phi(x_t)] &\leq 
\mathbb{E}[\Phi(x_{t-1})]-\frac{7}{16} \eta_x \mathbb{E}\|\nabla \Phi(x_{t-1})\|_2^2+(l+lk) \eta_x^2(\frac{\sigma^2}{m}+\mathbb{E}\|\xi_{t-1}\|_2^2) \\
& \quad+\frac{9}{16} \eta_x l^2[\gamma^{t-1}\Lambda^2+4 k^3 \eta_x^2 \sum_{j=0}^{t-2} \gamma^{t-2-j}\|\nabla \Phi(x_j)\|_2^2] \\
& \quad+\frac{9}{16} \eta_x l^2[(2 k^3 \eta_x^2+\frac{2}{l^2}) \frac{\mathcal{B}^2}{m}+2 k^3 \eta_x^2 \mathbb{E}\|\xi_{t-1}\|_2^2+\frac{2}{l^2} \mathbb{E}\|\zeta_{t-1}\|_2^2](\sum_{j=0}^{t-2} \gamma^{t-2-i}).
\end{aligned}
\end{equation}

Take the sum of \eqref{Phi_numerical_descent} over $t=1,2,\ldots,T+1$:
\begin{equation}
\label{Phi_sum}
    \begin{aligned}
\mathbb{E}[\Phi(x_{T+1})] &\leq 
\mathbb{E}[\Phi(x_{0})]-\frac{7}{16} \eta_x \sum_{t=0}^{T} \mathbb{E}\|\nabla \Phi(x_{t-1})\|_2^2+(l+l\kappa) \eta_x^2\frac{(T+1)\mathcal{B}^2}{m}+(l+l\kappa) \eta_x^2\sum_{t=1}^{t=T+1}\mathbb{E}\|\xi_{t-1}\|_2^2+\frac{9 \eta_{x} l^2 \Lambda^2}{16}(\sum_{t=0}^T \gamma^t)\\
&\quad+\frac{9 \eta_{x}^3 l^2 \kappa^3}{4}(\sum_{t=1}^{T+1} \sum_{j=0}^{t-2} \gamma^{t-2-j}\|\nabla \Phi(x_j)\|^2)+\frac{9}{16} \eta_x l^2[(2 \kappa^3 \eta_x^2+\frac{2}{l^2}) \frac{\mathcal{B}^2}{m}](\sum_{t=1}^{T+1} \sum_{j=0}^{t-2} \gamma^{t-2-j})\\
&\quad+[ \frac{9}{8}\eta_x^3\kappa^3l^2(\sum_{t=1}^{T+1} \sum_{j=0}^{t-2} \gamma^{t-2-j}\mathbb{E}\|\xi_{t-1}\|_2^2)+\frac{9\eta_x}{8}(\sum_{t=1}^{T+1} \sum_{j=0}^{t-2} \gamma^{t-2-j}  \mathbb{E}\|\zeta_{t-1}\|_2^2)].
\end{aligned}
\end{equation}

Since $\eta_x=\frac{1}{16(\kappa+1)^2 l}$, $\text { we have } \gamma \leq 1-\frac{1}{4 \kappa} \text { and } \frac{9 \eta_{x}^3 l^2 \kappa^3}{4} \leq \frac{9 \eta_x}{1024 \kappa} \text { and } \frac{2 \sigma^2 \kappa^3 \eta_{\times}^2}{m} \leq \frac{\sigma^2}{l^2m}$ \citep{lin2020gradient}.\\
This suggests that $\sum_{t=0}^T \gamma^t \leq 4 \kappa$. Therefore, we can see that:
\begin{equation}
\label{double_sum_1}
    \sum_{t=1}^{T+1} \sum_{j=0}^{t-2} \gamma^{t-2-j} \mathbb{E}[\|\nabla \Phi(x_j)\|^2]\leq 4 \kappa(\sum_{t=0}^T \mathbb{E}[\|\nabla \Phi(x_t)\|^2]).
\end{equation}

\begin{equation}
\label{double_sum_2}
    \sum_{t=1}^{T+1} \sum_{j=0}^{t-2} \gamma^{t-1-j}\leq 4 \kappa(T+1).
\end{equation}

Putting \eqref{double_sum_1} and \eqref{double_sum_2} with \eqref{Phi_sum}, we yield that:

\begin{equation}
  \begin{aligned}
\mathbb{E}[\Phi(x_{T+1})] &\leq \Phi(x_0)-\frac{103 \eta_x}{256}(\sum_{t=0}^T \mathbb{E}\|\nabla \Phi(x_t)\|_2^2)+\frac{9 \eta_x \kappa l^2 \Lambda^2}{4}+\frac{\eta_x \mathcal{B}^2(T+1)}{16 \kappa m} \\
& +\frac{27 \eta_k  \mathcal{B}^2 \kappa}{4m}(T+1)+(\frac{\eta_x }{16 \kappa }+\frac{9 \eta_x \kappa}{2}) \sum_{t=0}^T \mathbb{E}\|\xi_t\|_2^2+\frac{\eta_x}{8 k} \cdot \sum_{t=0}^T \mathbb{E} \|\zeta_t\|_2^2. \\
\end{aligned}  
\end{equation}

Rearranging the terms we have:
\begin{equation}
    \begin{aligned}
\frac{103 \eta_x}{256}\left(\sum_{t=0}^T \mathbb{E}\|\nabla \Phi(x_t)\|_2^2 \right) &\leq \Phi(x_0)-\mathbb{E}[\Phi(x_{t+1})]+\frac{9 \eta_x \kappa l^2 \Lambda^2}{4}+\frac{\eta_k \mathcal{B}^2(T+1)}{16 \kappa m} \\
& \quad+\frac{27 \eta_x  \mathcal{B}^2 k}{4m}(T+1)+\max  \{ \frac{9 \eta_x \kappa}{2}+\frac{\eta_x}{16 k},\frac{\eta_x}{8k}\}[\sum_{t=0}^T \mathbb{E}\|\xi_t\|_2^2+\mathbb{E}\|\zeta_t\|_2^2]. \\
\end{aligned}
\end{equation}

Therefore,
\begin{equation}
    \begin{aligned}
\left(\sum_{t=0}^T \mathbb{E}\|\nabla \Phi(x_t)\|_2^2 \right) &\leq \frac{256}{103 \eta_x}[\Phi(x_0)-\mathbb{E}[\Phi(x_{t+1})]]+\frac{576}{103} \kappa l^2 \Lambda^2+\frac{16\mathcal{B}^2(T +1)}{103 \kappa m} \\
& \quad+\frac{1728 \mathcal{B}^2}{103m} \kappa(T+1)+\frac{128}{103} \max \{9 \kappa+\frac{1}{8 \kappa},\frac{1}{4\kappa}\}[\sum_{t=0}^T \mathbb{E}\|\xi_t\|_2^2+\mathbb{E}\|\zeta_t\|_2^2]. \\
\end{aligned}
\end{equation}

Denote $\Delta_{\Phi}=\Phi(\mathrm{x}_0)-\min _{\mathrm{x}} \Phi(\mathrm{x})$, we have:
\begin{equation}
    \begin{aligned}
        \frac{1}{T+1}\left(\sum_{t=0}^T\mathbb{E}\|\nabla \Phi(x_t)\|_2^2\right) &\leq\frac{256 \Delta_\Phi}{103 \eta_x(T+1)}+\frac{576}{103} \frac{\kappa l^2 \Lambda^2}{T+1}+\frac{16\mathcal{B}^2}{103 \kappa m}+\frac{1728 \mathcal{B}^2 k}{103m} \\
& \quad+\frac{128}{103} \max \{9 \kappa+\frac{1}{8 \kappa},\frac{1}{4\kappa}\} \cdot \frac{d \log ({1}/{\delta}) \max \{G_w^2, G_v^2\}}{ n^2 \epsilon^2} (T+1) \\
&\leq\frac{3\Delta_\Phi}{\eta_x(T+1)}+\frac{6\kappa l^2 \Lambda^2}{T+1}+\frac{17\mathcal{B}^2 k}{m} \\
& \quad+2\max \{9 \kappa+\frac{1}{8 \kappa},\frac{1}{4\kappa}\} \cdot \frac{d \log ({1}/{\delta}) \max \{G_w^2, G_v^2\}}{ n^2 \epsilon^2} (T+1).
    \end{aligned}
\end{equation}

Taking $T \asymp n \epsilon \sqrt{\frac{(\frac{3 \Delta_\Phi}{\eta_x}+6 k l
^2 \Lambda^2)}{2 \max \{9 \kappa+\frac{1}{8 \kappa},\frac{1}{4\kappa}\} d\log ({1}/{\delta})\{G_w^2, G_\nu^2\}}}$ and $m =O(\frac{n\epsilon}{\sqrt{d\log(\frac{1}{\delta})}}),$
\begin{equation}
    \frac{1}{T+1}(\sum_{t=0}^T \mathbb{E}\|\nabla \Phi(x_t)\|_2^2)=O(\frac{\sqrt{d \log ({1}/{\delta})}}{n \epsilon}). 
\end{equation}

\end{proof}

\begin{proof}[{\bf Proof of Theorem \ref{thm_low:1}}]
    We first recall a lemma on the lower bound of empirical risk minimization with non-convex loss in $(\epsilon, \delta)$-DP. 
    \begin{lemma}[Theorem 2 in \cite{bassily2023differentially}]\label{lemma:low_0}
              Given $n, \epsilon=O(1)$, $2^{-\Omega(n)}\leq \delta\leq 1/n^{1+\Omega(1)}$, there exists an $O(1)$-Lipschtz, $O(1)$-smooth (convex) loss $\tilde{L}: \mathbb{R}^{d}\times \mathcal{Z}\mapsto \mathbb{R}$  and a dataset $D$ of $n$ samples such as for any $(\epsilon, \delta)$-DP algorithm with output $x^{\text{priv}}$ satisfies 
    \begin{equation*}
   \|\nabla \tilde{L}_S(x^{\text{priv}})\|^2\geq\Omega(\min\{1, \frac{d\log ({1}/{\delta})}{n^2\epsilon^2}\}),  
    \end{equation*} 
    where $\tilde{L}_S(x)=\frac{1}{n}\sum_{z\in S}\tilde{L}(x; z)$. 
    \end{lemma}
We consider the loss function $\hat{L}(x, y; z)=\tilde{L}(x; z)-\frac{1}{2}\|y\|^2$ in \eqref{emperical minimax} and $\mathcal{Y}$ as the unit ball, where $\tilde{L}(x; z)$ is the loss in Lemma \ref{lemma:low_0}. We can see that the loss $L$ satisfies Assumption \ref{NCSC}-\ref{smooth} with $G, l=O(1)$ for all $y\in \mathcal{Y}$. Moreover, we can easily see $\|\nabla \Phi(x)\|^2=\|\nabla \tilde{L}_S(x)\|^2 \geq \Omega(\min\{1, \frac{d_1\log ({1}/{\delta})}{n^2\epsilon^2}\})$. This holds for all $d_1, d_2$. Thus we can get the final result.

\end{proof}
\begin{proof}[{\bf Proof of Theorem \ref{thm_low:2}}]
    We consider the case where $S_1=\cdots=S_{d_2}=\frac{n}{d_2}$, and $S_i=\{(z_s, t_s)\}_{z_s\in \tilde{S}}$, where $\tilde{S}$ is the dataset in Lemma \ref{lemma:low_0} whose size is $\frac{n}{d_2}$ and 
    $t_s\in [0, B]$ is the label for any positive number $B$. We denote $F(x; z)=\tilde{F}(x; z_s)+t_s$, where is the loss in Lemma \ref{lemma:low_0}. Thus, by Lemma \ref{lemma:y*} we have 
    \begin{equation*}
        \|\nabla \Phi(x)\|^2=\|\sum_{i=1}^{d_2} \lambda_i^* \nabla F_{S_i}(x)\|^2= \|\nabla F_{S_1}(x)\|^2 \geq \Omega(\min\{1, \frac{d_1\log ({1}/{\delta})}{(n/d_2)^2\epsilon^2}\})= \Omega(\min\{1, \frac{d_2^2 d_1\log ({1}/{\delta})}{n^2\epsilon^2}\}). 
    \end{equation*}
This holds for all $d_1, d_2$. Thus we can get the final result.  
    
\end{proof}

\begin{proof}[{\bf Proof of Theorem \ref{thm:DP guarantee}}]

We first introduce a useful technical lemma concerning the $l_2-$sensitivity of our private estimator $v_{r+1}$ and $y_{r+1}$ for $r \leq R$.
\begin{lemma}[$l_2$-sensitivity of $v_{r+1}$]
    In Algorithm \ref{alg:Framwork}, when $r\% T=0$,  $v_{r+1}$ has  $ l_2$-sensitivity $\frac{2 C_1}{m}$. Furthermore, when $r\% T\neq 0$, given the outputs of the previous mechanisms $\{x_{r'},y_{r'}, \widetilde{v}_{r'}\}_{r^{\prime}-T+1 \leq r^{\prime} \leq r}$, $v_{r+1}$ has $l_2$-sensitivity $\frac{2 C_{2,{r}}}{m}$.
\end{lemma}

\begin{proof}

When $r\% T=0$, the $l_2$-sensitivity of $v_1= d_0$ for adjacent local datasets $D$ and $D^{\prime}$ can be bounded as

$$
\| \frac{1}{m} \sum_{j=1}^m \text{Clipping}(\nabla_{x} f(x_{r}, y_{r+1}; z_{r}^j),C_1)-\frac{1}{m} \sum_{j=1}^m \text{Clipping}(\nabla_{x} f(x_{r}, y_{r+1}; z_{r}^{'j}),C_1) \| \leq \frac{2C_{1}}{m}.
$$


When $r\% T \neq 0$, the $l_2$-sensitivity of $v_{r+1}=\mathbf{d_r}+\widetilde{v}_{r}$ for adjacent local datasets $D$ and $D^{\prime}$ can be bounded as
$$
\begin{aligned}
&\| \frac{1}{m} \sum_{j=1}^m \text{Clipping} (  \nabla_{x} f(x_{r}, y_{r+1}; z_{r}^j)-\nabla_{x} f(x_{r-1}, y_{r}; z_{r-1}^j),C_{2, r}) -\frac{1}{m} \sum_{j=1}^m \text{Clipping} (  \nabla_{x} f(x_{r}, y_{r+1}; z_{r}^{'j})-\nabla_{x} f(x_{r-1}, y_{r}; z_{r-1}^j),C_{2, r}) \| \\
&\leq \frac{2 C_{2, r}}{m}.
\end{aligned}
$$
This finishes the proof.
\end{proof}

\begin{lemma}[$l_2$-sensitivity of $y_{r+1}$]
    Consider Algorithm \ref{alg:DPGDSC}, under Assumption \ref{NCSC}, the $\ell_2$-sensitivity of $y'_{T_2}$ is bounded by $\frac{2C_0^2+\beta M}{n \mu}$ if $\eta_{y_i}=\frac{1}{\mu t}$. 
\end{lemma}
\begin{proof}
    We first introduce the following lemma on the stability of stochastic gradient descent for strongly convex loss. 
      \begin{lemma}\label{hardt}
        [Theorem 3.10 in \cite{hardt2016train} ] Assume the loss function $f(\cdot, z)\leq   M$ is   $\mu$-strongly concave,  $\beta$-smooth, and has gradients bounded by $L$ for all $z$. Let $D$ and $D^{\prime}$ be two samples of size $n$ differing in only a single element. Denote $y_t$ and $y^{\prime}_t$ as the outputs of the projected stochastic ascent method with stepsize $\eta_i=\frac{1}{\mu i}$ on datasets $D$ and $D'$ respectively at the $t$-th iteration, then  we have
    $$
    \|y_i-y^{\prime}_i \| \leq \frac{2 L^2+ \beta M}{\mu n}. 
    $$
     \end{lemma}
         Note that the original form of Lemma \ref{hardt} is for SGD while in  Algorithm \ref{alg:DPGDSC} we have the clipped version. Since we have $\text{Clipping}(\nabla_{y} f({x}, {y}^\prime_{i}; z_{i}^j), C_0)=\Pi_{\mathbb{B}}(\nabla_{y} f({x}, {y}^\prime_{i}; z_{i}^j))$, where $\Pi_{\mathbb{B}}$ is the projection onto the ball with radius $C_0$. For any $y, \tilde{y}$ we have 
         \begin{equation*}
             \|\Pi_{\mathbb{B}}(\nabla_{y} f({x}, y; z_{i}^j))-\Pi_{\mathbb{B}}(\nabla_{y} f({x}, \tilde{y}; z_{i}^j))\|\leq \|\nabla_{y} f({x}, y; z_{i}^j)-\nabla_{y} f({x}, \tilde{y}; z_{i}^j)\|\leq \ell_y \|y- \tilde{y}\|. 
         \end{equation*}
         Moreover we have $\|\Pi_{\mathbb{B}}(\nabla_{y} f({x}, y; z_{i}^j))\|\leq C_0$. Assumption \ref{NCSC} guarantees that $f({x}, \cdot; z_{i}^j))$ is $\mu$-strongly concave and bounded by $M$ . By using the same proof as in  Lemma \ref{hardt}. We can easily see that the sensitivity of the output in Algorithm \ref{alg:DPGDSC} is upper bounded by $\frac{2C_0^2+\beta M}{\mu n}$ if $\eta_{y_i}\leq \frac{1}{\mu i}$. 
\end{proof}

Next we consider the proof of Theorem \ref{thm:DP guarantee}. Denote 

\begin{equation}
    g_r(D_r)= \begin{cases} \frac{1}{2C_1 m} \sum_{j=1}^m \text{Clipping}(\nabla_{x} f({x}_r, {y}_{r+1}, z_{r}^j), C_1) & \text { if } r \% T=0, \\ 
\frac{1}{2C_{2,r}m} \sum_{j=1}^m \text{Clipping}(\nabla_{x} f({x}_r, {y}_{r+1}, z_r^j)-\nabla_{x} f({x}_{r-1}, {y}_{r}, z_{r-1}^j), C_{2,r}) & \text { otherwise. }\end{cases}
\end{equation}
We can see that $\Delta(g_r)=\frac{1}{m}$. By Lemma \ref{Lemma Abadi} (b) and Lemma \ref{amplification} , the log moment of the composite mechanism  $\mathcal{A}^{x}=(\mathcal{A}_1^{\mathrm{x}}, \cdots, \mathcal{A}_R^{\mathrm{x}})$ can be bounded as follows:
\begin{equation}
    \alpha_{\mathcal{A}^x}(\lambda) \leq \frac{m^2 R \lambda^2}{Tn^2 \tilde{\sigma}_{\mathrm{x_1}}^2}+(R-\frac{R}{T})\frac{m^2\lambda^2}{n^2 \tilde{\sigma}_{\mathrm{x_2}}^2}.
\end{equation}
 where $\tilde{\sigma}_{\mathbf{x_1}}=m\sigma_{\mathbf{x_1}} / 2C_1,\tilde{\sigma}_{\mathbf{x_2}}=m\sigma_{\mathbf{x_2}} / 2 C_{2,r}$

Similary,  the log moment of the mechanism with respect to variable $y$ can be bounded as:

\begin{equation}
    \alpha_{\mathcal{A}^y}(\lambda) \leq \frac{R \lambda^2}{\tilde{\sigma}_{\mathrm{y}}^2},
\end{equation}

where $\tilde{\sigma}_{y}=\frac{\sigma_{y}\mu n}{(2C_0^2+\beta M)}.$

By composition theorem, $\alpha_\mathcal{A}(\lambda)\leq \alpha_{\mathcal{A}^x}+\alpha_{\mathcal{A}^y}$:
\begin{equation}
    \alpha_{\mathcal{A}}(\lambda) \leq \frac{m^2 R \lambda^2}{Tn^2 \tilde{\sigma}_{\mathrm{x_1}}^2}+(R-\frac{R}{T})\frac{m^2\lambda^2}{n^2 \tilde{\sigma}_{\mathrm{x_2}}^2}+\frac{R \lambda^2}{\tilde{\sigma}_{\mathrm{y}}^2}.
\end{equation}

By Lemma \ref{Lemma Abadi} (a), to guarantee $A$ to be $(\epsilon, \delta)$-differentially private, it suffices that
$$
\begin{aligned}
    &\frac{m^2 R \lambda^2}{Tn^2 \tilde{\sigma}_{\mathrm{x_1}}^2} \leq \frac{\lambda \epsilon}{4}, (R-\frac{R}{T})\frac{m^2\lambda^2}{n^2 \tilde{\sigma}_{\mathrm{x_2}}^2} \leq \frac{\lambda \epsilon}{4}, \frac{R \lambda^2}{\tilde{\sigma}_{\mathrm{y}}^2} \leq \frac{\lambda \epsilon}{4}, \exp (-\frac{\lambda \epsilon}{4}) \leq \delta,\\ 
    & \lambda \leq \tilde{\sigma}_{\mathbf{x_1}}^2 \log (\frac{n}{m {\sigma}_{\mathbf{x_1}}}), \lambda \leq \tilde{\sigma}_{\mathbf{x_2}}^2 \log (\frac{n}{m{\sigma}_{\mathbf{x_2}}}) \text { and } \lambda \leq \tilde{\sigma_{y}^2} \ln \frac{1}{ \tilde{\sigma_y}}.
\end{aligned}
$$

It is now easy to verify that when $\epsilon=c_4 m^2 T / n^2$, we can satisfy all these conditions by setting
$$
{\sigma}_{\mathbf{x_1}} \geq \frac{c_5 \sqrt{\frac{R}{T} \log (1 / \delta)}}{n \epsilon}, {\sigma}_{\mathbf{x_2}} \geq \frac{c_6\sqrt{(R-\frac{R}{T}) \log (1 / \delta)}}{n \epsilon} \text{\quad and} \quad \sigma_{y} \geq c_7(2C_0^2+\beta M)\frac{\sqrt{R \log (1 / \delta)}}{n\epsilon}
$$
for some explicit constants $c_4, c_5$ and $c_6$ and $c_7$. The proof is complete.
\end{proof}

\textbf{Proof of Remark \ref{DP parameter remark}:}
Given $\delta=\frac{1}{n^2}$, the fourth inequality can be reformulated as $\lambda \geq \frac{8\log (n)}{\epsilon}$. Hence, by setting $\sigma_{x_1}=\frac{4 \sqrt{\frac{R}{T
   }\log (1 / \delta)}}{n \epsilon}, \sigma_{x_2}=\frac{4 \sqrt{R\log (1 / \delta)}}{n \epsilon} \text{ and }\sigma_{y}=\frac{4(2C_0^2+\beta M) \sqrt{R\log (1 / \delta)}}{ n\epsilon}$, the first inequality becomes $\lambda \leq \frac{8\log (n)}{\epsilon}$. Therefore, $\lambda=\frac{8 \log (n)}{\epsilon}$. Under $m=\max (1, n \sqrt{\epsilon /(8T)})$ and $\epsilon \leq 1$, such $\lambda$ satisfies the inequalities on the second row. The proof is complete.

\begin{proof}[{\bf Proof of Theorem \ref{thm: PrivateDiff utility}}]

We give some auxilliary lemmas and definitions here, which will be later used in the main proof of Theorem \ref{thm:DP guarantee}.
\begin{lemma}{For $l-$smooth function $\hat{L}(x,y;D)$, the spectral norm $\|\nabla_{\mathbf{xy}}\hat{L}(x,y)\|_2$ satisfies that:}
\label{hession upper bound}
\begin{equation}
   \|\nabla_{\mathbf{xy}}\hat{L}(x,y)\|_2\leq l. 
\end{equation}
\end{lemma}
\begin{proof}
    Let $u, v \in \mathbb{R}^d$ be arbitrary vectors, and define $\psi(t)=\langle\nabla_{y}\hat{L}(x+t u,y)-\nabla_{y}\hat{L}(x,y), v\rangle$.By Assumption \ref{smooth}, $\hat{L}(x,y;D)$ is $l$-smooth, which is equivalent to see that $\psi(t) \leq l t\|u\|\|v\|$.\\
   We can write $\psi^{\prime}(0)=\lim _{t \rightarrow 0}(\psi(t)-\psi(0)) / t = \langle\nabla_{\mathbf{xy}}\hat{L}(x,y) u, v\rangle \leq l\|u\|\|v\|$. \\ 
    Therefore, the spectral norm of $\nabla_{\mathbf{xy}}\hat{L}(x,y; D)$ is upper bounded by $l$.
\end{proof}

\begin{lemma}\label{SGD convergence}
[Proposition 1 in \cite{rakhlin2011making}]
Let $\vartheta \in(0,1 / e)$ and assume $T \geq 4$. Suppose $F(w)$ is $\lambda$-strongly convex over a convex set $\mathcal{W}$, and the stochastic gradient $\left\|\hat{\mathbf{g}}_t\right\|^2 \leq G^2$ with probability 1 . Then if we pick $\eta_t=1 / \lambda t$, the iterates in SGD holds with probability at least $1-\vartheta$ that for any $t \leq T$,
$$
\left\|\mathbf{w}_t-\mathbf{w}^*\right\|^2 \leq \frac{(624 \log (\log (T) / \vartheta)+1) G^2}{\lambda^2 t}.
$$

\end{lemma}

We can easily see  $\hat{L}({x}, \cdot; D))$ is $\mu$-strongly concave over a convex set $\mathcal{Y}$ by Assumption \ref{NCSC} . Recall that $f(x,\cdot;z_i)$ is $G_y$-Lipschitz in Assumption \ref{lipschitz cts}, it always holds that
$\left\|\frac{1}{m} \sum_{j=1}^m \nabla_{y} f({x}, {y}^\prime_{i}; z_{i}^j)\right\|^2\leq G_y^2 \leq G^2$. It can be drawn that the iterates of our Algorithm \ref{alg:DPGDSC} obey the following relationship:

\begin{equation}
\label{SGA convergence}
  \left\|\mathbf{y^\prime}_i-\mathbf{y}^*\right\|^2 \leq \frac{(624 \log (\log (T_2) / \delta)+1) G^2}{\mu^2 i}
\end{equation}

In the following we will show that with some values of $C_0, C_1, C_2, C_3$, there will be no clipping in the algorithm. 

\begin{lemma}\label{lemma:noclipping}
Consider the parameters in Theorem \ref{thm: PrivateDiff utility} with $C_3\geq \frac{50lG}{{\mu}}\sqrt{\frac{(\log(\log(T_2))/\vartheta)+1}{T_2}}$. There is no clipping with a probability of at least $1-\vartheta$ for every iteration.
\end{lemma}

\begin{proof}
By the Lipschizt assumption we can see taking $C_1=G_x$ and $C_0=G_y$ then there will be clipping at step 6 in Algorithm \ref{alg:Framwork} and step 3 in Algorithm \ref{alg:DPGDSC}. Next we will show an upper bound of $d_r$ in step 7 of Algorithm \ref{alg:Framwork}. Noted that 

    $$
\begin{aligned}
& \quad \frac{1}{m}\|\sum_{j=1}^m \nabla_{x}f(x_r, y_{r+1};z_{r}^j)-\sum_{i=1}^m \nabla_{x}f(x_{r-1}, y_r;z_{r-1}^j)\|_2 \\
& \leq \frac{1}{m} \sum_{j=1}^m\{\| \nabla_{x}f(x_r, y_{r+1};z_{r}^j)-\nabla_{x}f(x_r, y^*(x_r);z_{r}^j)\|_2+\| \nabla_{x}f(x_r, y^*(x_r);z_{r}^j)-\nabla_{x}f(x_{r-1}, y^*(x_{r-1});z_{r-1}^j) \|_2\\
& \quad+\| \nabla_{x}f(x_{r-1}, y^*(x_{r-1});z_{r-1}^j)-\nabla_{x}f(x_{r-1}, y_r; z_{r-1}^j)  \|_2\}\\
& \stackrel{(a)}\leq \frac{1}{m}\sum_{i=1}^m \{\|\nabla_{x y} f(x_r, \tilde{y})\|_2\|y_{r+1}-y^*(x_r)\|_2+(l+ \kappa l)\|x_r-x_{r-1}\|_2\} \\
& \quad+\|\nabla_{x y} f(x_{r-1}, \hat{y})\|_2\|y^*(x_{r-1})-y_r\|_2\} \\
& \stackrel{(b)}\leq (l+ \kappa l)\|x_r-x_{r-1}\|+ \frac{50lG}{{\mu}}\sqrt{\frac{(\log(R \log(T_2))/\vartheta)+1}{T_2}}\\
& \stackrel{(c)} \leq C_{2,r}.
\end{aligned}
$$
We get (a) directly from the mean value theorem. $\tilde{y}$ is some vector lying on the segment joining $y_{r+1}$ and $y^*(x_r)$ and $\hat{y}$ is a vector lying on the segment joining $y_{r}$ and $y^*(x_{r-1})$.

(b) is the joint effect of Lemma \ref{hession upper bound} and Lemma \ref{SGD convergence}. 

(c) is due to the fact that we set $C_{2, r}=C_2\|x_{r}-x_{r-1}\|$ $+50\kappa G\sqrt{\frac{(\log( R \log(T_2))/\vartheta)+1}{T_2}}$.

Thus, we if we take $C_2=\ell+\kappa \ell$ and $C_3= 50\kappa G\sqrt{\frac{(\log( R\log(T_2))/\vartheta)+1}{T_2}}$, then will probability at least $1-\vartheta$ we have $d_r\leq  (l+ \kappa l)\|x_r-x_{r-1}\|+ \frac{50lG}{{\mu}}\sqrt{\frac{(\log(R \log(T_2))/\vartheta)+1}{T_2}}$ for every $r$. 
\end{proof}

In the following we will always assume that Lemma \ref{lemma:noclipping} holds. We first present some lemmas in the convenience of utility analysis.

\begin{lemma}{ Suppose that Assumption \ref{smooth} holds, with $\eta_x \leq \frac{1}{2(l+\kappa l)}$, Our PrivateDiff algorithm satisfies:}
\begin{equation}
\label{alg2: desent lemma for phi}
    \mathbb{E}\|\nabla \Phi(x_r)\|_2^2 \leq \frac{2}{\eta_t}[\Phi(x_{r-1})-\Phi(x_r)]-\frac{1}{2 \eta_x^2} \mathbb{E}\|x_r-x_{r-1}\|_2^2+2 \mathbb{E}\|\tilde{v}_r-\nabla \Phi(x_{r-1})\|_2^2.
\end{equation}   
\end{lemma}

\begin{proof}
    From $(l+kl)$-smoothness of $\Phi$ by Lemma \ref{phi smooth}, we have
$$
\begin{aligned}
\Phi(x_r) \leq & \Phi(x_{r-1})+\langle\nabla \Phi(x_{r-1}), x_r-x_{r-1}\rangle+\frac{l+\kappa l}{2}\|x_r-x_{r-1}\|^2 \\
= & \Phi(x_{r-1})+\frac{1}{\eta_x}(-\frac{\eta_x^2}{2}\|\nabla \Phi(x_{r-1})\|^2-\frac{1}{2}\|x_r-x_{r-1}\|^2+\frac{1}{2}\|x_r-x_{r-1}+\eta_x \nabla \Phi(x_{r-1})\|^2) \\
& +\frac{l+\kappa l}{2}\|x_r-x_{r-1}\|^2 \\
= & \Phi(x_{r-1})-\frac{\eta_x}{2}\|\nabla \Phi(x_{r-1})\|^2-(\frac{1}{2 \eta_x}-\frac{l+\kappa l}{2})\|x_r-x_{r-1}\|^2+\frac{\eta_x}{2}\|\frac{1}{\eta_x}(x_{r-1}-x_r)-\nabla f(x_{r-1})\|^2 .
\end{aligned}
$$

Take expectation on both hand sides and recall $\eta_x \leq \frac{1}{2(l+\kappa l)}$, we prove the given lemma by arranging some terms.
\end{proof}

\begin{lemma}{Suppose Assumption \ref{smooth}, \ref{stochastic gradient variance} holds, we establish the following result:}
\label{Lemma:vr-phi}
    \begin{equation}
    \label{eq:vr-phi}
    \mathbb{E}\|\tilde{v}_r-\nabla \Phi(x_{r-1}) \|_2^2 \leq \sigma_{x_1}^2 C_1^2 d+\sigma_{x_2}^2 d \sum_{r^{\prime}=T(r)+2}^r C_{2, r^\prime}^2 + l\mathbb{E}\| y_r-y^*(x_{r-1})\|_2^2+\frac{\mathcal{B}^2}{m}, 
\end{equation}
where $T(r)$  is the integer that satisfies $ r+1-T \leq T(r)<r $ and $ T(r) \% T=0.
$
\end{lemma}

\begin{proof}
    By the update rule of $x_r$ in the step 9 of Algorithm \ref{alg:Framwork}, 
    \begin{equation}
    \label{vr-phi_elementary}
        \mathbb{E}\|\frac{1}{\eta}(x_{r-1}-x_r)-\nabla \Phi(x_{r-1})\|_2^2=\mathbb{E}\|\tilde{v}_r-\nabla \Phi(x_{r-1}) \|_2^2.
    \end{equation}
    Recall that $v_{r+1}=\mathbf{d_r}+\widetilde{v}_{r}$ and $\widetilde{v}_{r+1}=v_{r+1}+\xi_{x_{r+1}}$, \eqref{vr-phi_elementary} becomes:
    \begin{equation}
    \label{vr-phi_1}
        \mathbb{E}\|\tilde{v}_{r-1}+\xi_r+\frac{1}{m}\sum_{j=1}^{m
}\nabla_{x} f(x_{r-1}, y_r;z_{r-1}^{j})-\frac{1}{m}\sum_{j=1}^{m
}\nabla_{x} f(x_{r-2}, y_{r-1};z_{r-2}^{j})-\nabla \Phi(x_{r-1}) \|_2^2
    \end{equation}
As the noise $\xi_r$ is sampled from a zero-mean normal distribution, \eqref{vr-phi_1} is equivalent to:
\begin{equation}
\label{vr-phi_2}
    \mathbb{E}\|\tilde{v}_{r-1}-\frac{1}{m}\sum_{j=1}^{m
}\nabla_{x} f(x_{r-2}, y_{r-1};z_{r-2}^{j})+\frac{1}{m}\sum_{j=1}^{m
}\nabla_{x} f(x_{r-1}, y_r;z_{r-1}^{j})-\nabla \Phi(x_{r-1})\|_2^2+\mathbb{E}\|\xi_r\|_2^2.
\end{equation}
We do the above procedures once again to get another form of \eqref{vr-phi_elementary}:
\begin{equation}
\begin{aligned}
&\mathbb{E}\| \tilde{v}_{r-2}+\frac{1}{m}\sum_{j=1}^{m
}\nabla_{x} f(x_{r-2}, y_{r-1};z_{r-2}^{j})-\frac{1}{m}\sum_{j=1}^{m
}\nabla_{x} f(x_{r-3}, y_{r-2};z_{r-3}^{j})-\frac{1}{m}\sum_{j=1}^{m
}\nabla_{x} f(x_{r-2}, y_{r-1};z_{r-2}^{j})\\
&\quad+\frac{1}{m}\sum_{j=1}^{m
}\nabla_{x} f(x_{r-1}, y_r;z_{r-1}^{j})-\nabla \Phi(x_{r-1})\|_2^2+\mathbb{E}\| \xi_r \|_2^2+\mathbb{E}\| \xi_{r-1}\|_2^2.
\end{aligned}
\end{equation}

Inductively, \eqref{vr-phi_elementary} equals to:
\begin{equation}
\label{vr-phi-after}
    \sum_{r^{\prime}=T(r)+1}^r \mathbb{E}\|\xi_{r^{\prime}}\|_2^2+\mathbb{E}\|\frac{1}{m}\sum_{j=1}^{m
}\nabla_{x} f(x_{r-1}, y_r;z_{r-1}^{j})-\nabla \Phi(x_{r-1})\|_2^2.
\end{equation}

By the bounded variance Lemma \ref{stochastic gradient variance}, we yield the following relationship:
\begin{equation}
\begin{aligned}
    \mathbb{E}\|\tilde{v}_r-\nabla \Phi(x_{r-1}) \|_2^2 &\leq \sum_{r^{\prime}=T(r)+1}^r \mathbb{E}\|\xi_r\|_2^2+\mathbb{E}\| \nabla_{x y} L(x_{r-1}, \tilde{y})(y_r-y^*(x_{r-1})) \|_2^2+\frac{\mathcal{B}^2}{m}\\
    & \stackrel{(a)}\leq \sum_{r^{\prime}=T(r)+1}^r \mathbb{E}\|\xi_r\|_2^2 + l \mathbb{E}\| y_r-y^*(x_{r-1})\|_2^2+\frac{\mathcal{B}^2}{m}\\
&\stackrel{(b)} =\sigma_{x_1}^2 C_1^2 d+\sigma_{x_2}^2 d \sum_{r^{\prime}=T(r)+2}^r C_{2, r^\prime}^2 + l\mathbb{E} \| y_r-y^*(x_{r-1})\|_2^2+\frac{\mathcal{B}^2}{m},
\end{aligned}
\end{equation}
where $(a)$ comes from the smoothness property of loss function. $(b)$ is natural by the definition of our added noise $\xi_r$.
Therefore, we derive our lemma.
\end{proof}

\begin{lemma}
\begin{equation}
\begin{aligned}
   \frac{1}{R} \sum_{r=1}^R \mathbb{E}\|\widetilde{v}_r-\nabla \Phi(x_{r-1})\|^2 &\leq \sigma_{x_1}^2 C_1^2 d+2T \sigma_{x_2}^2 C_2^2 d \frac{1}{R} \sum_{r=1}^R\|x_{r-1}-x_{r-2}\|^2+ l\frac{1}{R} \sum_{r=1}^R\mathbb{E}\| y_r-y^*(x_{r-1})\|_2^2+\frac{\mathcal{B}^2}{m}\\
&\quad + 5000\sigma_{x_2}^2 d\frac{G^2\kappa^2T}{ T_2}(\log(\log(T_2))/\vartheta)+1).
\end{aligned}
\end{equation}
\end{lemma}

\begin{proof}
    By taking the algebraic average of \eqref{eq:vr-phi} in Lemma \ref{Lemma:vr-phi}, we yield that:
    \begin{equation}
    \label{vr-Phi_value}
       \begin{aligned}
\frac{1}{R} \sum_{r=1}^R \mathbb{E}\|\widetilde{v}_r-\nabla \Phi(x_{r-1})\|^2 &\leq\sigma_{x_1}^2 C_1^2 d+\sigma_{x_2}^2  d \frac{1}{R} \sum_{r=1}^R \sum_{r^\prime=T(r)+2}^r C_{2, r^\prime}^2 + l \frac{1}{R} \sum_{r=1}^R \mathbb{E}\| y_r-y^*(x_{r-1})\|_2^2 +\frac{\mathcal{B}^2}{m}\\
&\stackrel{(a)}\leq\sigma_{x_1}^2 C_1^2 d+2\sigma_{x_2}^2 C_2^2 d \frac{1}{R} \sum_{r=1}^R \sum_{r^\prime=T(r)+2}^r\|x_{r-1}-x_{r-2}\|^2 + l \frac{1}{R} \sum_{r=1}^R \mathbb{E}\| y_r-y^*(x_{r-1})\|_2^2 +\frac{\mathcal{B}^2}{m}\\
&\quad+2\sigma_{x_2}^2 d \frac{1}{R} \sum_{r=1}^R \sum_{r^\prime=T(r)+2}^r\frac{2500G^2\kappa^2}{ T_2}(\log(R \log(T_2))/\vartheta)+1)\\
& \stackrel{(b)}\leq \sigma_{x_1}^2 C_1^2 d+2T \sigma_{x_2}^2 C_2^2 d \frac{1}{R} \sum_{r=1}^R\|x_{r-1}-x_{r-2}\|^2+ l\frac{1}{R} \sum_{r=1}^R\mathbb{E}\| y_r-y^*(x_{r-1})\|_2^2+\frac{\mathcal{B}^2}{m}\\
&\quad + 5000\sigma_{x_2}^2 d\frac{G^2T\kappa^2}{T_2}(\log(R \log(T_2))/\vartheta)+1),
\end{aligned}
    \end{equation}
where $(a)$ is due to the Clipping radius defined in the step 7 of Algorithm \ref{alg:Framwork}. It is obvious to observe relation $(b)$ by noting the restart interval T.
\end{proof}


$\textbf{Main Proof of Theorem \ref{thm: PrivateDiff utility}}:$

By averaging the \eqref{alg2: desent lemma for phi},
we get:
\begin{equation}
\label{avergate gradient}
     \frac{1}{R}\sum_{r=1}^R\|\nabla\Phi(x_r)\|_2^2 \leq \frac{2}{R \eta_x}[\Phi(x_0)-{\Phi}(x^*)]-\frac{1}{2 \eta_x^2} \frac{l}{R} \sum_{r=1}^R \mathbb{E}\|x_r-x_{r-1}\|_2^2+2 \cdot \frac{1}{R} \sum_{r=1}^R\|\tilde{v}_r-\nabla \Phi_{(x_{r-1})}\|_2^2.
\end{equation}
Plugging \eqref{vr-Phi_value} to \eqref{avergate gradient} and letting $\eta \leq 1 /(2 \sqrt{2T \sigma_{x_2}^2 C_2^2 d})$ , it holds that
\begin{equation}
    \mathbb{E}\|\nabla \Phi (x^{\text{priv}})\|^2 \leq O\left(\frac{\Phi(x_0)-\Phi(x^*)}{\eta R}+\sigma_{x_1}^2 C_1^2 d+\frac{2l}{R} \sum_{r=1}^R\|y_r-y^*(x_{r-1})\|_2^2+\frac{\mathcal{B}^2}{m}+ \sigma_{x_2}^2 d \frac{TG^2 \kappa^2}{T_2}(\log(\log(T_2))/\vartheta)+1))\right) .
\end{equation}

Under $C_1=\Theta(G)$, we know that
\begin{equation}
    \mathbb{E}\|\nabla \Phi (x^{\text{priv}})\|^2 \leq O\left(\frac{\Phi(x_0)-\Phi(x^*)}{\eta R}+\sigma_{x_1}^2 G^2 d+\frac{2l}{R} \sum_{r=1}^R\|y_r-y^*(x_{r-1})\|_2^2+\frac{\mathcal{B}^2}{m}+ \sigma_{x_2}^2 d \frac{TG^2 \kappa^2}{T_2}(\log(R \log(T_2))/\vartheta)+1)) \right) .
\end{equation}

Suppose that $\Phi(x_0)-\Phi(x_*)=O(1)$.
Recall that $\sigma_{x_1}^2=\widetilde{\Theta}(R /(T n^2\epsilon^2))$ and $\sigma_{x_2}^2=\widetilde{\Theta}(R /(n^2\epsilon^2))$ respectively in Theorem \ref{thm:DP guarantee},  $\eta_x$, we substitue $\eta_x=\Theta(\min \{1 /2(l+\kappa l), 1 / 2\sqrt{2T \sigma_{x_2}^2 C_2^2 d}\})$ and use Lemma \ref{SGD convergence} to have the following:
$$
\begin{aligned}
\mathbb{E}\|\nabla \Phi(x^{\text{priv}})\|^2 & \leq O(\frac{1}{\eta R}+\frac{\mathcal{B}^2}{m})+\widetilde{O}(\frac{l}{T_2}+\frac{R G^2 d}{T n^2\epsilon^2}+\frac{TRd}{T_2n^2\epsilon^2}) \\
& =O(\frac{l}{R}+\frac{\mathcal{B}^2}{m})+\widetilde{O}(\frac{l}{T_2}+\frac{\sqrt{T}l\sqrt{d}}{n\epsilon \sqrt{R}}+\frac{R G^2 d}{T n^2\epsilon^2}+\frac{TRd}{T_2n^2\epsilon^2}) .
\end{aligned}
$$
Assume that $n=\Omega(\frac{G^2 \sqrt{d}}{l\epsilon})$. We define
$$
T:=\Theta(1 \vee(\frac{G^2 \sqrt{d}}{ln\epsilon})^{\frac{2}{3}} R)
$$
with $1 \leq T \leq R$.
Then, we obtain
$$
\mathbb{E}\|\nabla \Phi (x^{\text{priv}})\|^2 \leq O(\frac{l}{R}+\frac{\mathcal{B}^2}{m})+\widetilde{O}(\frac{l}{T_2}+\frac{l\sqrt{d}}{n\epsilon\sqrt{R}}+\frac{(lG d)^{\frac{2}{3}}}{(n\epsilon)^{\frac{4}{3}}}+\frac{TRd}{T_2n^2\epsilon^2}) .
$$
Finally, setting
$$
R=\widetilde{\Theta}(1 \vee \frac{l}{\varepsilon_{\mathrm{opt}}}) \vee \widetilde{\Theta}(\frac{l^2 d}{n^2\epsilon^2 \varepsilon_{\mathrm{opt}}^2}), T_2=\Theta (\frac{(n\epsilon)^{\frac{4}{3}}}{d^{\frac{2}{3}}})\vee\tilde{\Theta(TR
\cdot \frac{d^{\frac{1}{3}}}{(n\epsilon)^{\frac{2},{3}}})}
$$
where $\varepsilon_{\mathrm{opt}}:=\Theta(\frac{(lG d)^{\frac{2}{3}}}{(n\epsilon)^{\frac{4}{3}}})$. With the batch size $ m=\Theta(\frac{(n\epsilon)^{\frac{4}{3}}}{{d^{\frac{2}{3}}\log(\frac{1}{\delta})}})$, we have the desired utility bound.

Therefore, we know that the utility upper bound should be:
\begin{equation}
    \mathbb{E}\|\nabla \Phi(x^{\text{priv}})\|^2 \leq \widetilde{O}(\frac{d^{\frac{2}{3}}}{(n\epsilon)^{\frac{4}{3}}}).
\end{equation}

\end{proof}

\section{Additional Experiments}
\label{sec:add_exp}

\subsection{AUC Maximiazation}
\label{app:auc}
\subsubsection{Background}
AUC refers to the area under the Receiver Operating Characteristic (ROC) curve, generated by plotting the true positive rate (TPR) against the false positive rate (FPR) at various threshold levels. By definition, its value ranges from 0 to 1, which can be interpreted as follows.

\begin{itemize}
    \item AUC $=$ 1: The model perfectly distinguishes between the two classes.
    \item AUC $=$ 0.5: The model performs no better than random chance.
    \item AUC $<$ 0.5: The model performs worse than random guessing.
    \item A higher AUC indicates better performance.
\end{itemize}
Maximizing AUC has been shown to be equivalent to a minimax problem with auxiliary variables $a, b, v$ \citep{Yuan_2021_ICCV},

\begin{equation}
    \min _{\substack{\mathbf{w} \in \mathbb{R}^d \\(a, b) \in \mathbb{R}^2}} \max _{\alpha \in \mathbb{R^+}} f(\mathbf{w}, a, b, \alpha):=\mathbb{E}_{z}[F(\mathbf{w}, a, b, \alpha ; z)],
\end{equation}
where

\begin{equation}
\begin{split}
    F(\mathbf{w},a,b,\alpha; z) &=(1-p)(h_{\mathbf{w}}(x)-a)^2\mathbb{I}_{[y=1]} +p(h_{\mathbf{w}}(x)-b)^2\mathbb{I}_{[y=-1]} \\
            &\quad +2\alpha(p(1-p)+ p h_{\mathbf{w}}(x)\mathbb{I}_{[y=-1]}-(1-p)h_{\mathbf{w}}(x)\mathbb{I}_{[y=1]})\\
            &\quad -p(1-p)\alpha^2.
\end{split}
\end{equation}
$h_{\mathbf{w}}$ is the prediction scoring function, e.g., deep neural network, $p$ is the ratio of positive samples to all samples, $a$, $b$ are the running statistics of 
        the positive and negative predictions, $\alpha$ is the auxiliary variable derived from the problem formulation.

\subsubsection{Implementation Details}
The training settings for PrivateDiff and DP-SGDA on MNIST and Fashion-MNIST are shown in Table~\ref{tab:hparam}. Noted that learning rates, $\eta_x$ and $\eta_y$, are obtained by grid search among $\{0.02, 0.2, 2\}$. Python libraries of Pytorch \citep{paszke2019pytorch} and LibAUC \citep{yuan2023libauc, yang2022algorithmic} are used for code implementation.

\begin{table*}[h]
\centering

\begin{tabular}{lccccccccc} 
\toprule
                               & $C_1$ & $C_2$ & $T$ & $T_2$ & Batch Size & Epochs  \\ 
\midrule
DP-SGDA                            & 1     & 1     & N/A & N/A   & 2048       & 80    \\
PrivateDiff    & 1     & 1     & 2   & 3     & 2048       & 80     \\
\bottomrule
\end{tabular}
\caption{Hyperparameter Settings and Training Configurations.}
\label{tab:hparam}
\end{table*}

\subsubsection{Additional Figures and Analysis}
To better evaluate PrivateDiff, learning curves of AUC Maximization are depicted on Figure~\ref{fig:curve_mnist} and Figure~\ref{fig:curve_fmnist} for MNIST and Fahsion-MNIST, respectively. For reference, non-private learning curves are also included in Figure~\ref{fig:curve_nonprivate} to compare.

Across all datasets and privacy budgets, PrivateDiff consistently outperforms DP-SGDA in terms of AUC. PrivateDiff achieves higher and more stable AUC scores throughout the training process, regardless of the specific privacy budget or the nature of the dataset (balanced or imbalanced). In contrast, DP-SGDA exhibits significant instability, with frequent fluctuations in AUC, particularly in the earlier epochs. This instability is more pronounced in lower privacy budgets, where DP-SGDA struggles to converge, highlighting its sensitivity to the privacy-utility trade-off. The non-private performance results provide an essential baseline, showing that both methods are capable of achieving nearly perfect AUC scores when privacy constraints are removed. This confirms that the observed differences in AUC under private settings are indeed due to the privacy mechanisms implemented by each method and not due to inherent flaws in the algorithms themselves.

\subsubsection{Additional Results of Differentially Private Transfer Learning on CIFAR10}
We consider a similar setting in \cite{tramer2020differentially} to conduct transfer learning from CIFAR-100 to CIFAR-10, where CIFAR-100 data is assumed public. A resnet-20 pretrained on CIFAR-100 is differentially private finetuned on CIFAR-10. The results in Table~\ref{tab:cifar} demonstrate that PrivateDiff Minimax consistently outperforms DP-SGDA across all CIFAR-10 variants and privacy budgets.

\begin{table}[h!]
\centering
\caption{Comparison of AUC performance in DP-SGDA and PrivateDiff Minimax on various CIFAR-10 datasets.}
\begin{tabular}{lcccccccc}
\toprule
Dataset               & \multicolumn{2}{c}{Balanced CIFAR-10} & \multicolumn{2}{c}{Imbalanced CIFAR-10} & \multicolumn{2}{c}{Heavy-Tailed CIFAR-10} \\
\cmidrule(lr){2-3} \cmidrule(lr){4-5} \cmidrule(lr){6-7}
                      & DP-SGDA $\uparrow$ & PrivateDiff $\uparrow$ & DP-SGDA $\uparrow$ & PrivateDiff $\uparrow$ & DP-SGDA $\uparrow$ & PrivateDiff $\uparrow$ \\
\midrule
Non-private           & 0.9669             & 0.9664                 & 0.9319             & 0.9318                 & 0.9086             & 0.9095                 \\
$\epsilon = 0.5$      & 0.5586             & 0.9383                 & 0.5319             & 0.8777                 & 0.5590             & 0.8499                 \\
$\epsilon = 1$        & 0.5557             & 0.9521                 & 0.5327             & 0.9022                 & 0.5571             & 0.8780                 \\
$\epsilon = 5$        & 0.5569             & 0.9631                 & 0.5450             & 0.9252                 & 0.5797             & 0.9045                 \\
$\epsilon = 10$       & 0.5587             & 0.9647                 & 0.5505             & 0.9285                 & 0.6260             & 0.9076                 \\
\bottomrule
\end{tabular}
\label{tab:cifar}
\end{table}

\begin{figure*}[h]

\centering

\begin{subfigure}[t]{0.246\textwidth}
    \includegraphics[width=\textwidth]  
    {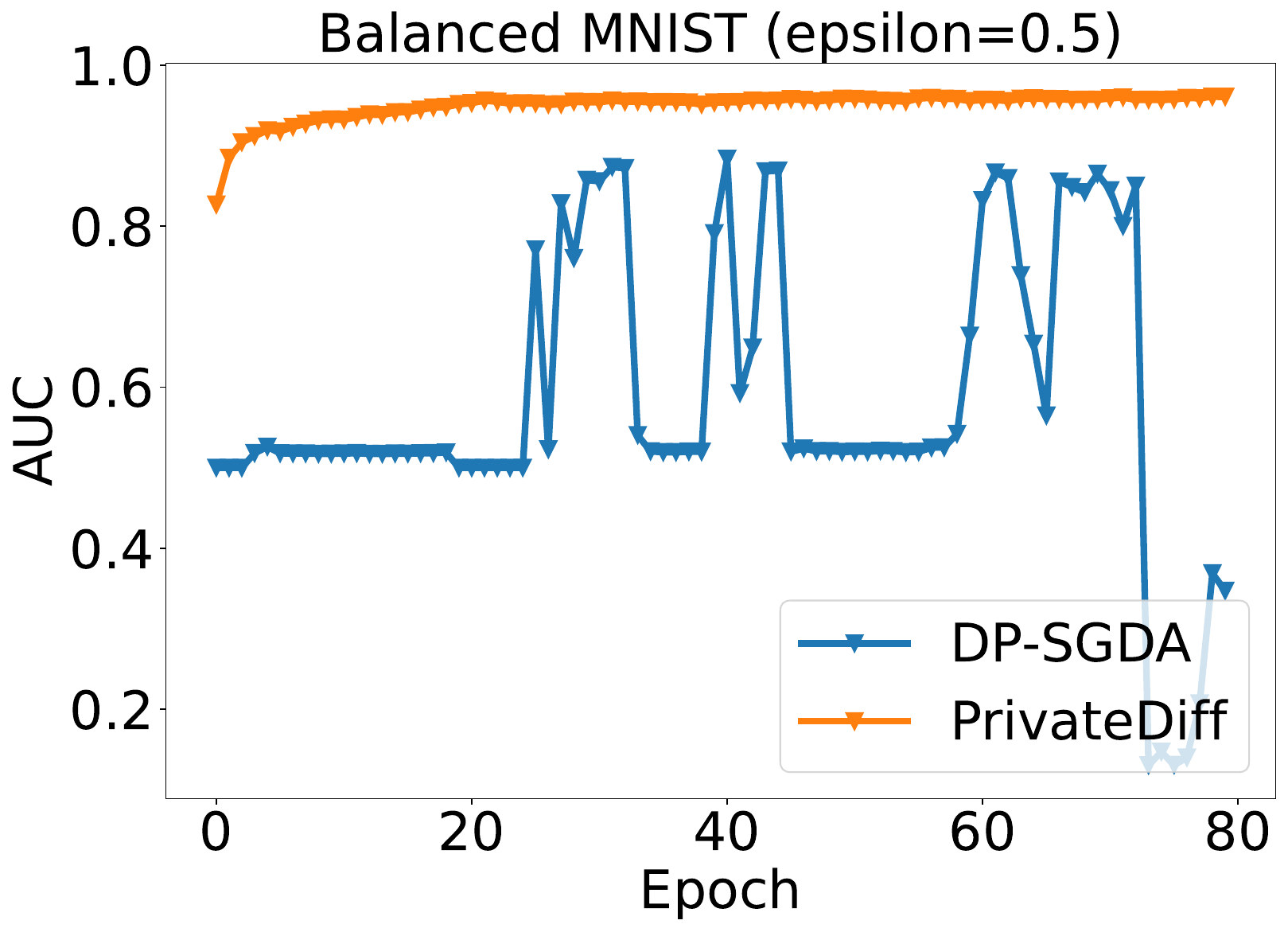}
    \includegraphics[width=\textwidth]
    {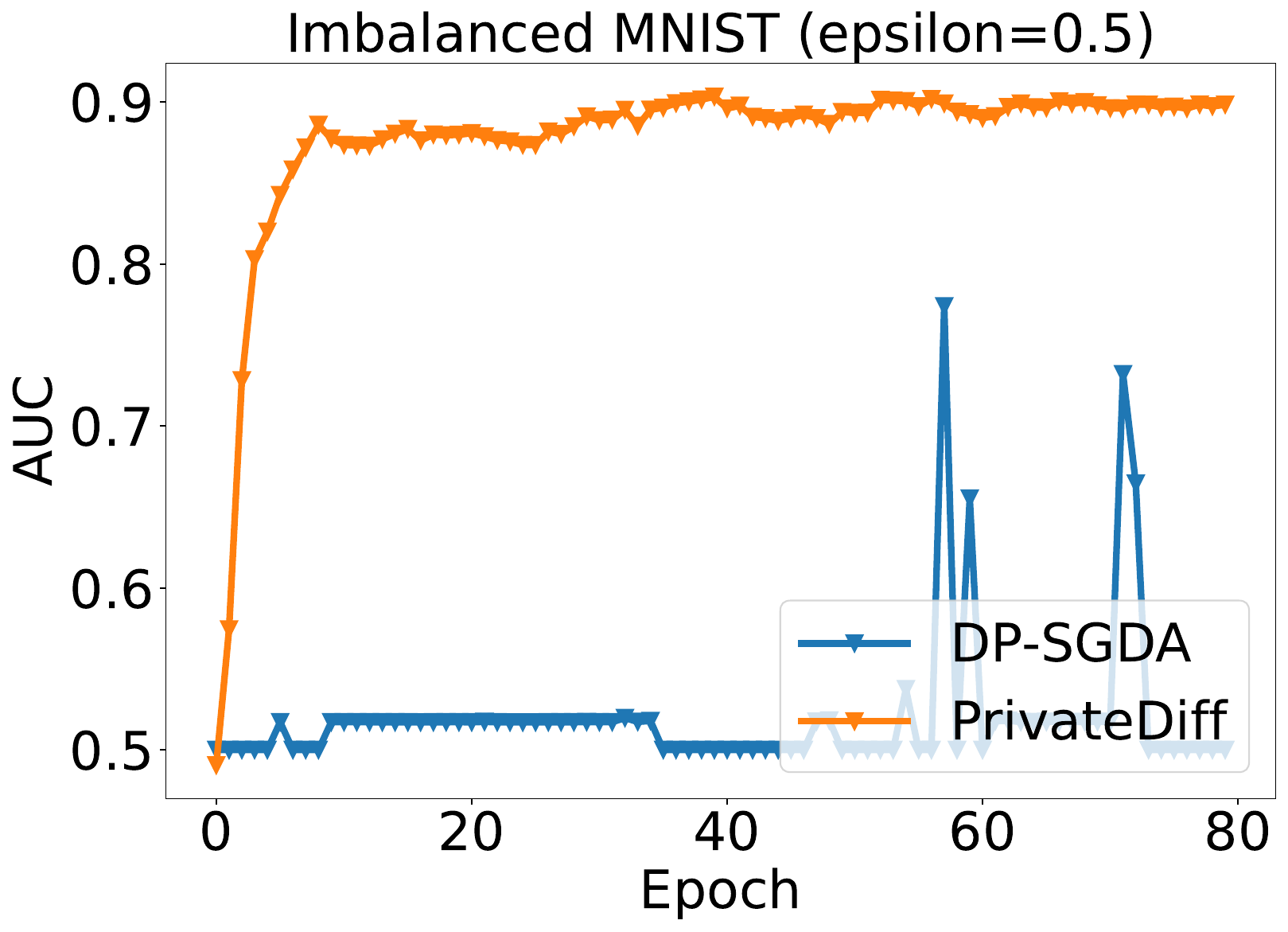}
    \caption{$\epsilon=0.5$}
\end{subfigure}
\begin{subfigure}[t]{0.246\textwidth}
    \includegraphics[width=\textwidth]  
    {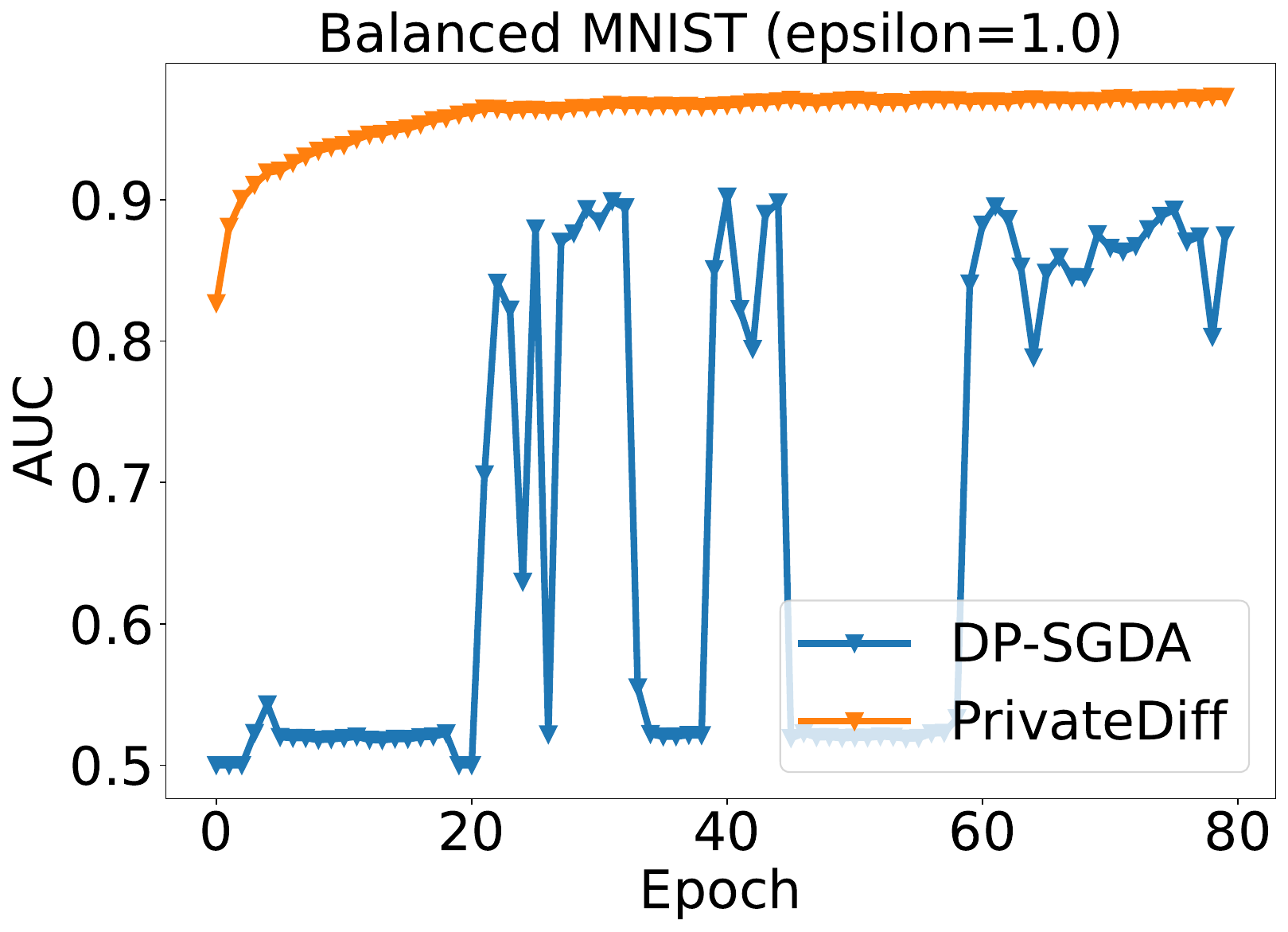}
    \includegraphics[width=\textwidth]
    {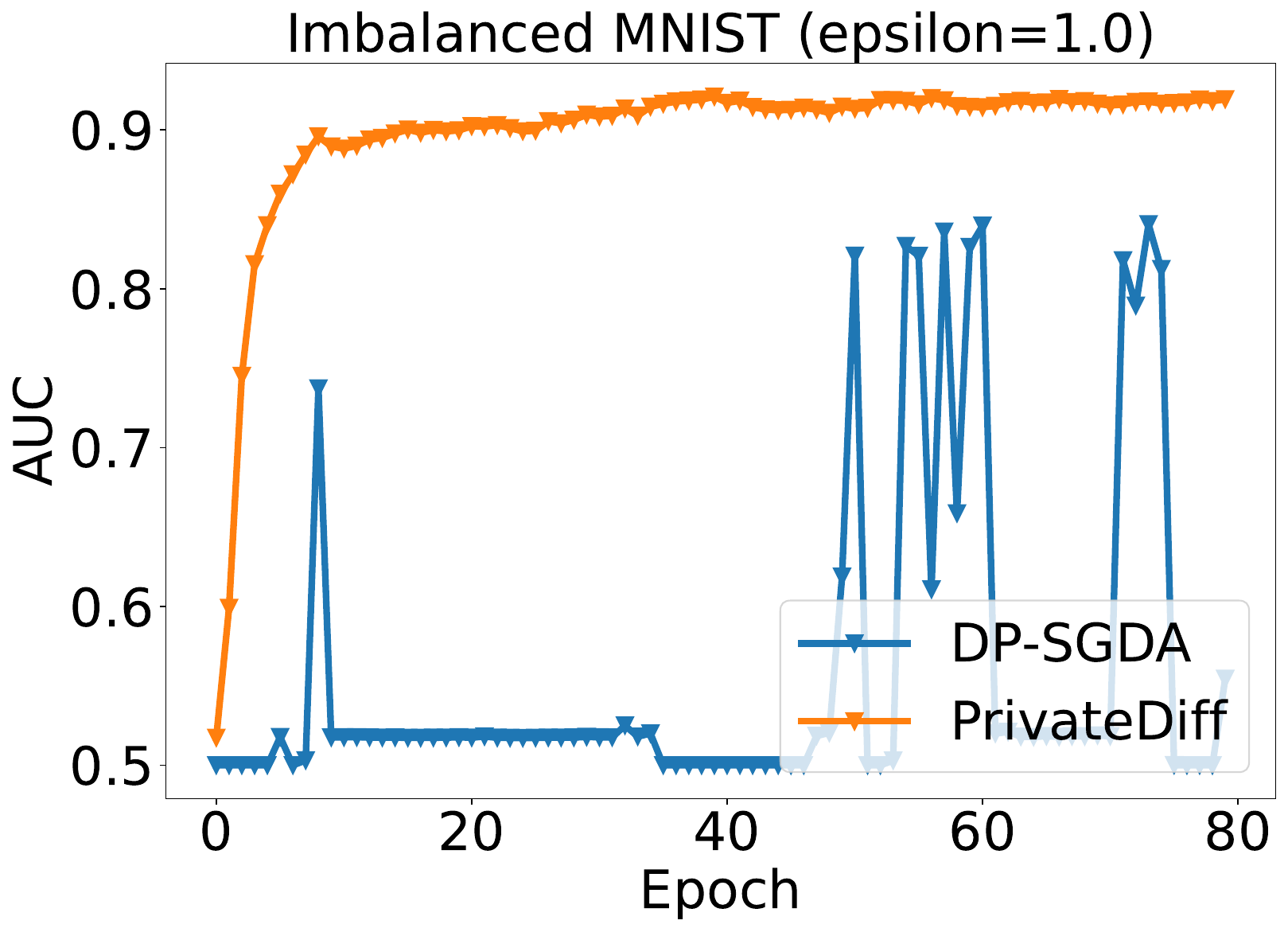}
    \caption{$\epsilon=1$}
\end{subfigure}
\begin{subfigure}[t]{0.246\textwidth}
    \includegraphics[width=\textwidth]  
    {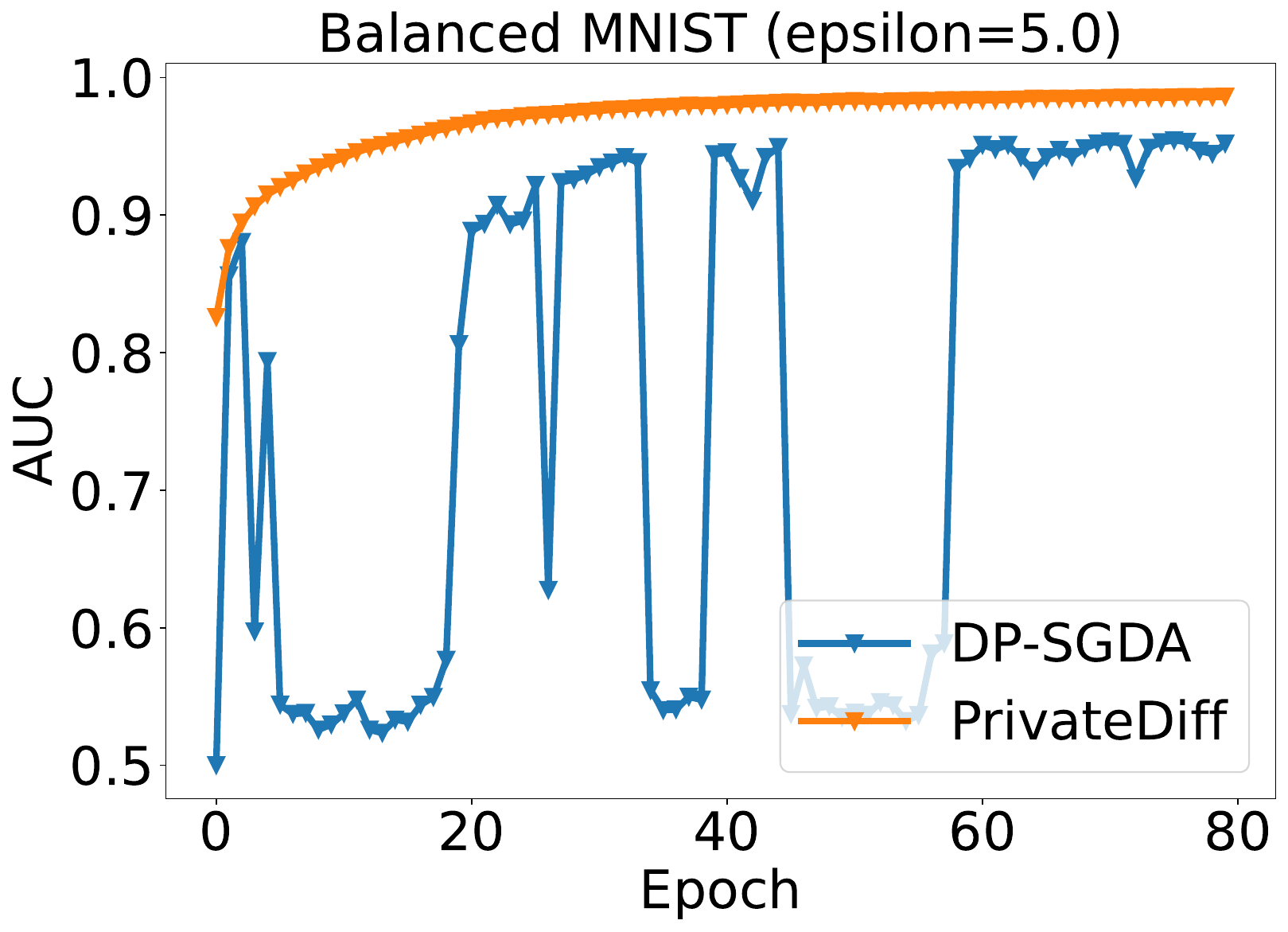}
    \includegraphics[width=\textwidth]
    {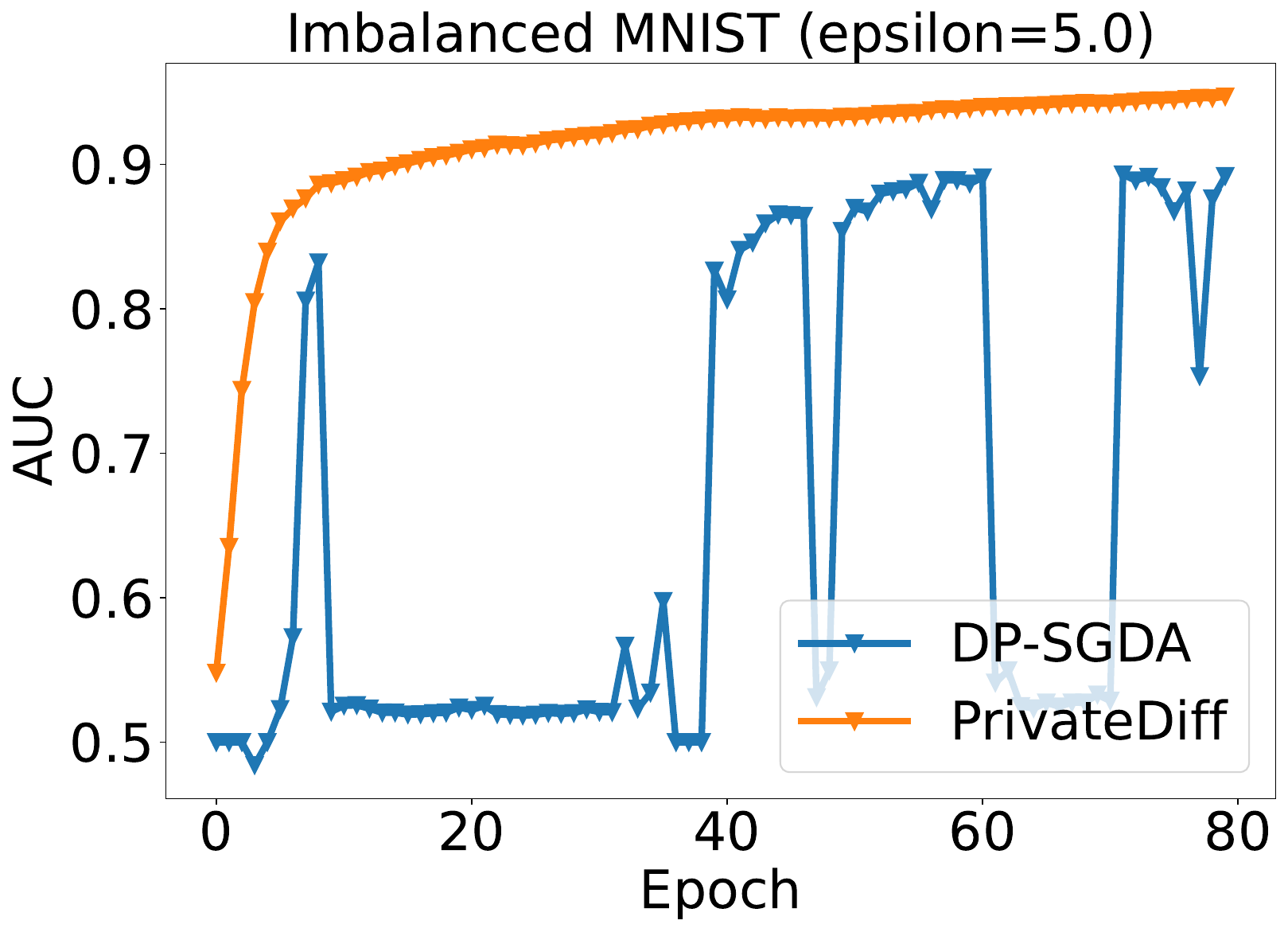}
    \caption{$\epsilon=5$}
\end{subfigure}
\begin{subfigure}[t]{0.246\textwidth}
    \includegraphics[width=\textwidth]  
    {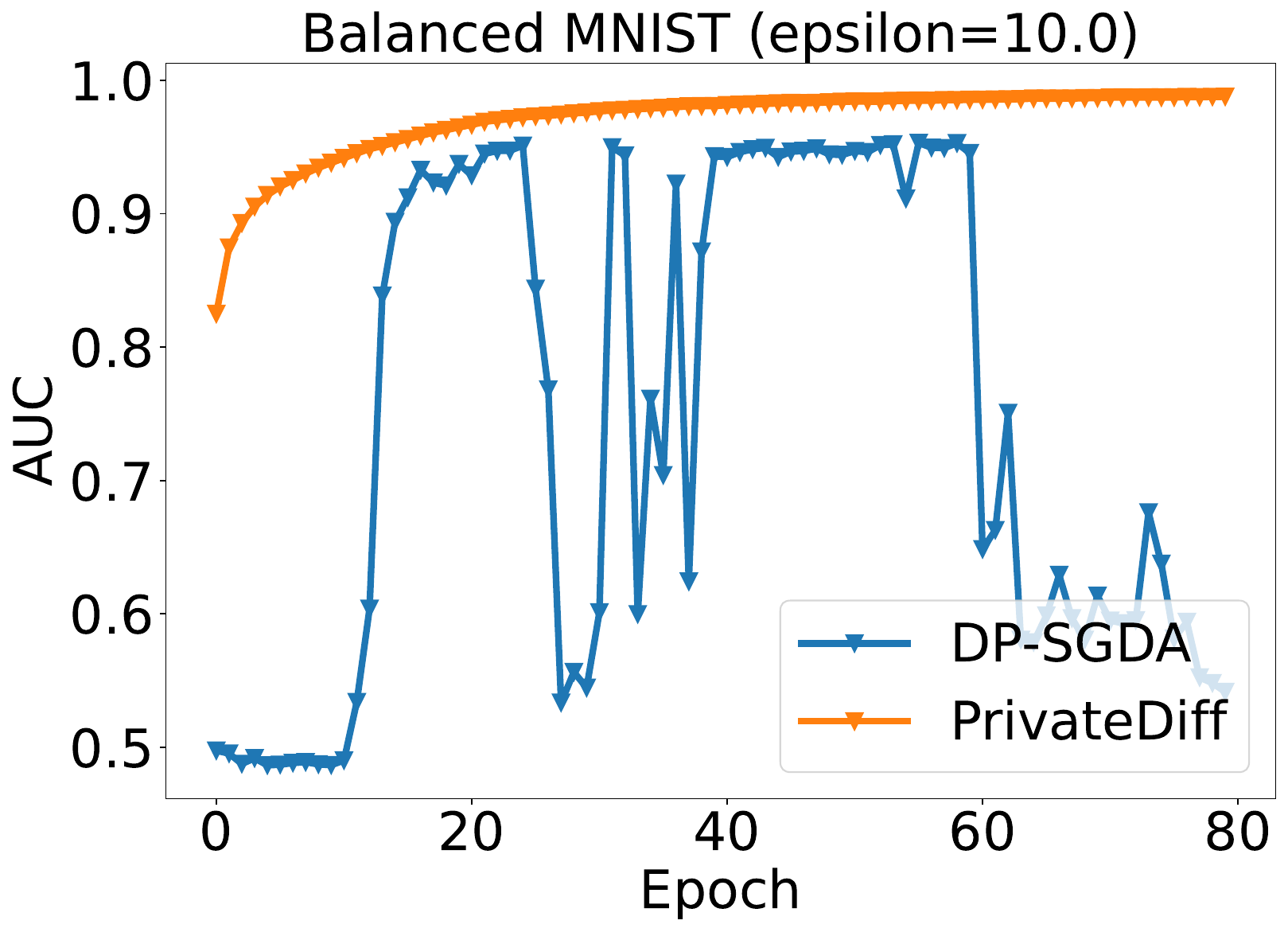}
    \includegraphics[width=\textwidth]
    {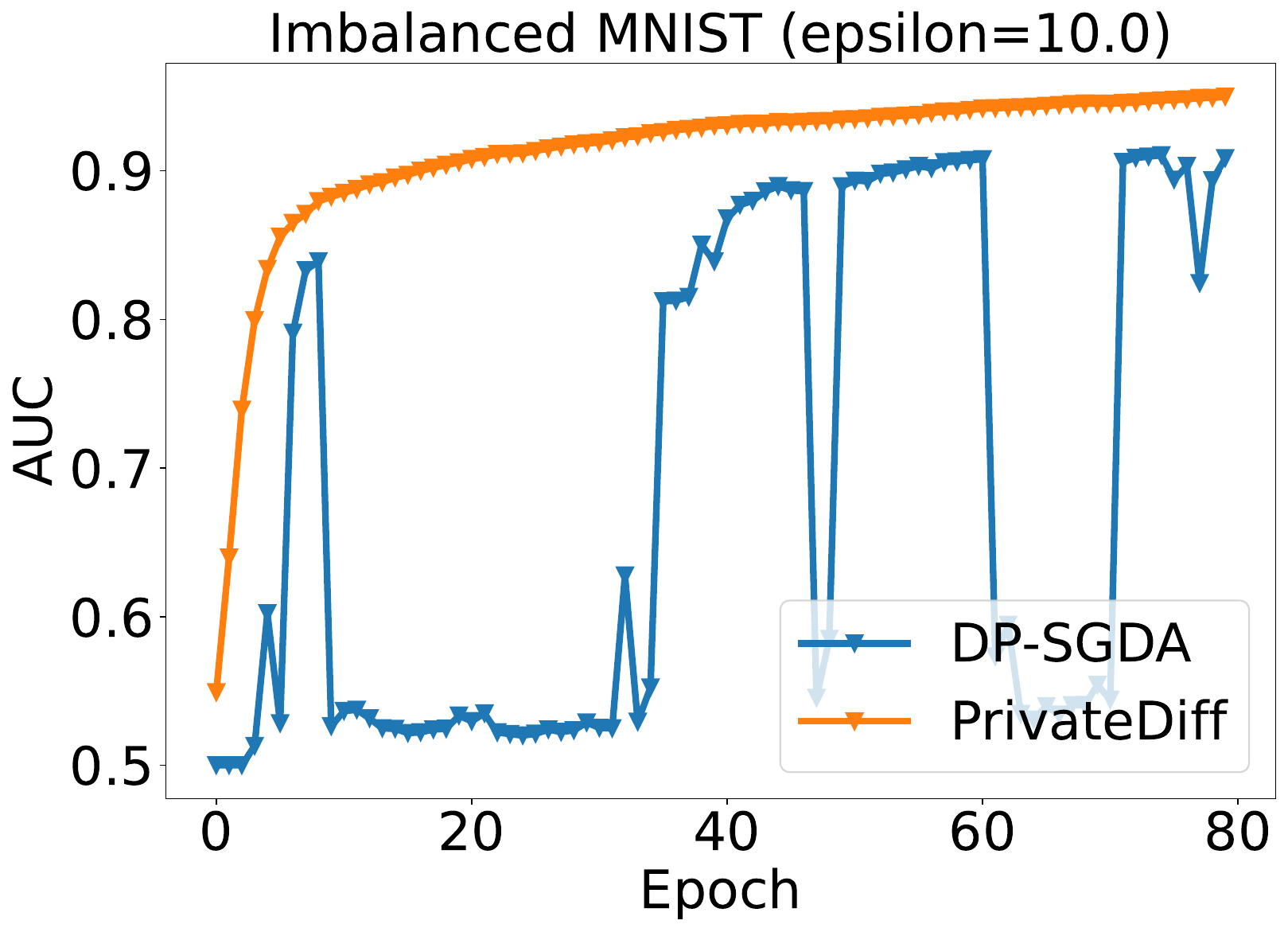}
    \caption{$\epsilon=10.0$}
\end{subfigure}
\caption{Comparison of AUC performance in DP-SGDA and PrivateDiff Minimax on MNIST dataset.}
\label{fig:curve_mnist}
\end{figure*}

\begin{figure*}[h]

\centering

\begin{subfigure}[t]{0.246\textwidth}
    \includegraphics[width=\textwidth]  
    {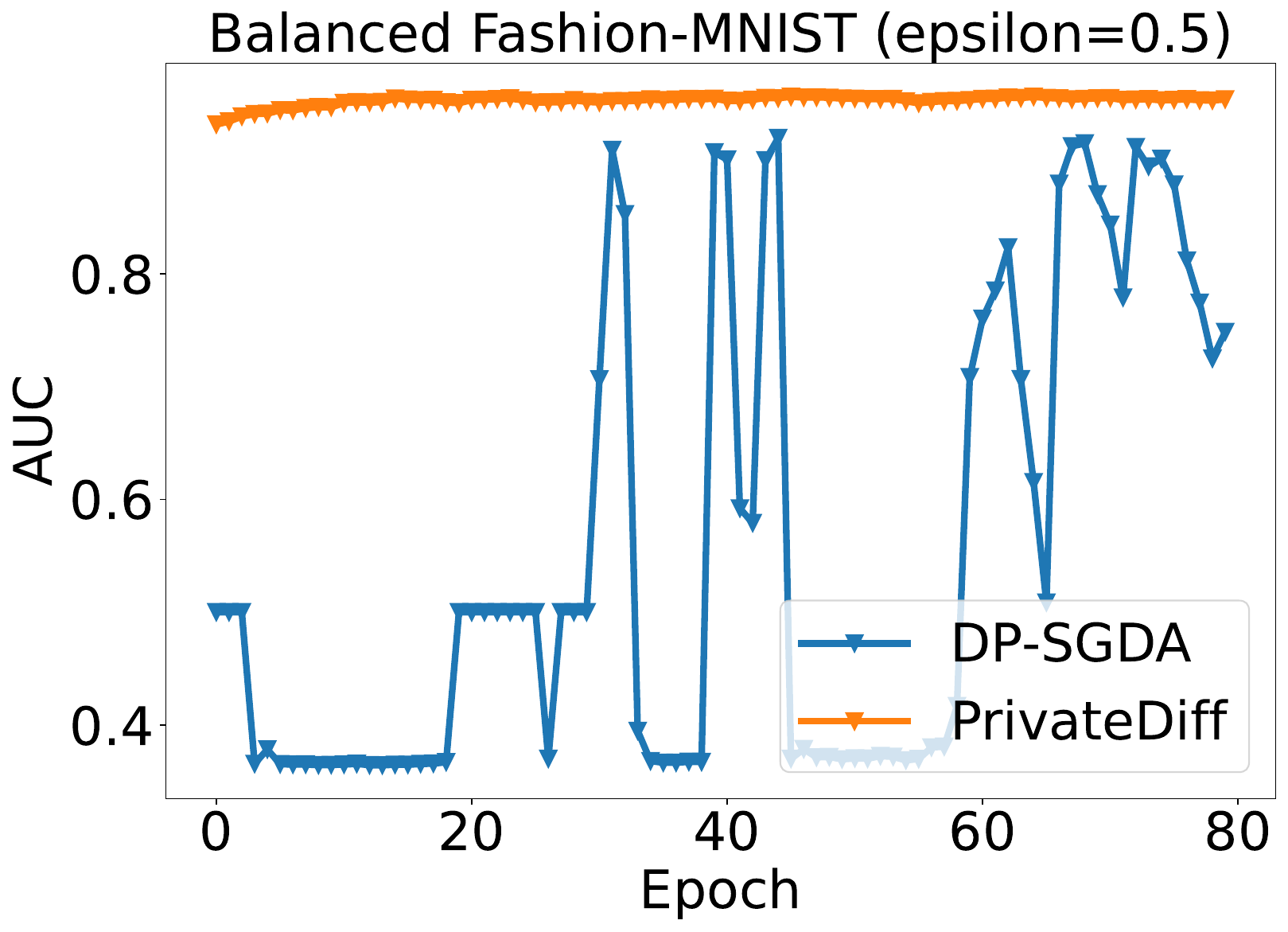}
    \includegraphics[width=\textwidth]
    {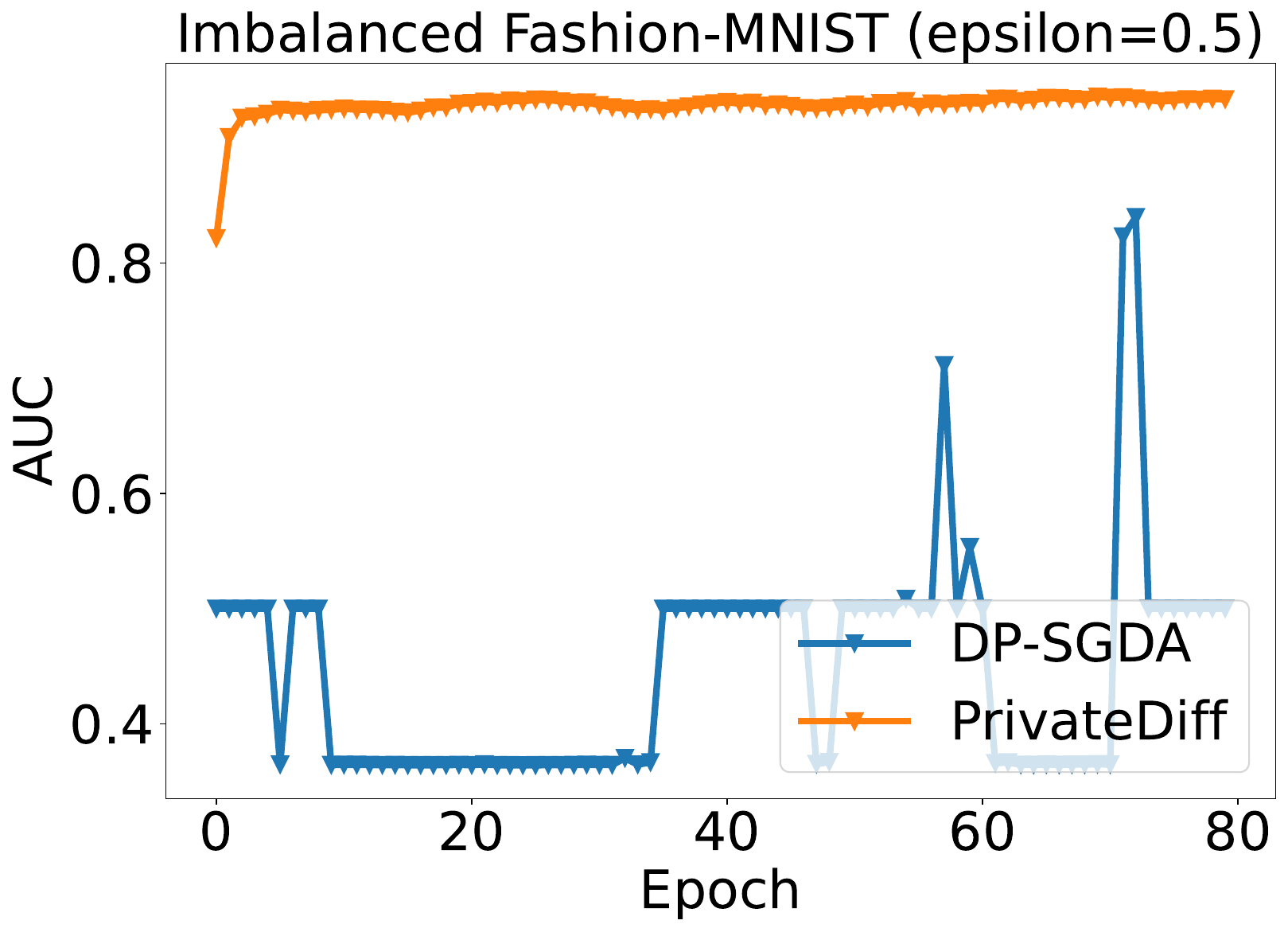}
    \caption{$\epsilon=0.5$}
\end{subfigure}
\begin{subfigure}[t]{0.246\textwidth}
    \includegraphics[width=\textwidth]  
    {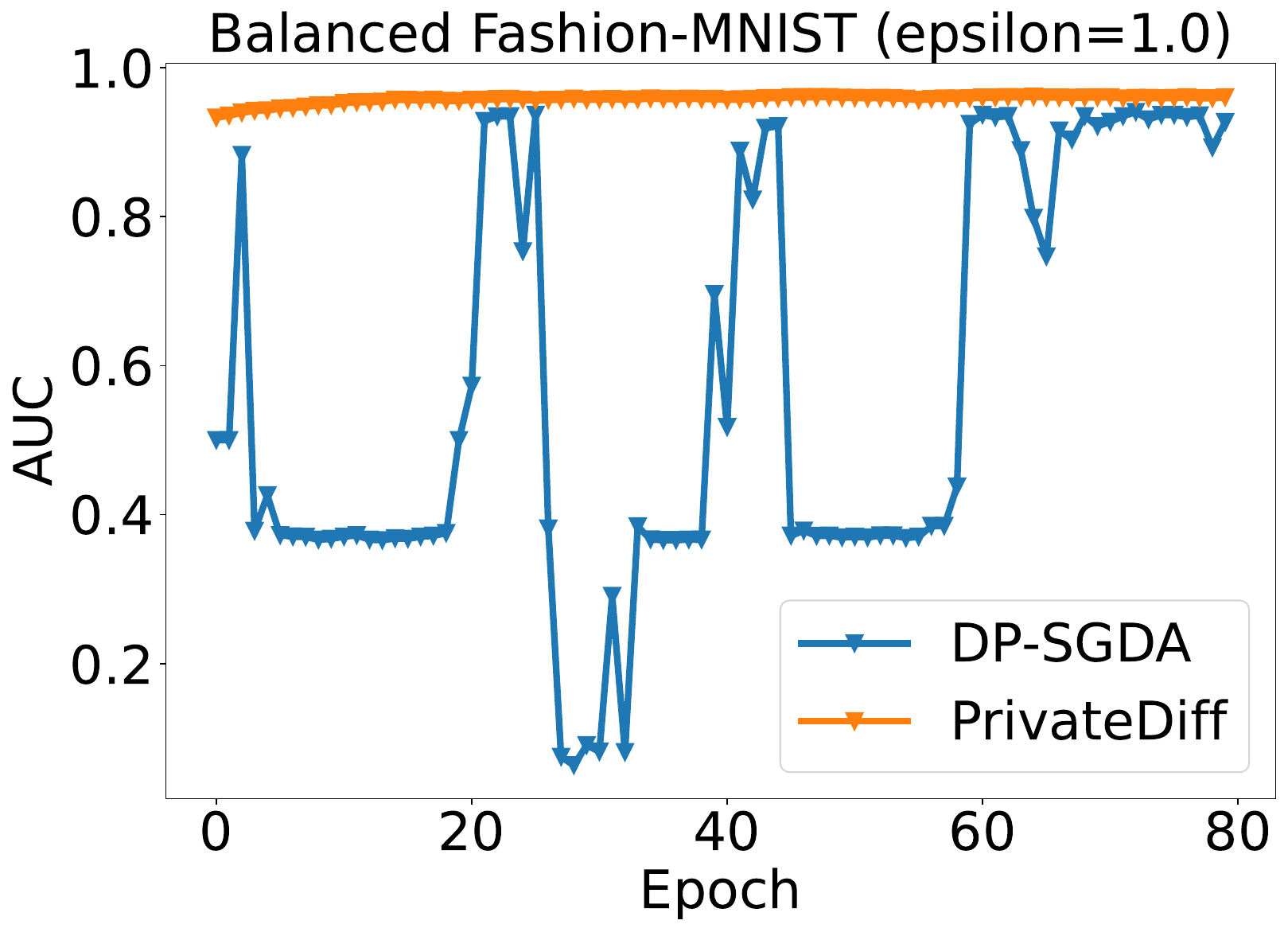}
    \includegraphics[width=\textwidth]
    {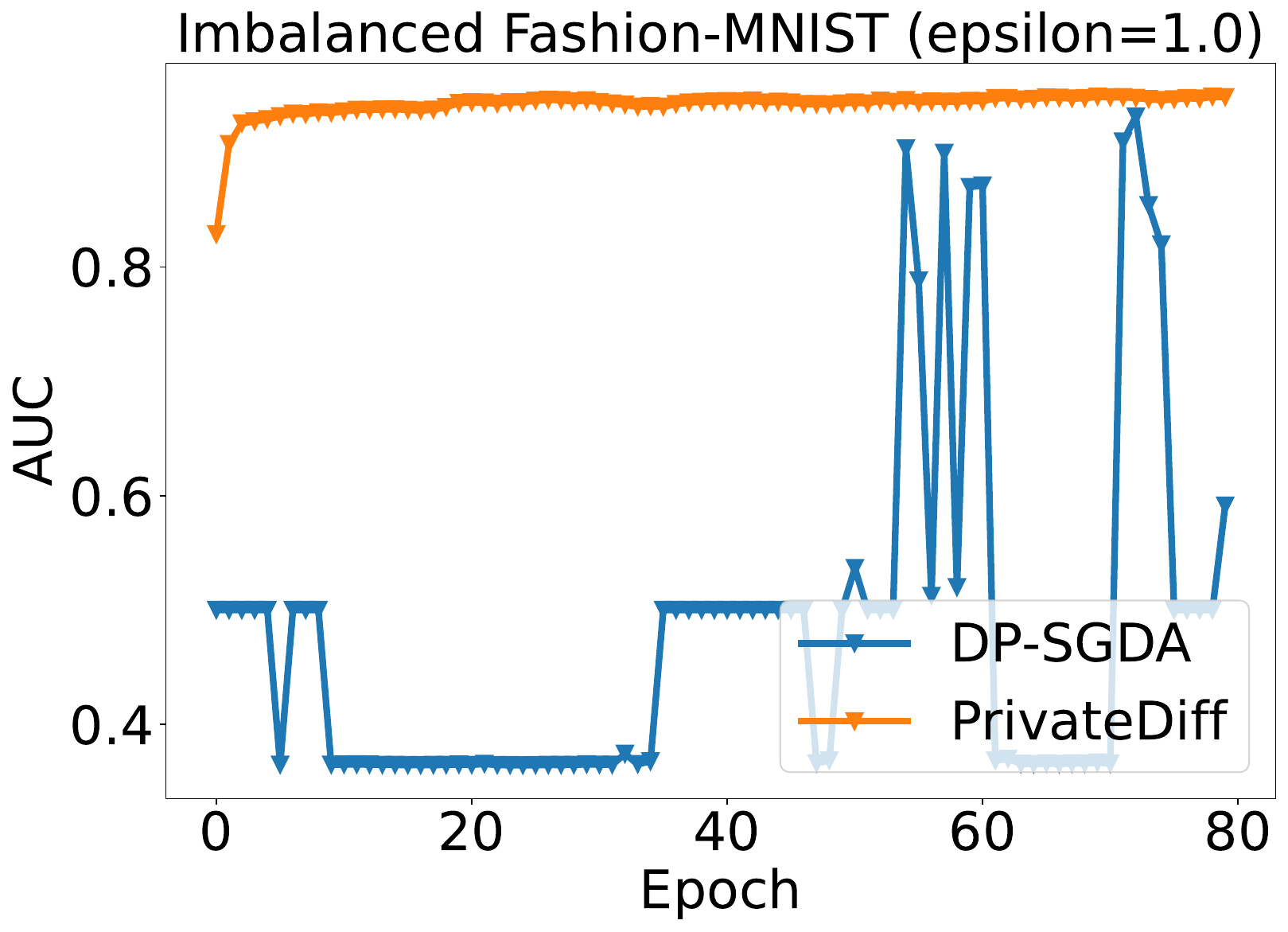}
    \caption{$\epsilon=1$}
\end{subfigure}
\begin{subfigure}[t]{0.246\textwidth}
    \includegraphics[width=\textwidth]  
    {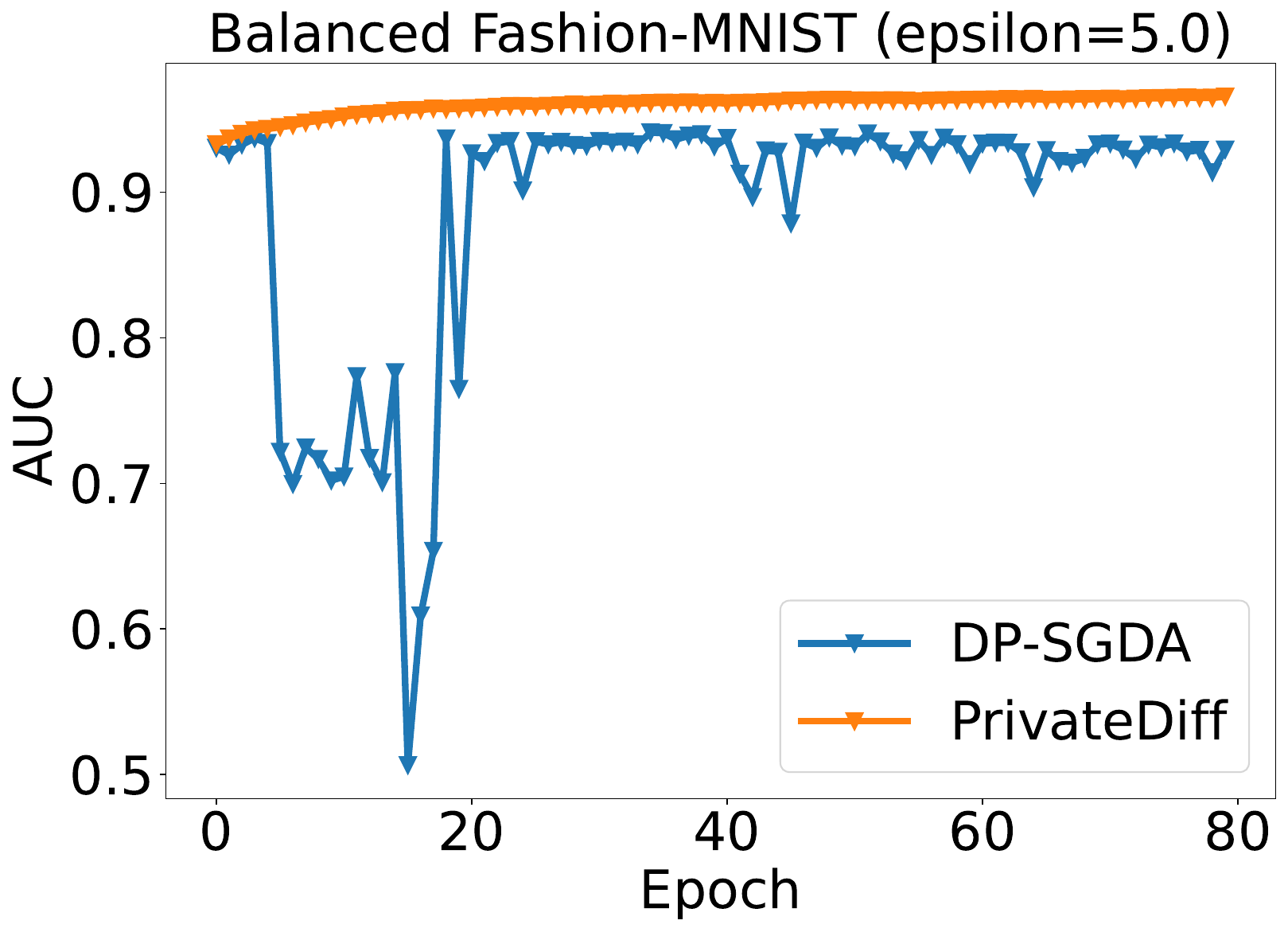}
    \includegraphics[width=\textwidth]
    {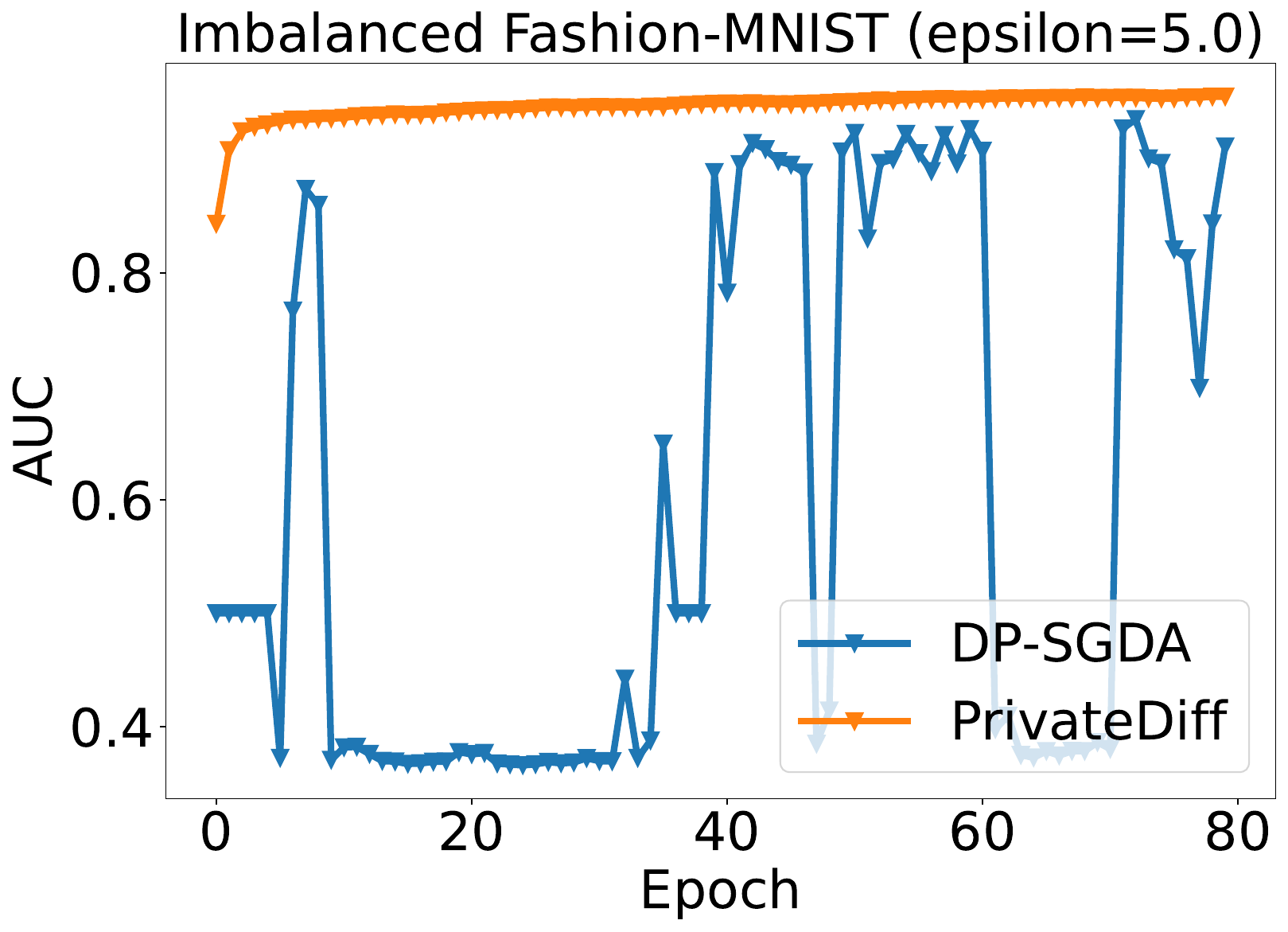}
    \caption{$\epsilon=5$}
\end{subfigure}
\begin{subfigure}[t]{0.246\textwidth}
    \includegraphics[width=\textwidth]  
    {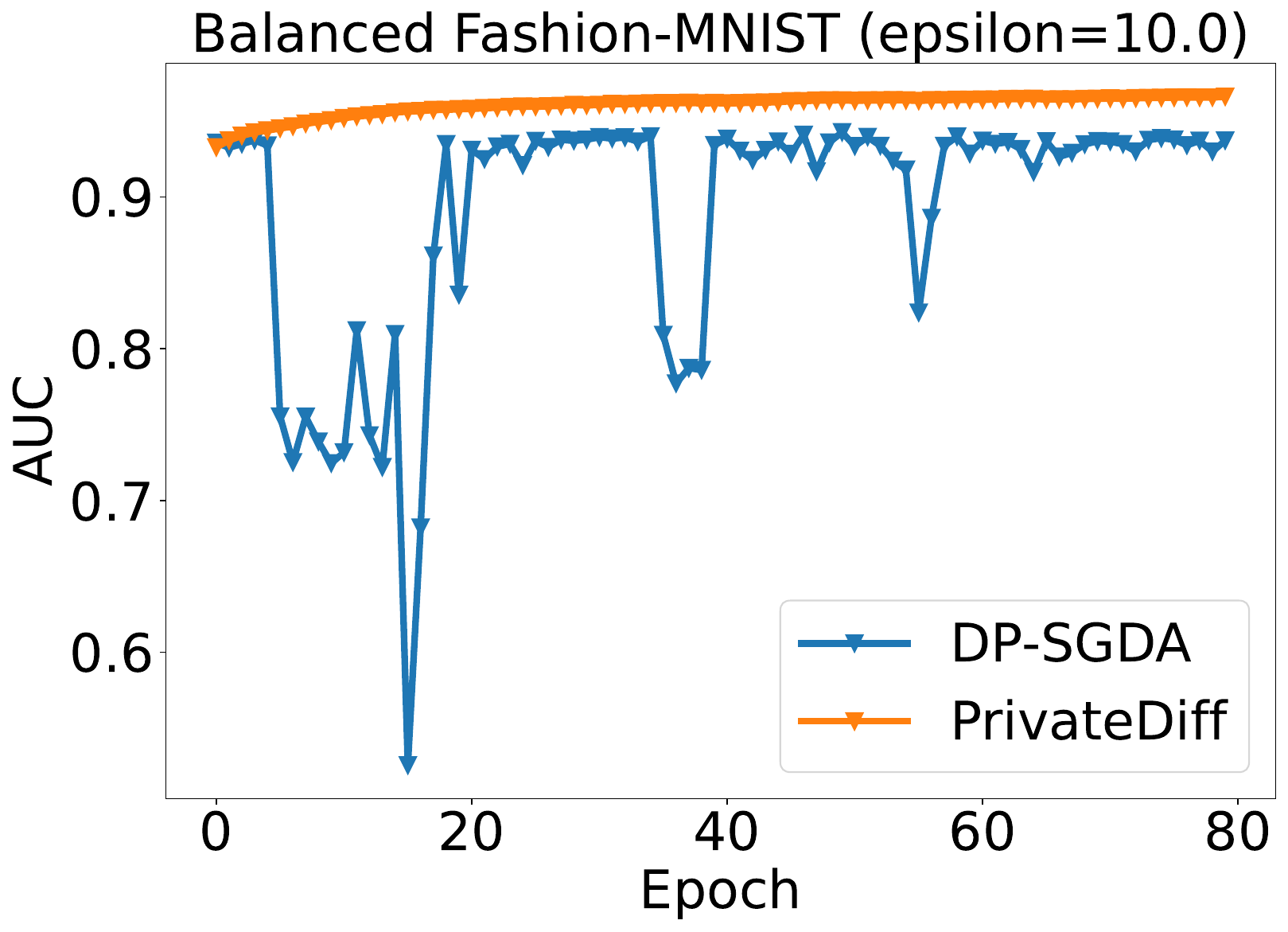}
    \includegraphics[width=\textwidth]
    {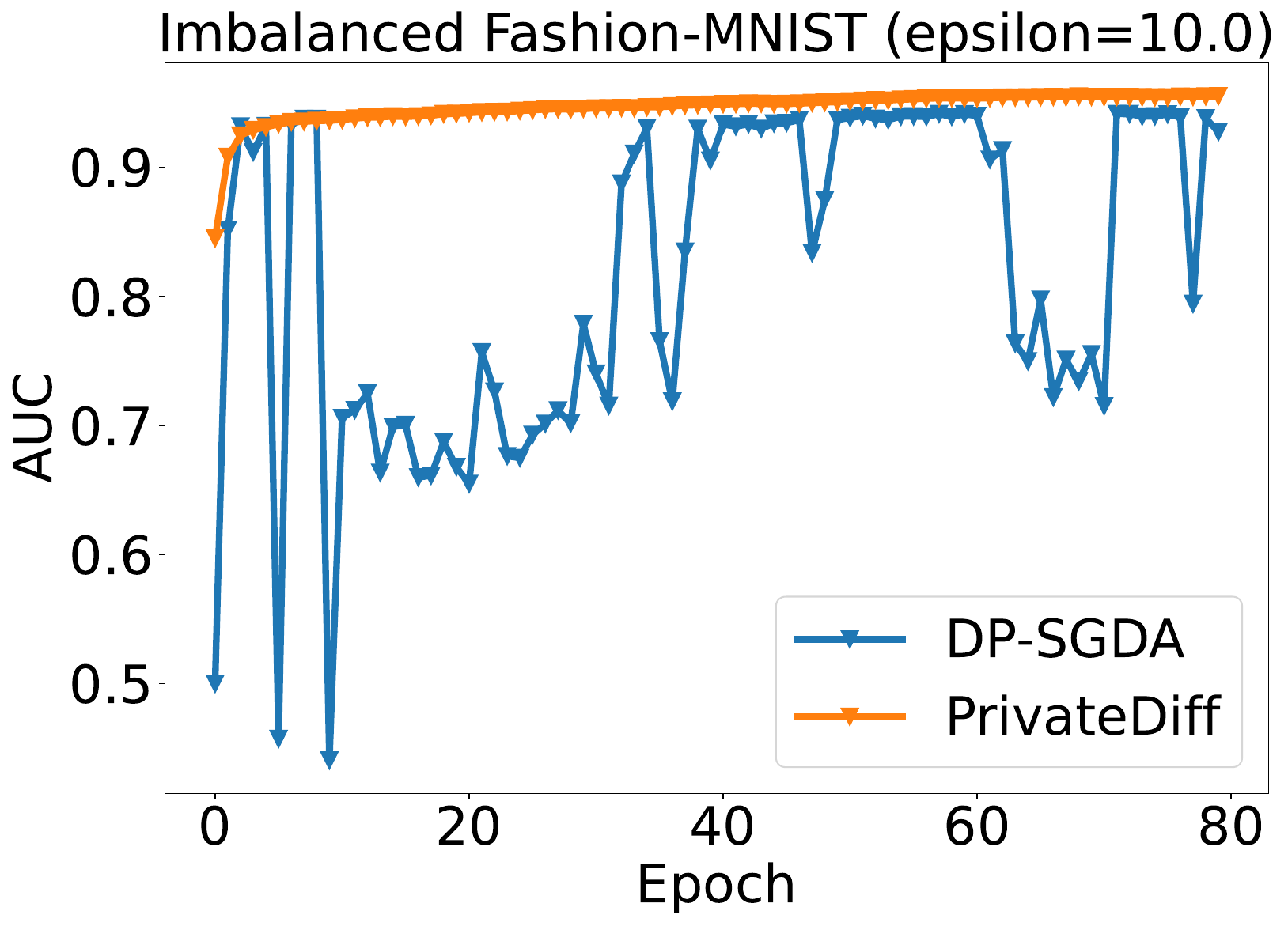}
    \caption{$\epsilon=10.0$}
\end{subfigure}

\caption{Comparison of AUC performance in DP-SGDA and PrivateDiff Minimax on Fashion-MNIST dataset.}
\label{fig:curve_fmnist}
\end{figure*}

\begin{figure*}[h]
    \centering
    
    \begin{subfigure}{0.246\textwidth}
        \centering
        \includegraphics[width=\linewidth]{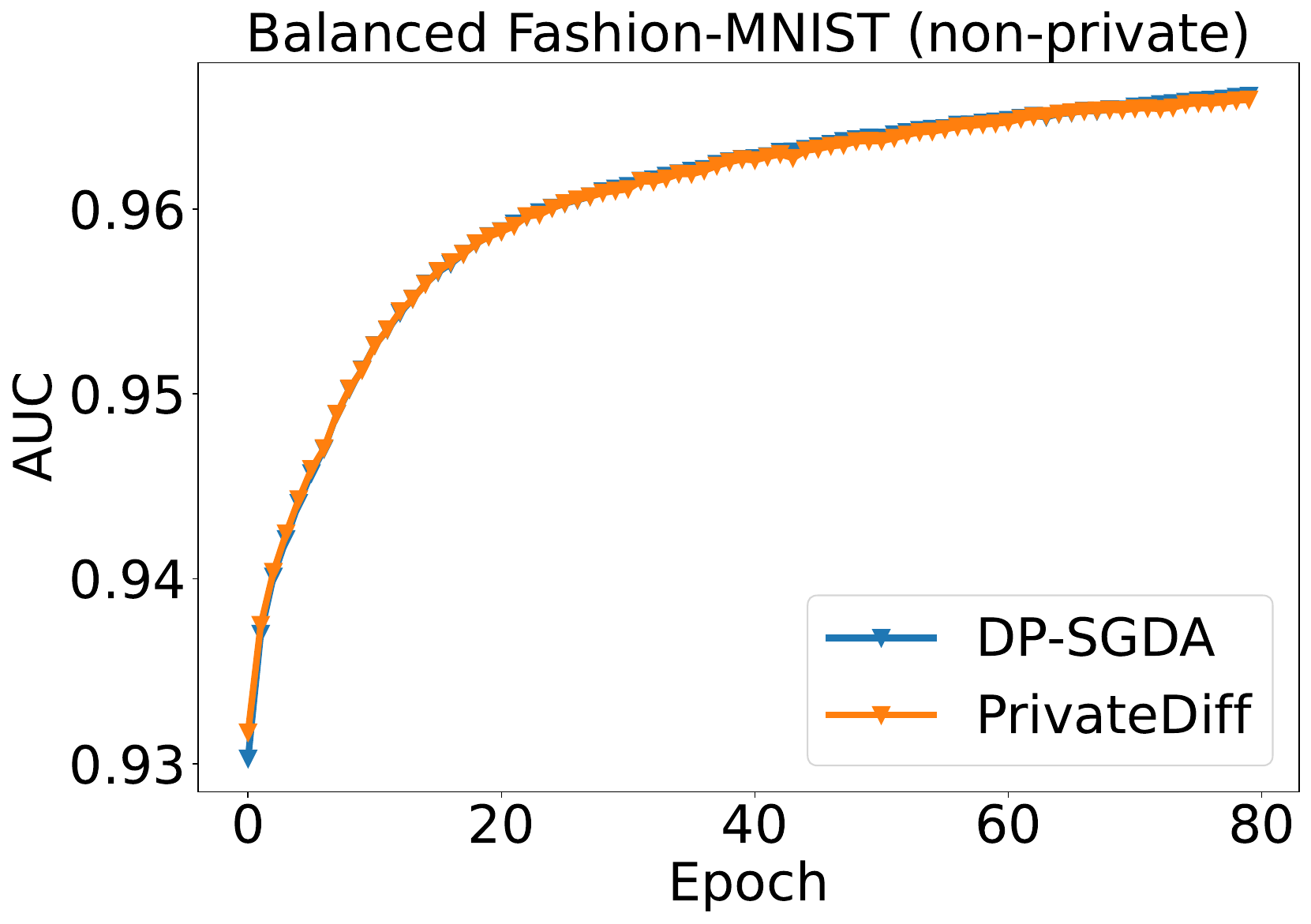}
        \caption{Fashion-MNIST}
    \end{subfigure}
    \begin{subfigure}{0.246\textwidth}
        \centering
        \includegraphics[width=\linewidth]{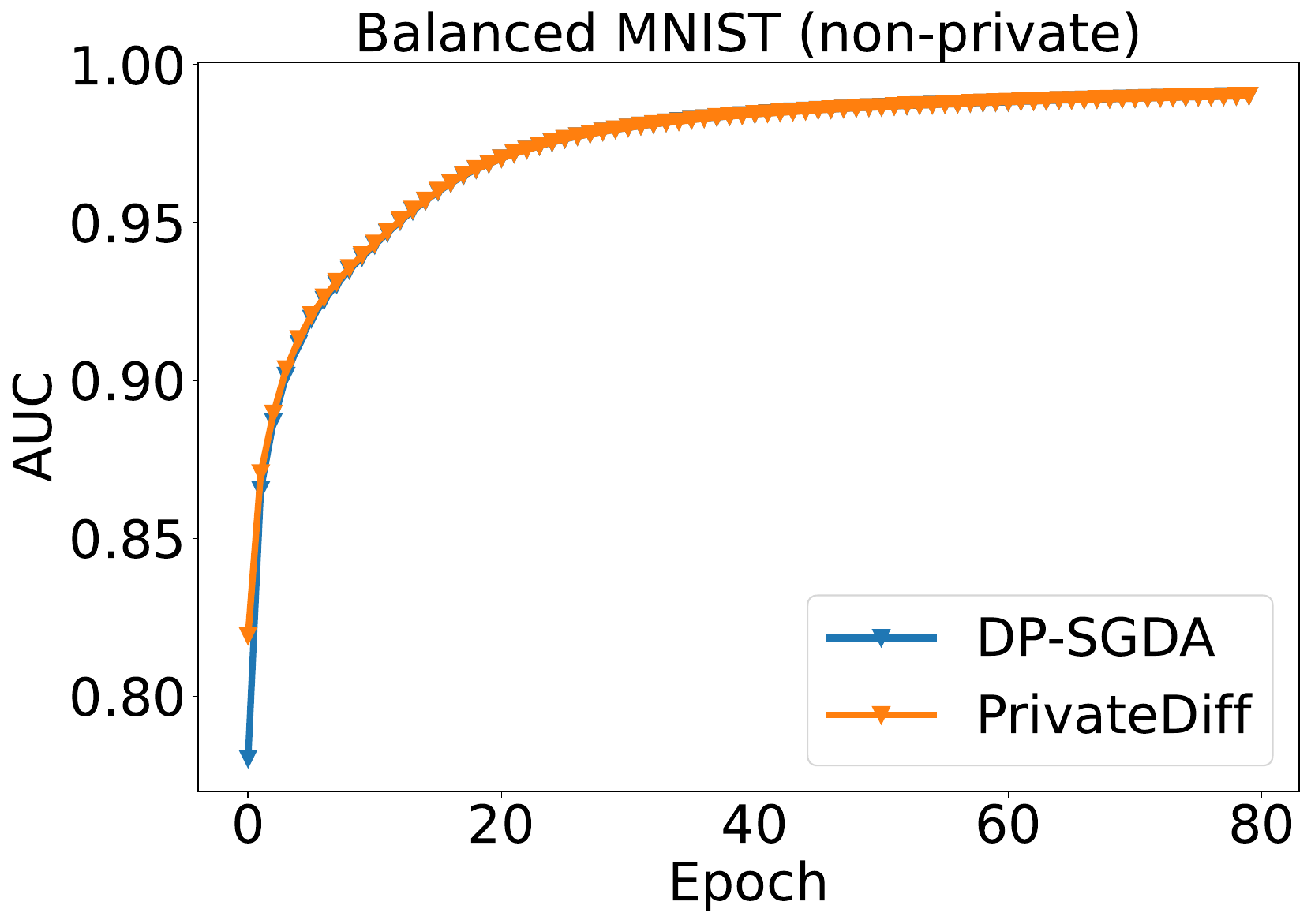}
        \caption{MNIST}
    \end{subfigure}%
    \begin{subfigure}{0.246\textwidth}
        \centering
        \includegraphics[width=\linewidth]{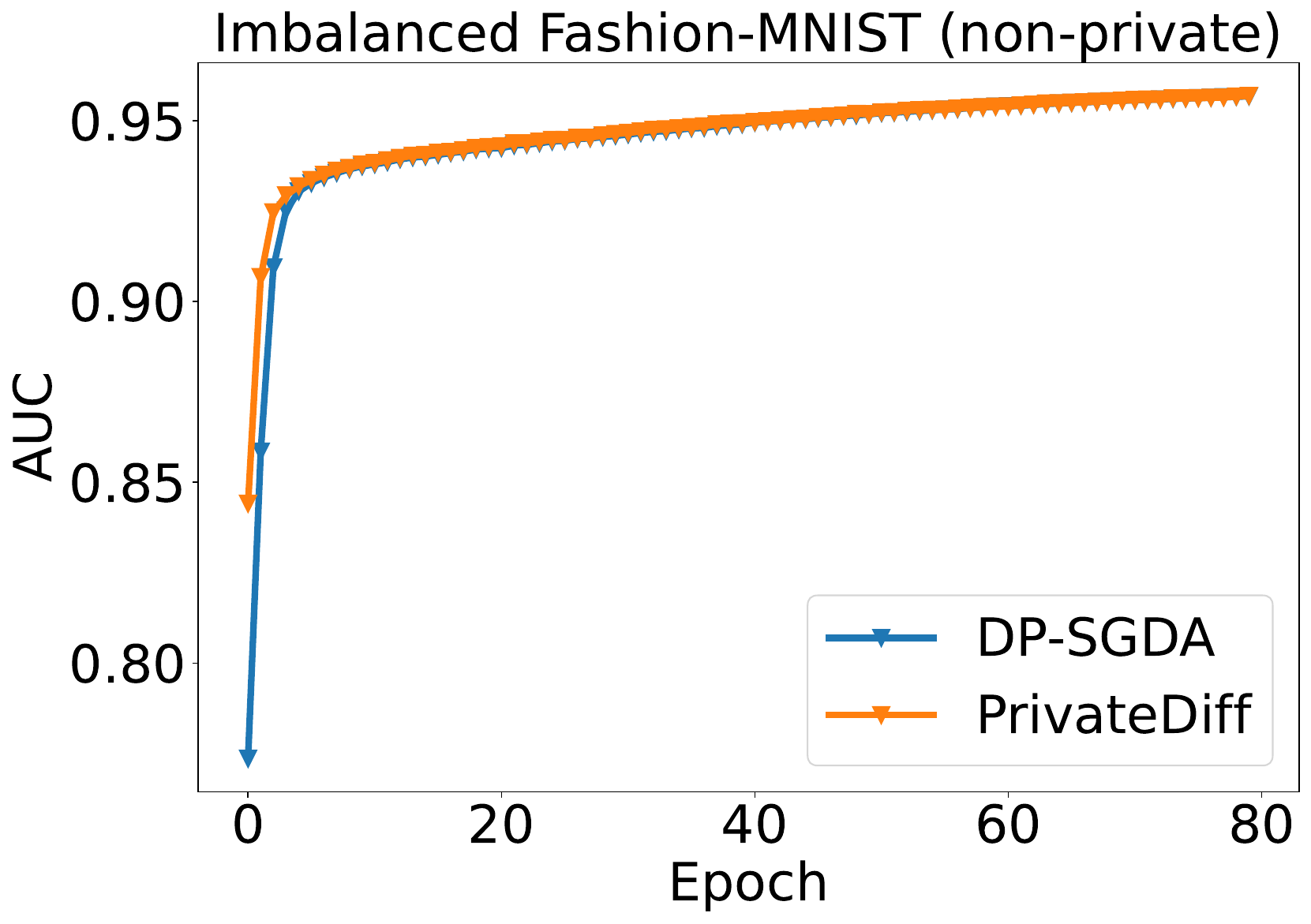}
        \caption{Imbalanced Fashion-MNIST}
    \end{subfigure}
    \begin{subfigure}{0.246\textwidth}
        \centering
        \includegraphics[width=\linewidth]{AAAI/figures/auc/Imbalanced_Fashion-MNIST_non-private.pdf}
        \caption{Imbalanced MNIST}
    \end{subfigure}

\caption{Non-private Performance across Different Dataset. }
\label{fig:curve_nonprivate}
\end{figure*}

\begin{figure*}[H]
    \centering
    
    \begin{subfigure}{0.246\textwidth}
        \centering
        \includegraphics[width=\linewidth]{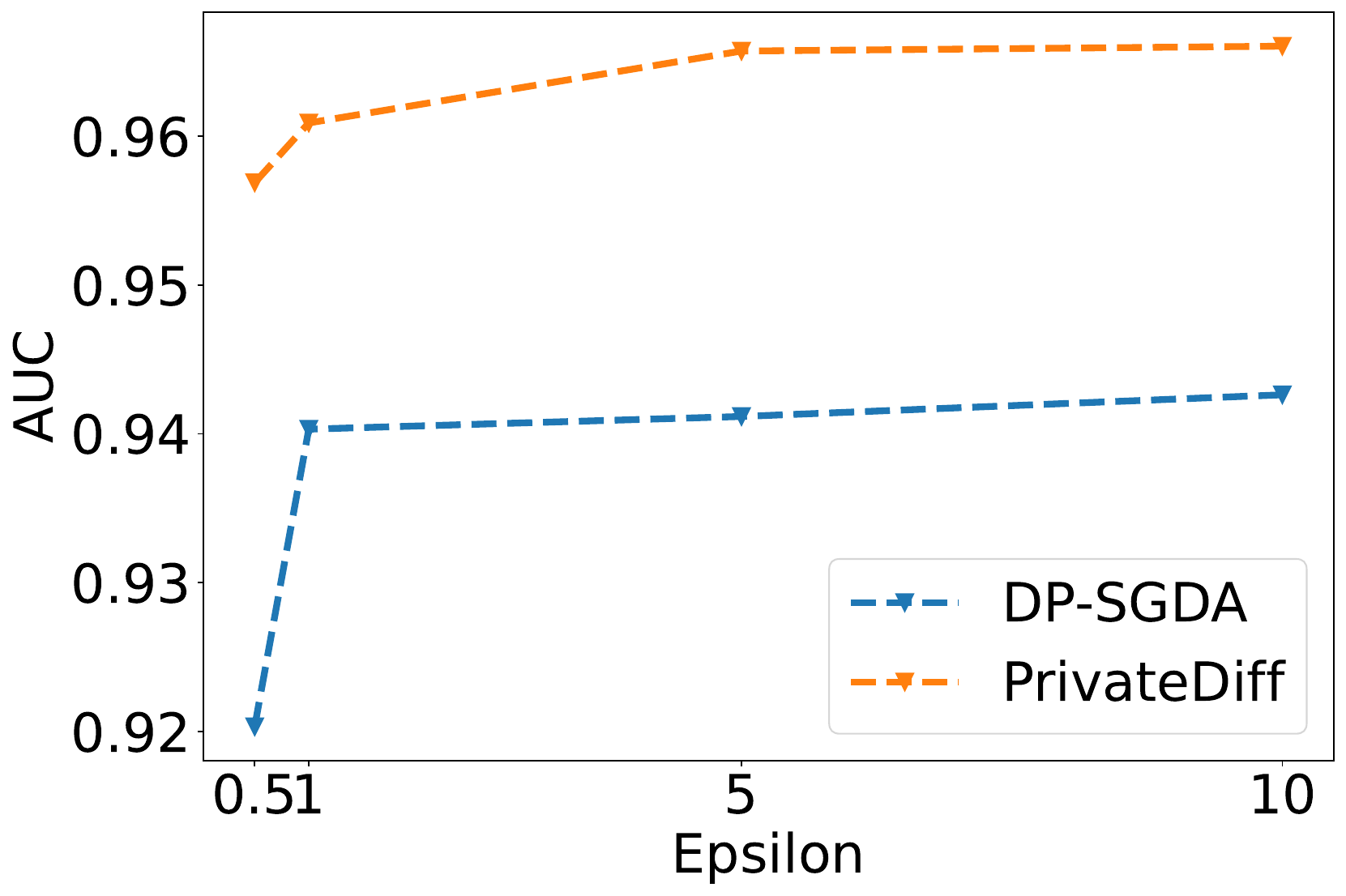}
        \caption{Fashion-MNIST}
    \end{subfigure}
    \begin{subfigure}{0.246\textwidth}
        \centering
        \includegraphics[width=\linewidth]{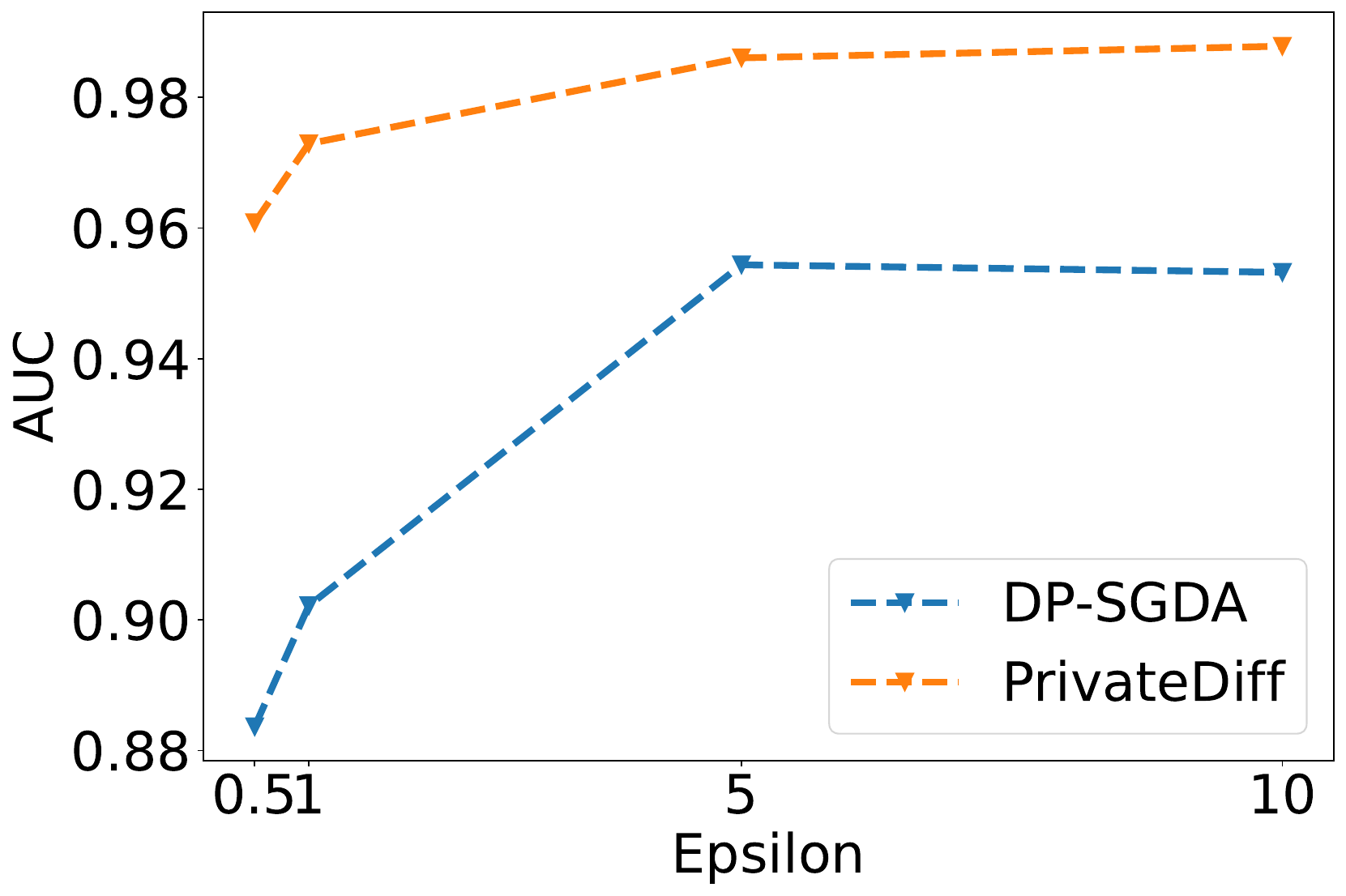}
        \caption{MNIST}
    \end{subfigure}%
    \begin{subfigure}{0.246\textwidth}
        \centering
        \includegraphics[width=\linewidth]{AAAI/figures/auc/Impact_of_Epsilon_Imbalanced_Fashion-MNIST.pdf}
        \caption{Imbalanced Fashion-MNIST}
    \end{subfigure}
    \begin{subfigure}{0.246\textwidth}
        \centering
        \includegraphics[width=\linewidth]{AAAI/figures/auc/Impact_of_Epsilon_Imbalanced_MNIST.pdf}
        \caption{Imbalanced MNIST}
    \end{subfigure}
    \caption{Impact of Privacy Budget for PrivateDiff Algorithm across Different Dataset}
\end{figure*}
\subsection{Generative Adversarial Network}
\subsubsection{Background}
Generative Adversarial Network (GAN) \citep{goodfellow2014generative} is a powerful framework for generating realistic synthetic data. GANs consist of two neural networks, the generator $G$ and the discriminator $D$, that are trained simultaneously in a competitive setting. Wasserstein GAN (WGAN) \citep{arjovsky2017wasserstein} is a widely used variant due to its advantage of learning stability over traditional GAN. The optimization of WGAN is formulated as a minimax problem of the Wasserstein distance estimation between real samples and fake samples,

\begin{equation}
    \min_{w_{G}}\max_{w_{D}}\mathbb{E}_{x}[D_{w_{D}}(x)]-\mathbb{E}_{z\sim\mathcal{N}(0,1)}[D_{w_{D}}(G_{w_{G}}(z))]-\lambda||w_{D}||^{2},
    \label{eq:gan}
\end{equation}
where $x$ represents the real sample, $z$ is the Gaussian noise generated by $\mathcal{N}(0,1)$. $\lambda$ is the penalty coefficient. $w_G$ and $w_D$ correspond to generator and discriminator parameters, respectively. Our experiment optimizes Equation~\ref{eq:gan} using PrivateDiff.

\subsubsection{Implementation Details}
We train a WGAN to generate digits using the MNIST dataset. The training settings are presented in Table~\ref{tab:ganhparam}. Both the generator and discriminator are configured as multilayer perceptrons. The generator consists of 4 hidden layers of 128, 256, 512, and 1024 neurons sequentially. The discriminator consists of 2 hidden layers of 512 and 256 neurons sequentially.

\begin{table*}[h]
\centering

\begin{tabular}{lccccccccc} 
\toprule
                              & $C_1$ & $C_2$ & $T$ & $T_2$ & Batch Size & Epochs/Iterations &   \\ 
\midrule
DP-SGDA                         & 0.3     & 0.3     & N/A & N/A   & 256       & 50/11750     \\
PrivateDiff    & 0.3     & 0.3     & 2   & 1     & 256       & 50/11750     \\
\bottomrule
\end{tabular}
\caption{Hyperparameter Settings and Training Configurations.}
\label{tab:ganhparam}
\end{table*}

\subsubsection{Learning curve analysis}
The learning curve of PrivateDiff and DP-SGDA is presented in Figure~\ref{fig:gan}. It is optimal that the Wasserstein estimate is close to zero. Across all three $\epsilon$ values, the PrivateDiff method shows a more stable and smoother trend in the Wasserstein estimate compared to DP-SGDA. In contrast, DP-SGDA exhibits significant fluctuations in the Wasserstein estimate throughout the iterations, especially at lower $\epsilon$ values. This suggests that PrivateDiff is more robust and less prone to oscillations during training, which is critical for achieving consistent performance.

\begin{figure*}[h]
    \centering
    
    \begin{subfigure}{0.33\textwidth}
        \centering
        \includegraphics[width=\linewidth]{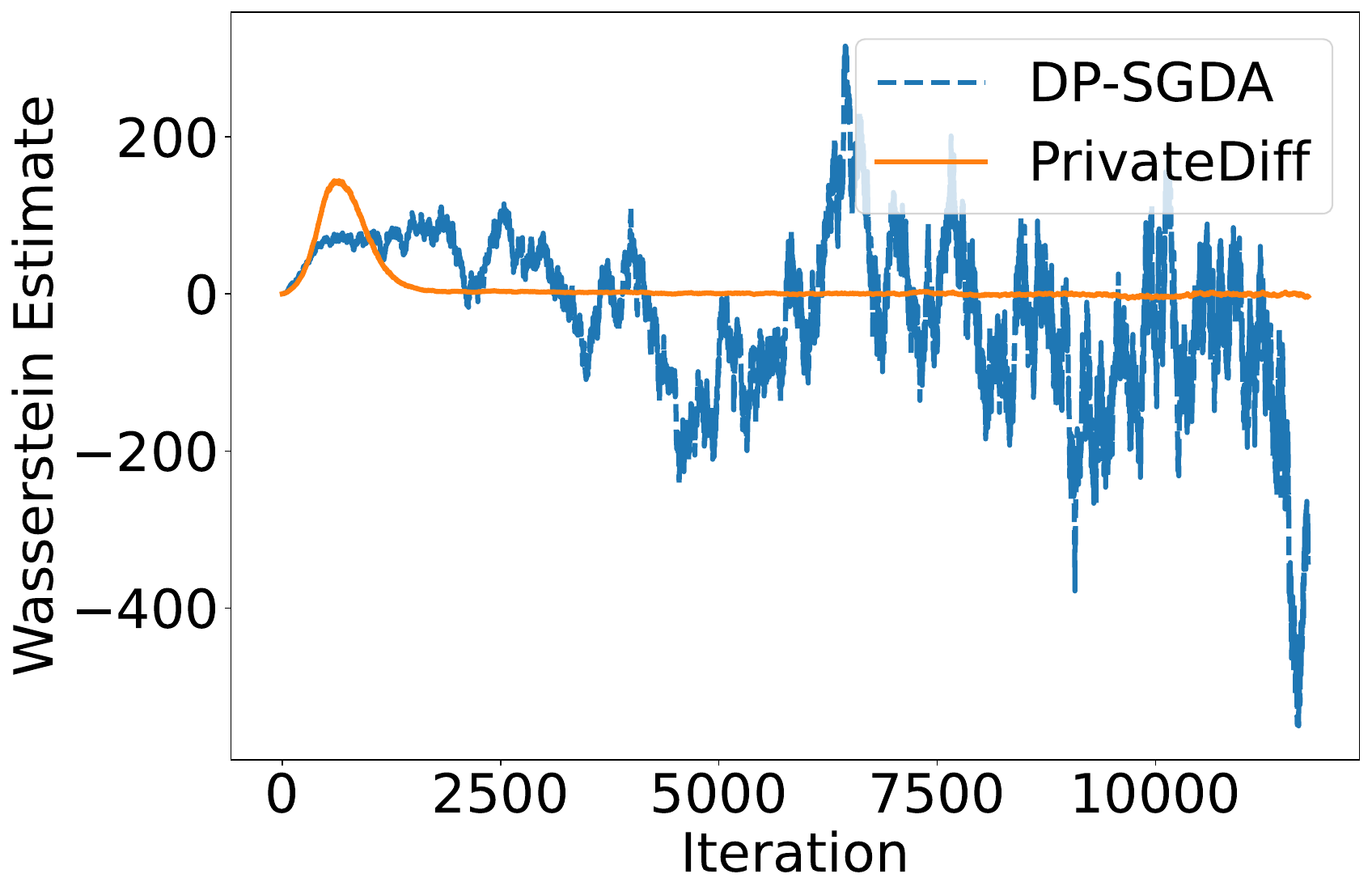}
        \caption{$\epsilon=8$}
    \end{subfigure}
    \begin{subfigure}{0.33\textwidth}
        \centering
        \includegraphics[width=\linewidth]{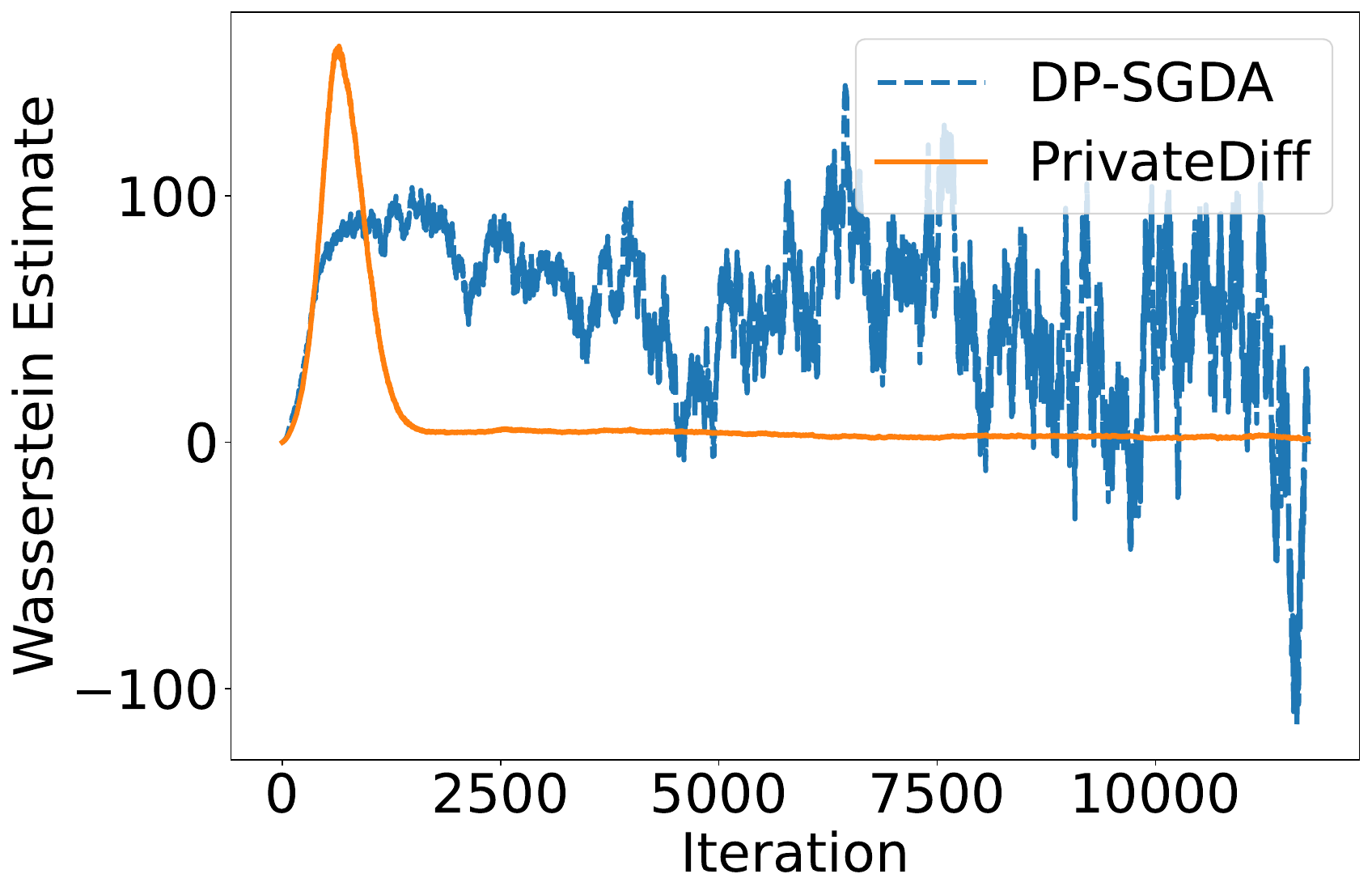}
        \caption{$\epsilon=16$}
    \end{subfigure}%
    \begin{subfigure}{0.33\textwidth}
        \centering
        \includegraphics[width=\linewidth]{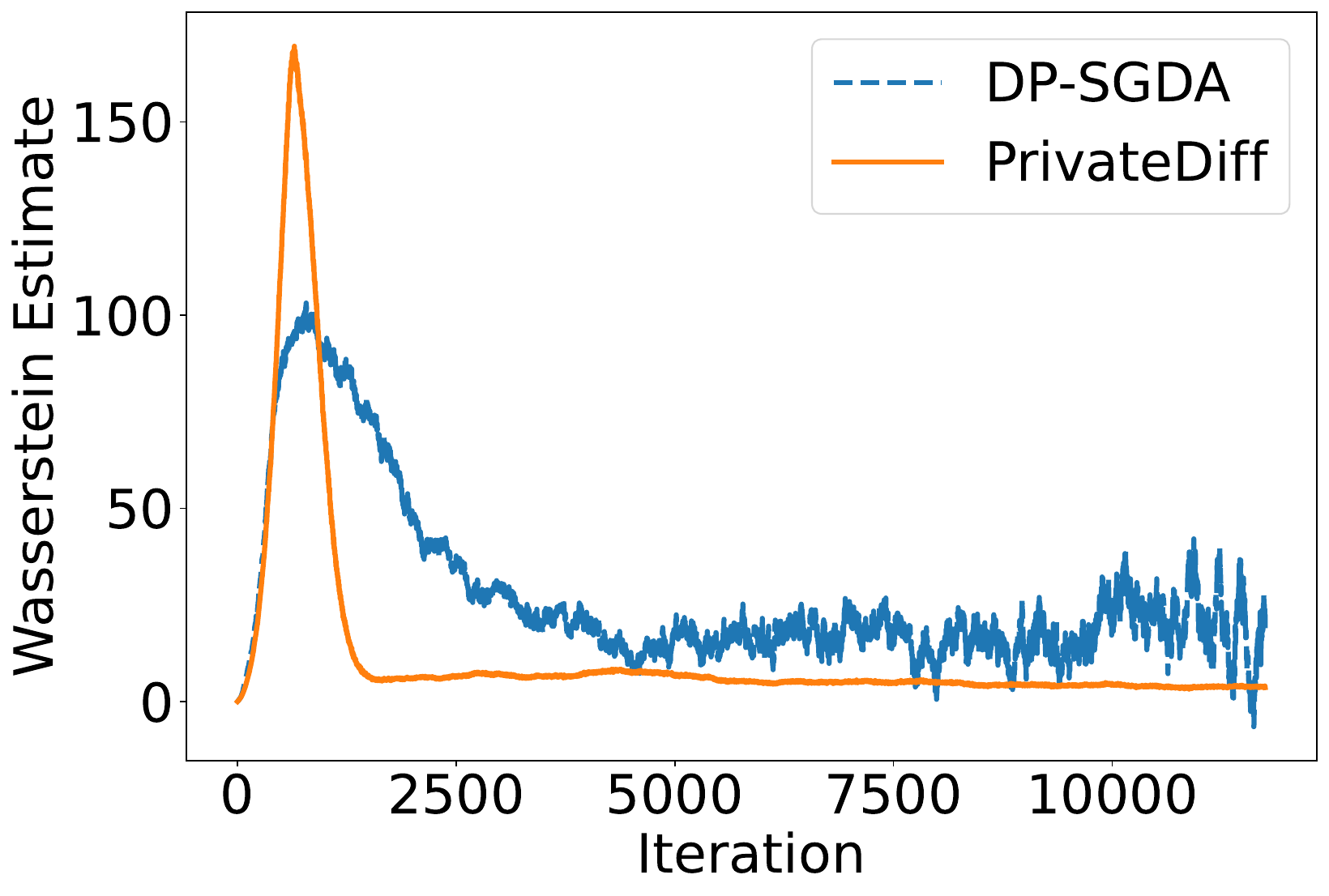}
        \caption{$\epsilon=32$}
    \end{subfigure}
    \caption{Learning curve of PrivateDiff and DP-SGDA on MNIST Dataset.}
    \label{fig:gan}
\end{figure*}

\subsection{Reinforcement Learning}
\subsubsection{Background}
Reinforcement Learning (RL) is a type of machine learning where agents learn to make decisions (policies) by interacting with an environment, aiming to maximize cumulative rewards over time. Temporal Difference (TD) Learning \citep{sutton1988learning} is a key method within RL that enhances this learning process by updating the value function, an estimation of the expected long-term reward, incrementally, after each action. The problem can be formulated as follows using Markov Decision Process (MDP).

In RL, an environment is denoted as a MDP $\mathcal{M}=(\mathcal{S}, \mathcal{A}, P, R, \gamma)$, where $\mathcal{S}$ is the state space, $\mathcal{A}$ is the action space, $P: \mathcal{S} \times \mathcal{A} \rightarrow \Delta(\mathcal{S})$ is the transition probability kernel, $R: \mathcal{S} \times \mathcal{A} \rightarrow \mathbb{R}$ is the reward function, and $\gamma \in[0,1)$ is the discount factor. The objective of TD Learning is to learn a value function $V^\pi: \mathcal{S} \rightarrow \mathbb{R}$ of given policy $\pi$, by minimizing the mean-squared Bellman error (MSBE),
\begin{equation}
\text { MSBE }=\frac{1}{2}\left\|V^\pi-R^\pi-\gamma P^\pi V^\pi\right\|^2,
\end{equation}
where $R^\pi(s)=\mathbb{E}_{a \sim \pi(\cdot \mid s)}[R(s, a)]$ is the reward function and $P^\pi\left(s, s^{\prime}\right)=\int_{\mathcal{A}} \pi(a \mid s) P\left(s^{\prime}\mid s, a\right) \mathrm{d} a$ is the reward function. It is shown that this objective is equivalent to a minimax problem by first introducing the general mean-squared projected Bellman error (MSPBE),
\begin{equation}
\text{MSPBE}=\frac{1}{2} \mathbb{E}_{\mu^\pi}\left[\delta^\pi(s) \Psi^\pi(s)^{\top}\right] G_\theta^{-1} \mathbb{E}_{\mu^\pi}\left[\delta^\pi(s) \Psi^\pi(s)\right],
\end{equation}
where $\delta^\pi(s)=R^\pi(s)+\gamma P^\pi V_\theta^\pi\left(s^{\prime}\right)-V_\theta^\pi(s)$ is the TD error,$V_\theta^\pi$ denotes the value function under policy $\pi$ parameterized by $\theta$, $\Psi^\pi(s)=\nabla_\theta V_\theta^\pi(s)$ is the gradient evaluated at state $s$, $G_\theta=$ $\mathbb{E}_{\mu^\pi}\left[\Psi^\pi(s) \Psi^\pi(s)^{\top}\right] \in \mathbb{R}^{d \times d}$, and $\mu^\pi$ is the stationary distribution over $\mathcal{S}$. The superscript $\pi$ is dropped in the following when it is clear from the context.

The MSPBE minimization problem has a primal-dual formulation with a auxiliary variable $\omega$ as
\begin{equation}
\min _{\theta \in \Theta} \text{MSPBE}(\theta)=\min _{\theta \in \Theta} \max _{\omega \in \Omega}\left\{\mathcal{L}(\theta, \omega):=\mathbb{E}_{s, a, s^{\prime}}\left[\ell\left(\theta, \omega ; s, a, s^{\prime}\right)\right]\right\},
\label{eq:rlloss}
\end{equation}
where $\ell\left(\theta, \omega ; s, a, s^{\prime}\right):=\left\langle\delta(s) \Psi(s), \omega\right\rangle-$ $\frac{1}{2} \omega^{\top}\left[\Psi(s) \Psi(s)^{\top}\right] \omega$ and $\mathbb{E}_{s, a, s^{\prime}}$ is the expectation taken over $s \sim \mu^\pi, a \sim \pi(\cdot \mid s), s^{\prime} \sim$ $P(\cdot \mid s, a)$. Our experiment optimizes the loss from Equation~\ref{eq:rlloss} using PrivateDiff.

\subsubsection{Implementation Details} 
We follow the setting in \citep{10.1145/3539597.3570470} to evaluate our method compared to DP-SGDA. The experiment includes three classical control tasks, Cart Pole, Acrobot, and Atari 2600 Pong, in OpenAI Gym \citep{towers2024gymnasium} environments. The training settings are presented in Table~\ref{tab:rlhparam}. A two-layer multilayer perceptron with one hidden layer of 50 neurons is trained to estimate the value function. The DPTD algorithm proposed in \citep{10.1145/3539597.3570470} is also included in the following for reference.

\subsubsection{Learning curve analysis}
The learning curve of all algorithms are presented in Figure~\ref{fig:rlcurve}. It is optimal that the loss value is close to zero \citep{10.1145/3539597.3570470}. Across different combinations of environments and privacy budgets, PrivateDiff consistently outperforms DP-SGDA, demonstrating greater stability and lower loss values across the board. PrivateDiff quickly converges to zero and adhere to it stably, while DP-SGDA fails and also shows shows significant fluctuations in loss. Figure~\ref{fig:rlpd} shows the impact of $\epsilon$ on PrivateDiff, demonstrating its robustness on various privacy budgets.

\begin{table*}[h]
\centering

\begin{tabular}{lcccccccc}
\toprule
                          & $C_1$ & $C_2$ & $T$ & $T_2$ & Epochs \\ 
\midrule
DP-SGDA                            & 3     & 3     & N/A & N/A   & 100     \\
DPTD                                & 3     & 3     & N/A & N/A   & 100    \\
PrivateDiff      & 3     & 3     & 2   & 3     & 100     \\

\bottomrule
\end{tabular}
\caption{Hyperparameter Settings and Training Configurations.}
\label{tab:rlhparam}
\end{table*}

\begin{figure*}[h]

\centering

\begin{subfigure}[t]{0.33\textwidth}
    \includegraphics[width=\textwidth]  
    {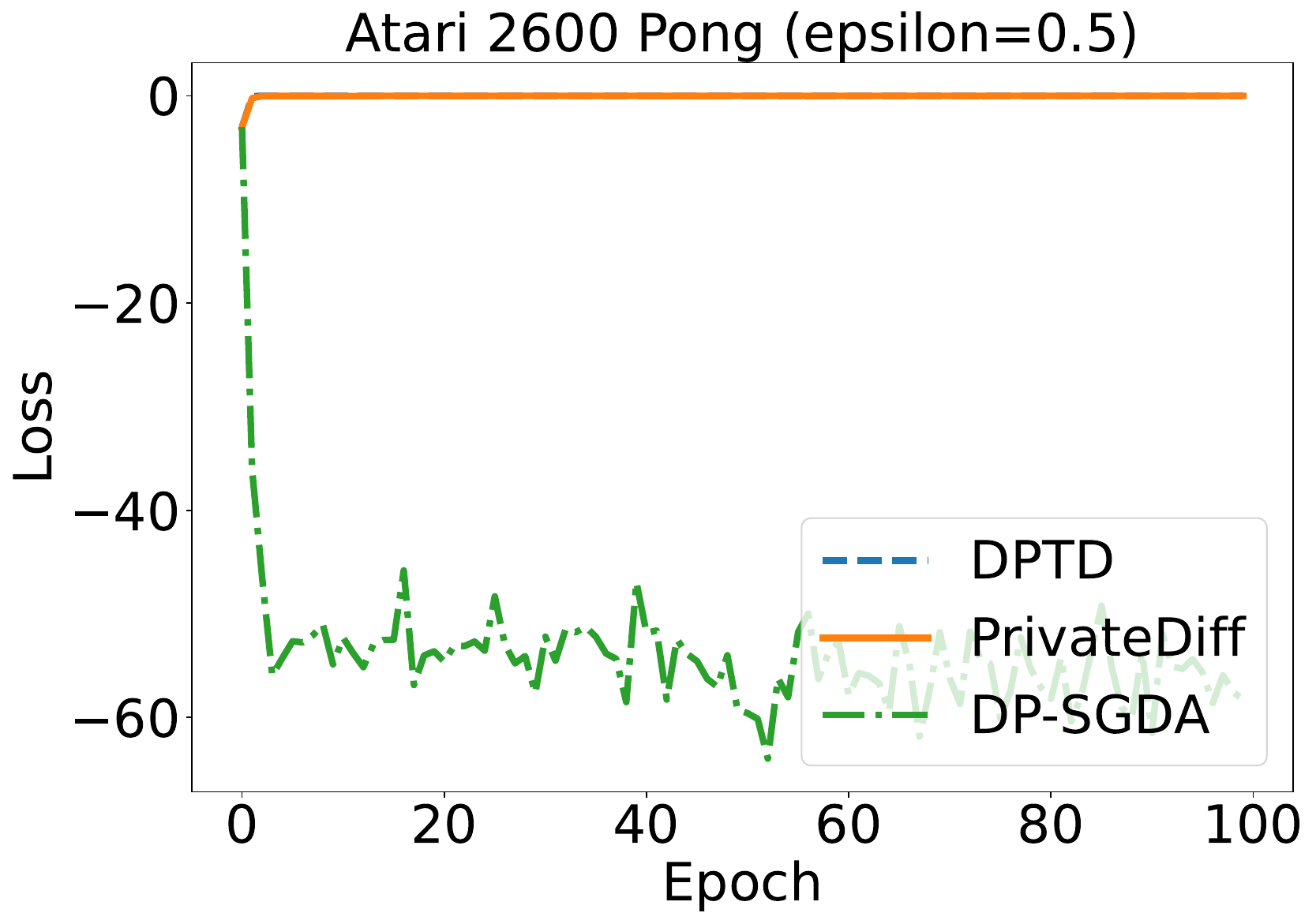}
    \includegraphics[width=\textwidth]
    {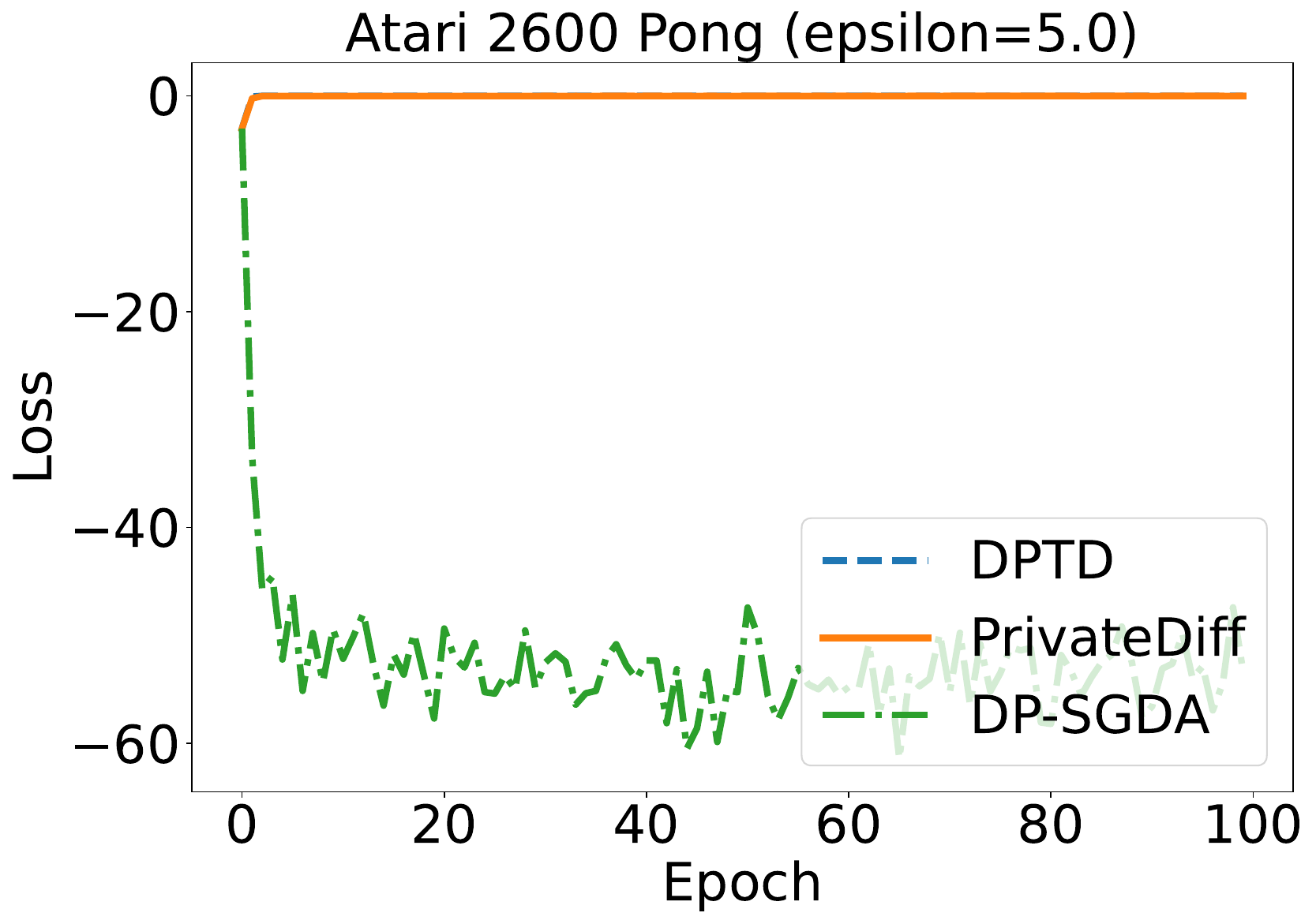}
    \includegraphics[width=\textwidth]
    {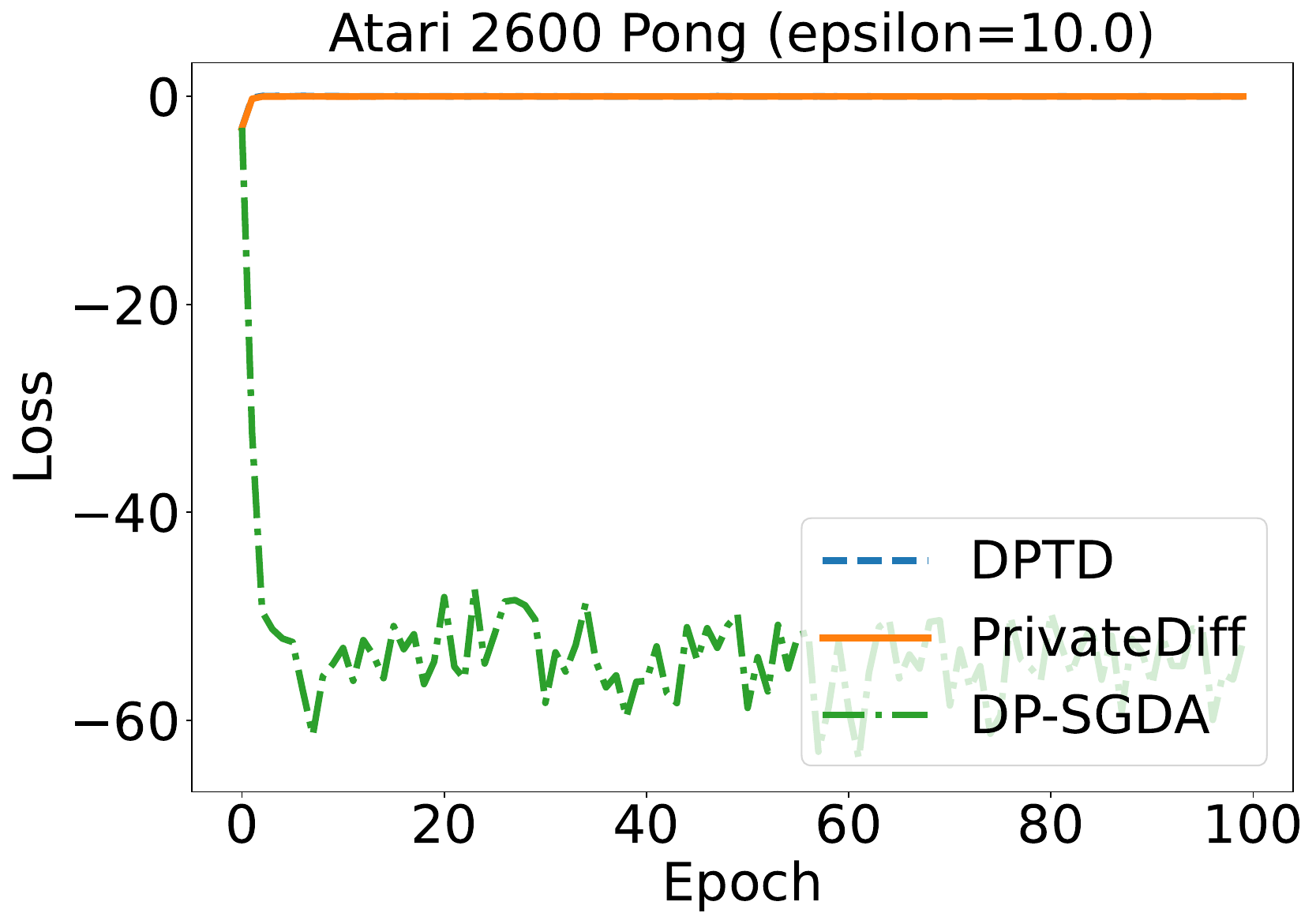}
    \caption{Atri 2600 Pong}
\end{subfigure}
\begin{subfigure}[t]{0.33\textwidth}
    \includegraphics[width=\textwidth]  
    {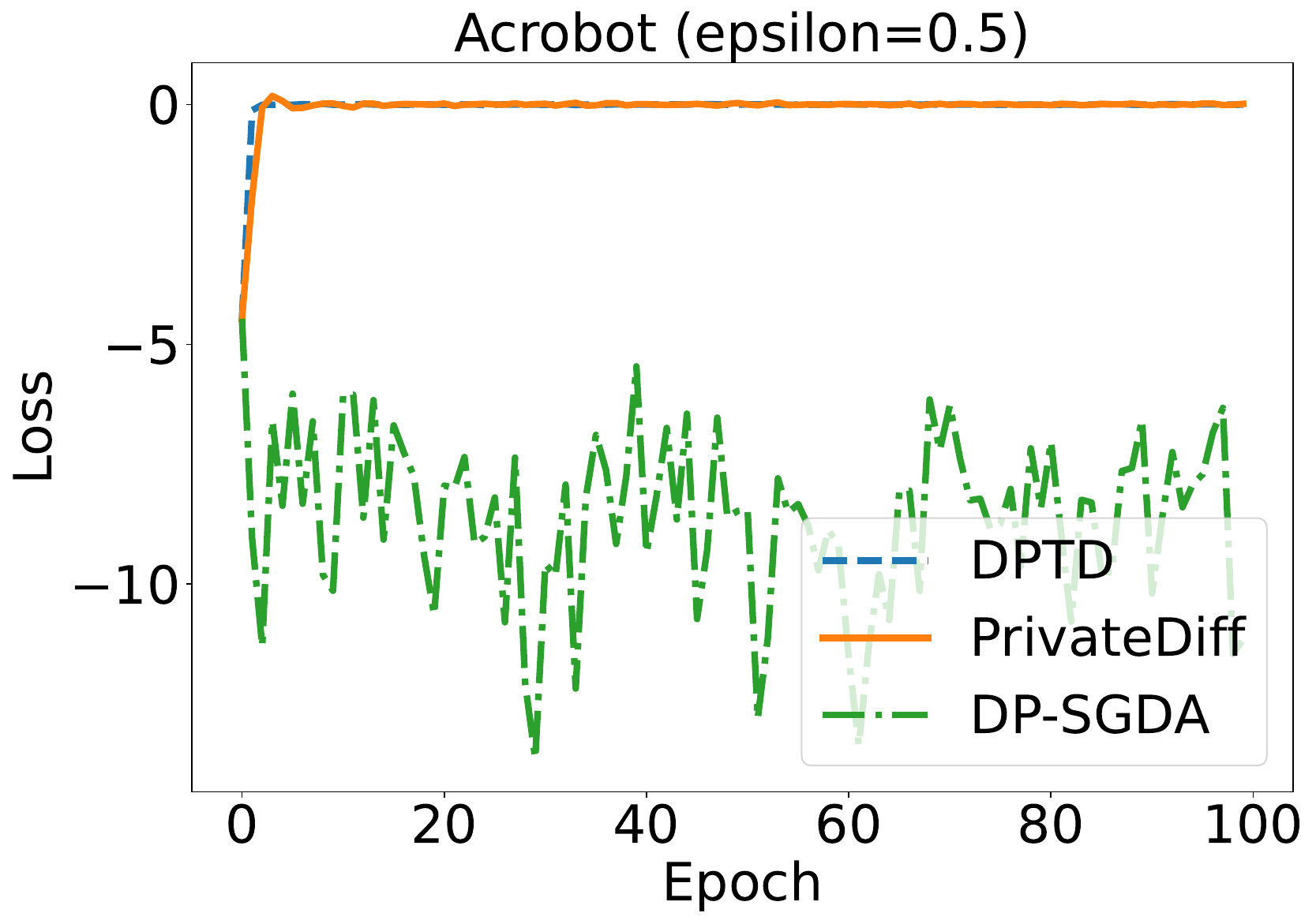}
    \includegraphics[width=\textwidth]
    {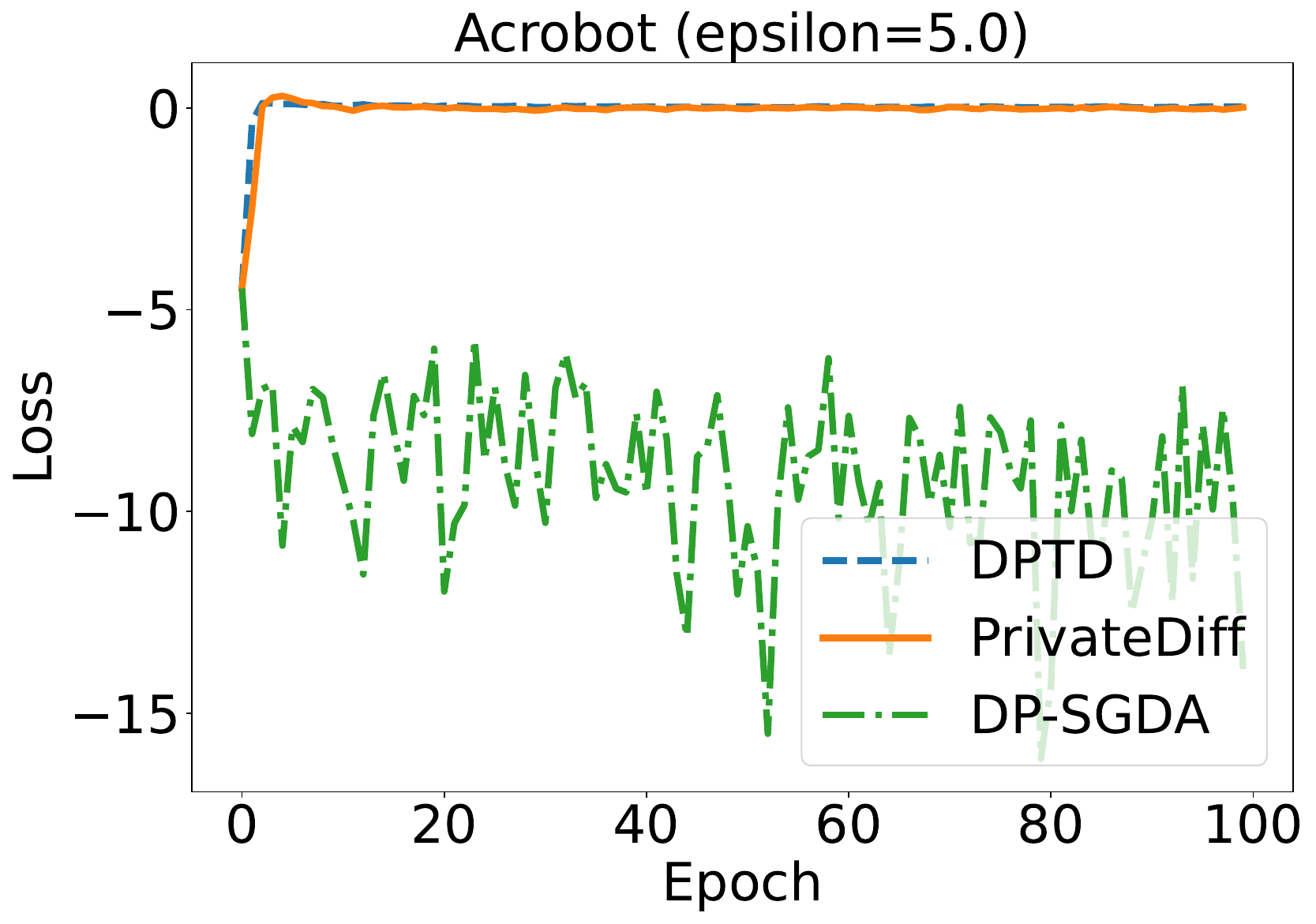}
    \includegraphics[width=\textwidth]
    {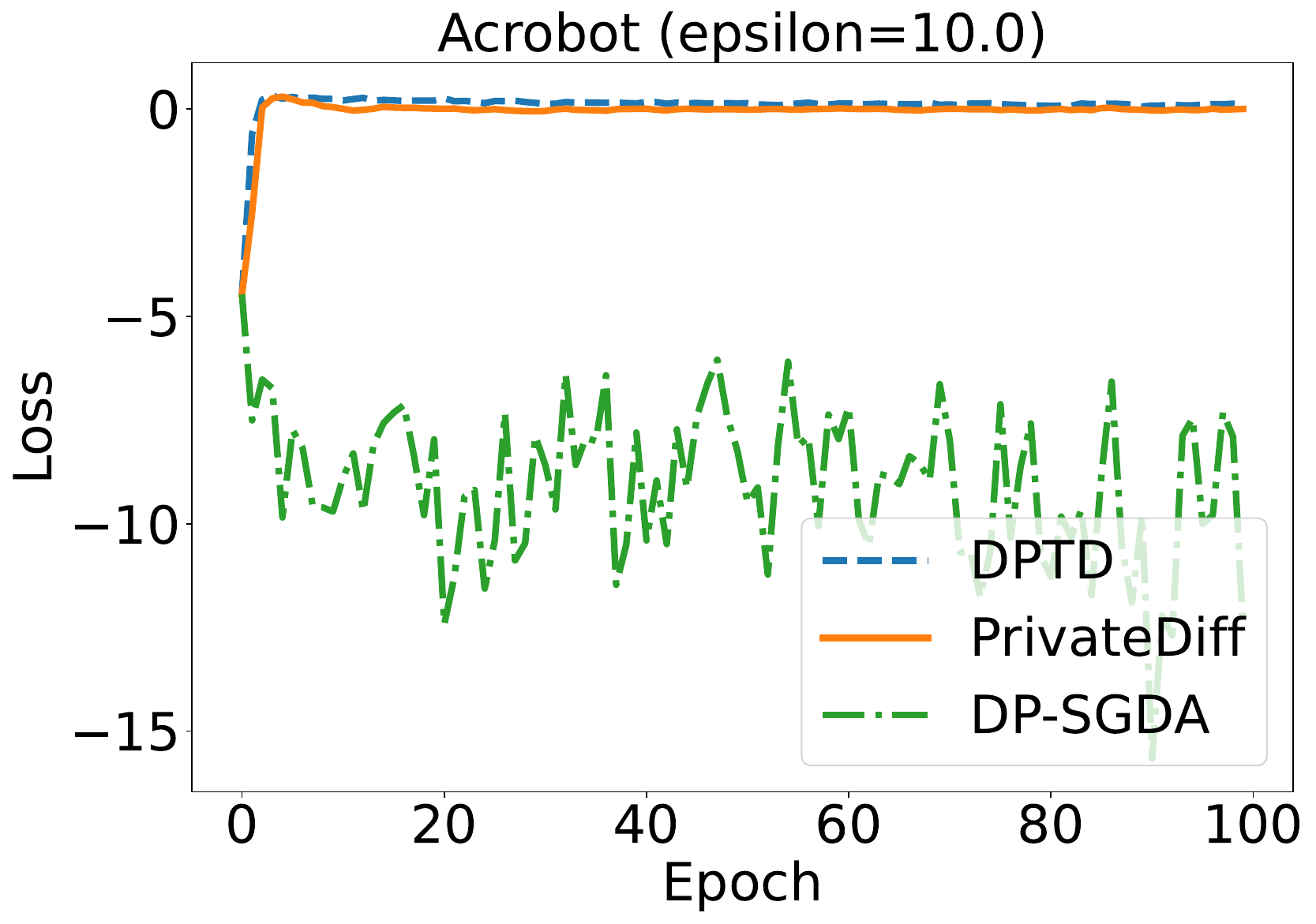}
    \caption{Acrobot}
\end{subfigure}
\begin{subfigure}[t]{0.33\textwidth}
    \includegraphics[width=\textwidth]  
    {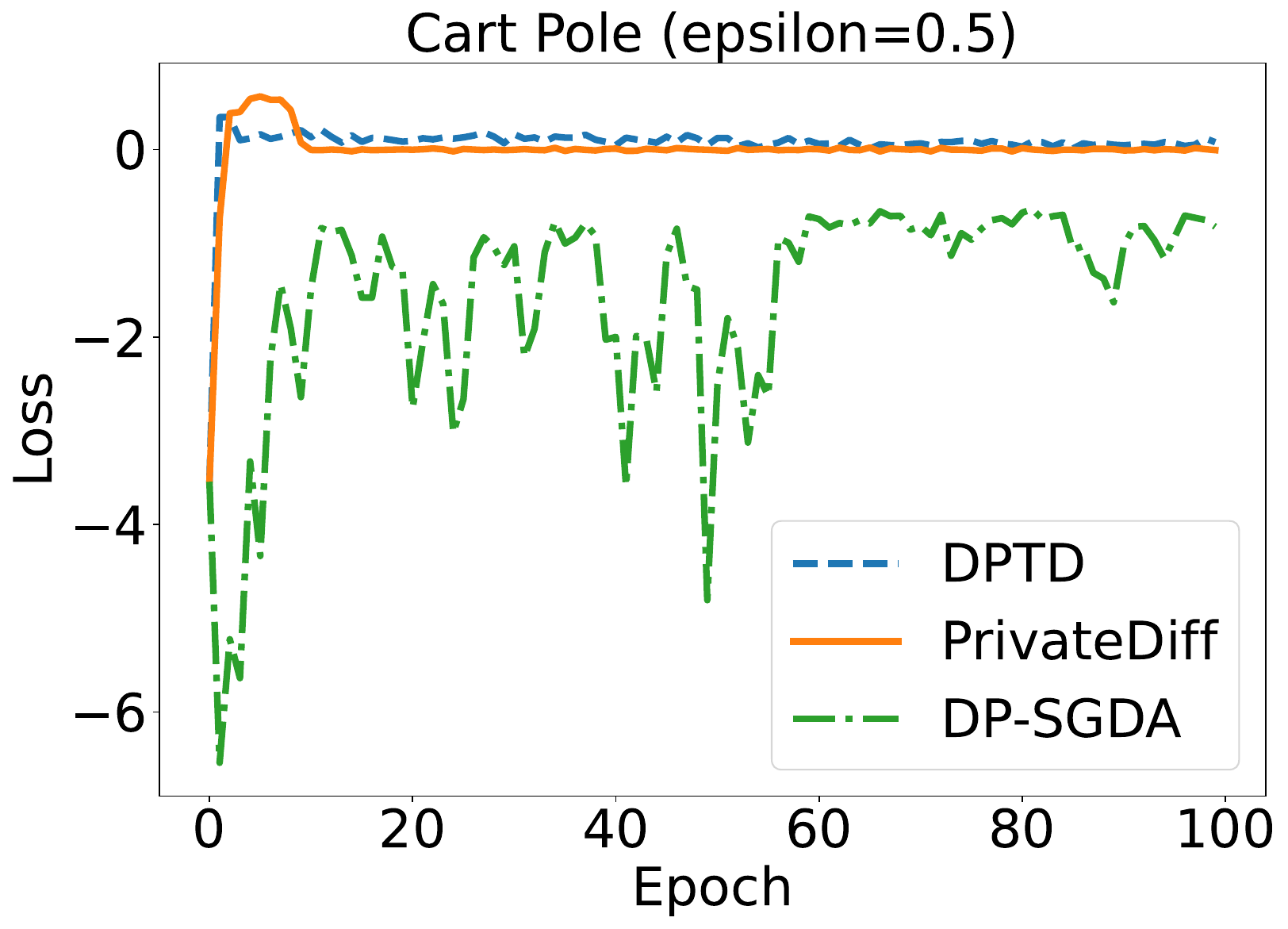}
    \includegraphics[width=\textwidth]
    {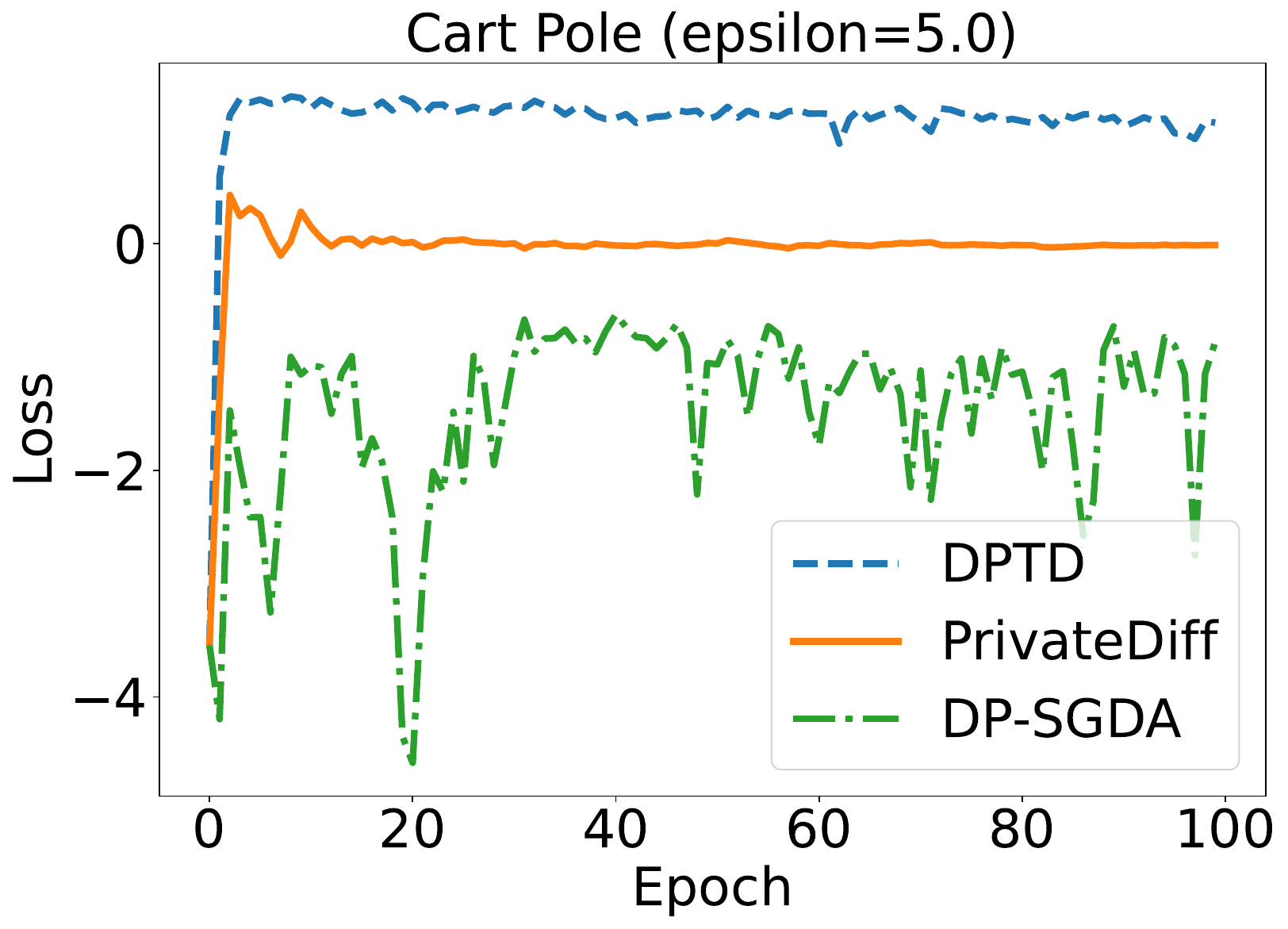}
    \includegraphics[width=\textwidth]
    {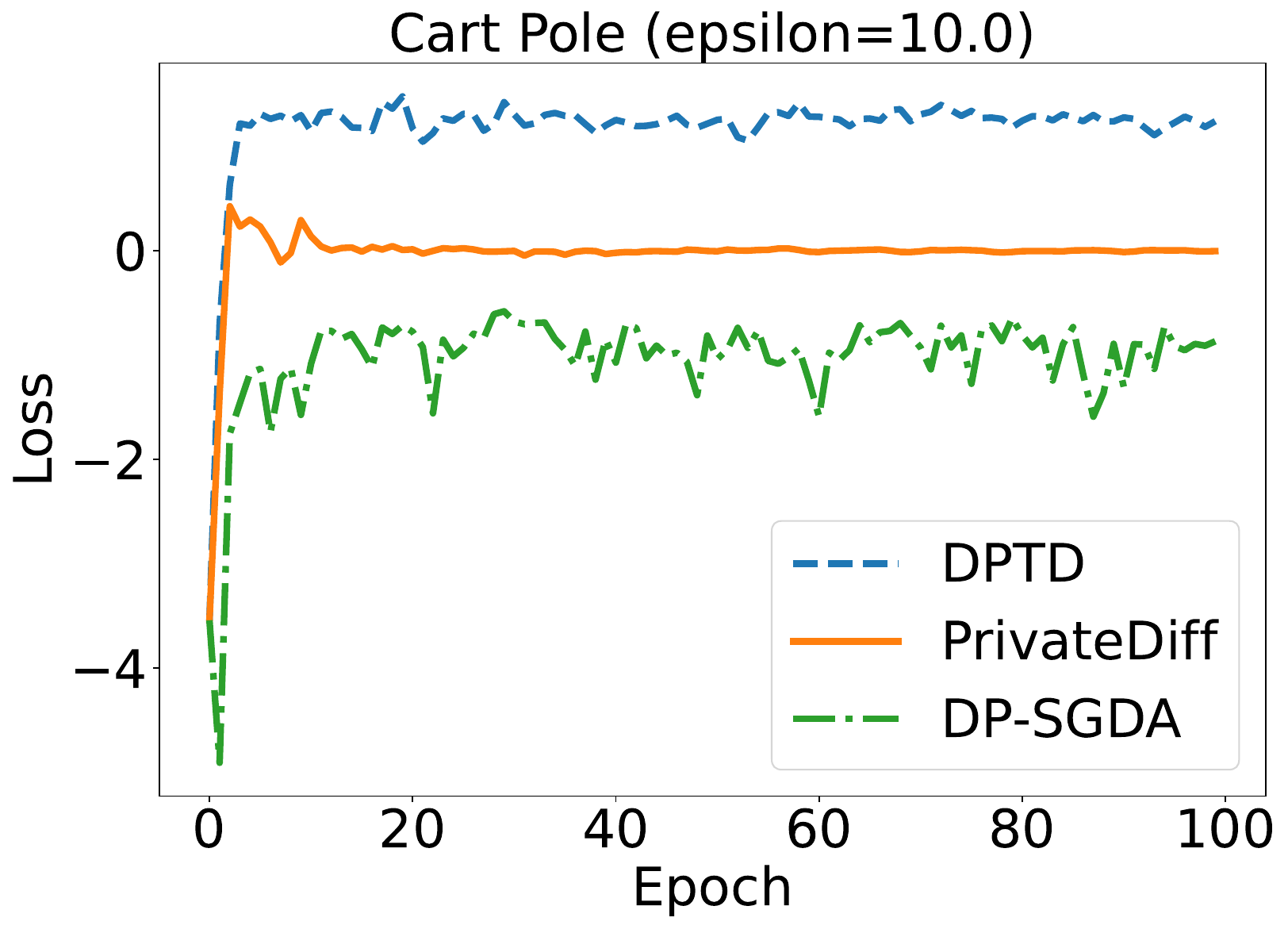}
    \caption{Cart Pole}
\end{subfigure}

\caption{Learning curve of PrivateDiff and DP-SGDA across Different Dataset.}
\label{fig:rlcurve}
\end{figure*}

\begin{figure*}[h]
    \centering
    
    \begin{subfigure}{0.33\textwidth}
        \centering
        \includegraphics[width=\linewidth]{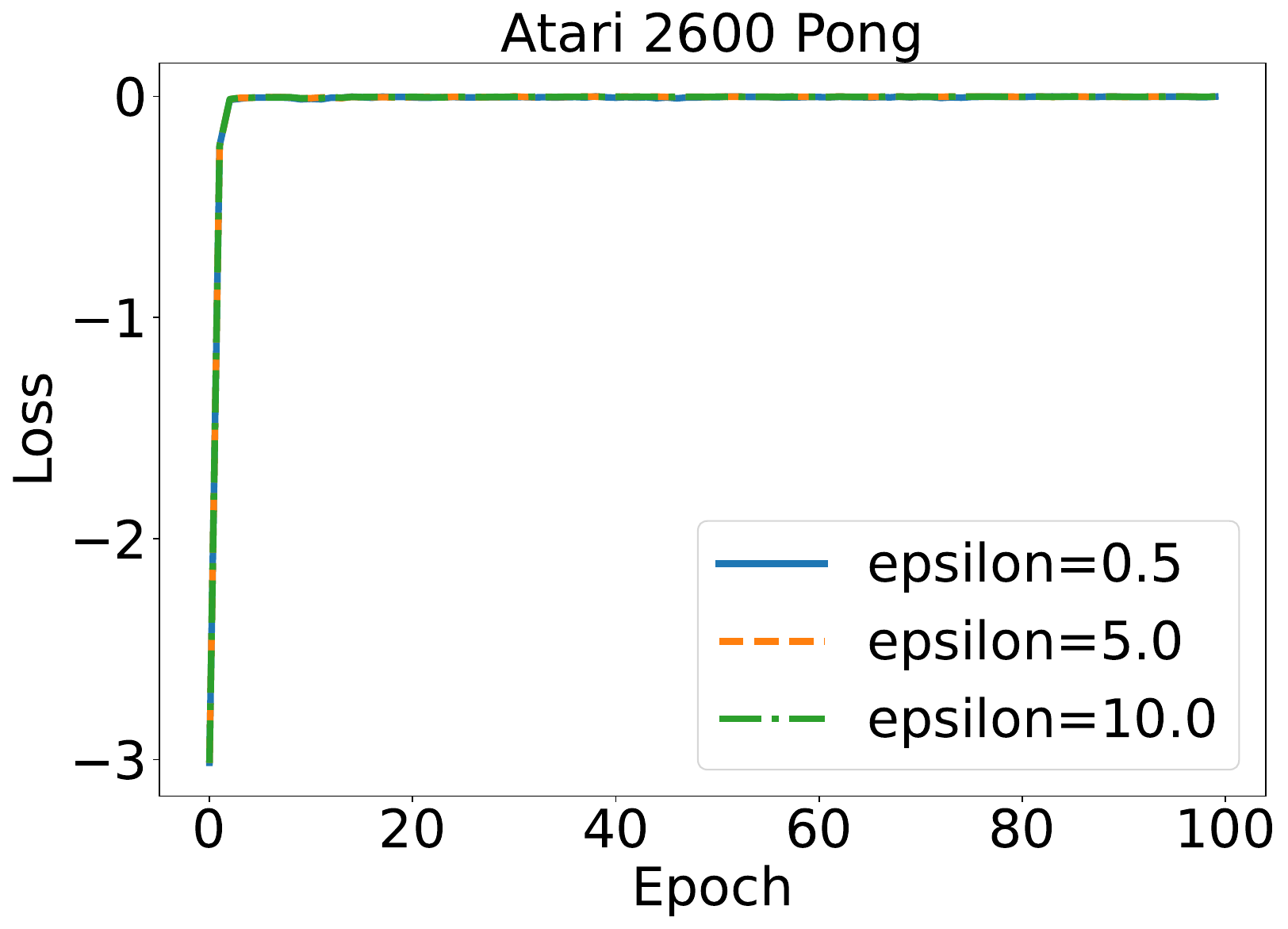}
        \caption{PrivateDiff on Atari}
        \label{fig:atar_pd}
    \end{subfigure}
    \begin{subfigure}{0.33\textwidth}
        \centering
        \includegraphics[width=\linewidth]{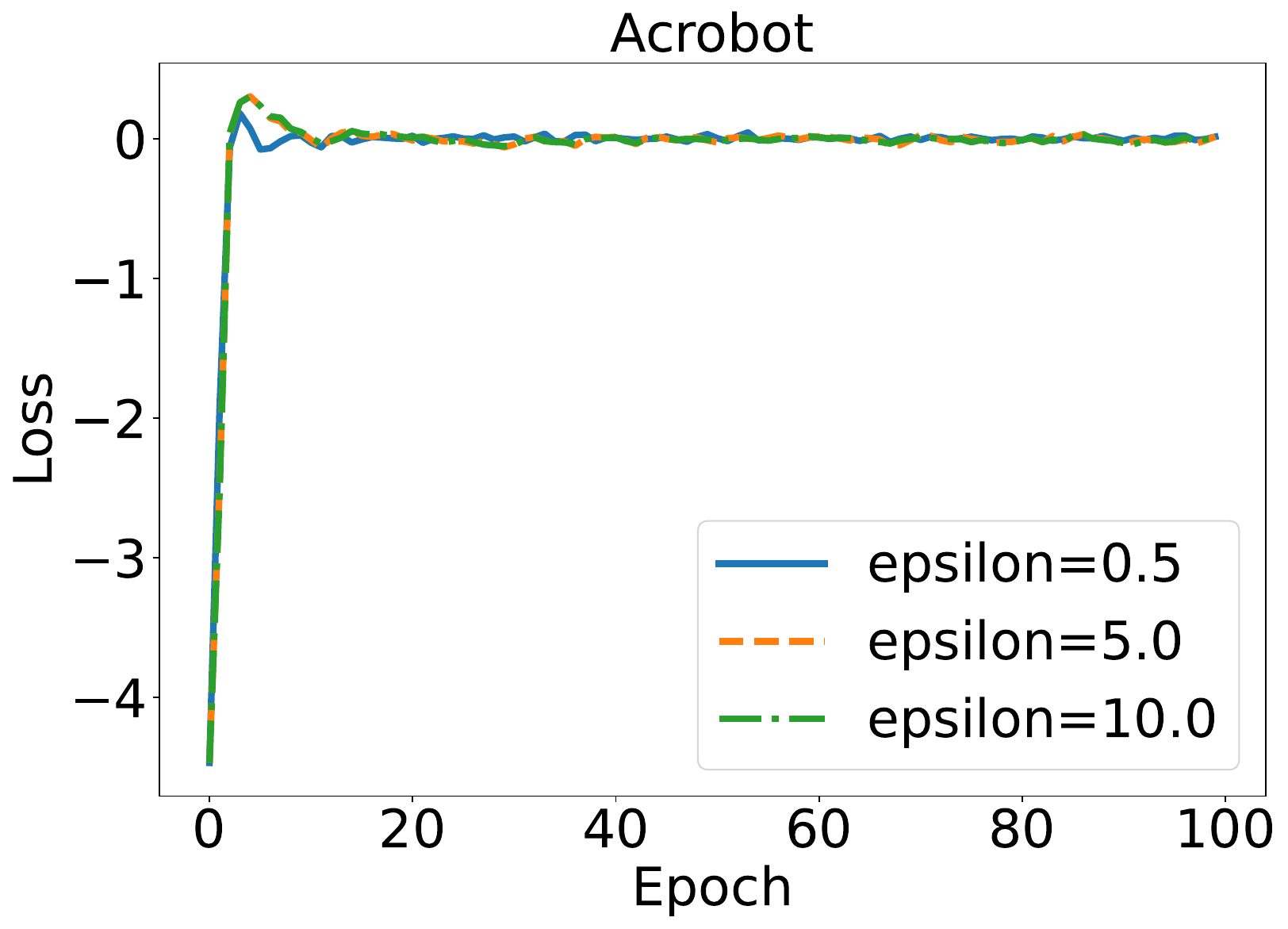}
        \caption{PrivateDiff on Acrobot}
        \label{fig:acro_pd}
    \end{subfigure}%
    \begin{subfigure}{0.33\textwidth}
        \centering
        \includegraphics[width=\linewidth]{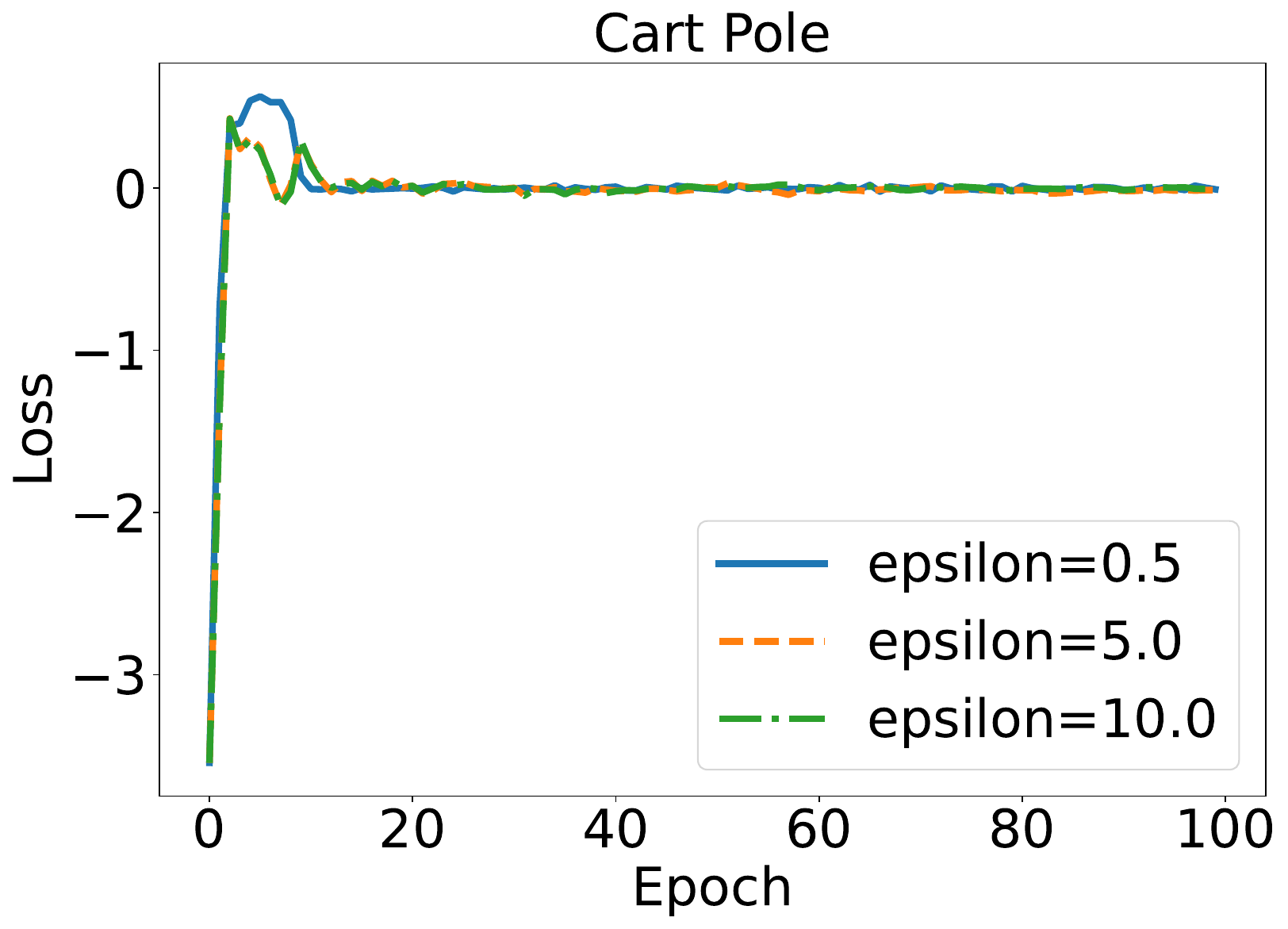}
        \caption{PrivateDiff on Cart Pole}
        \label{fig:cart_pd}
    \end{subfigure}
    \caption{The Sensitivity of Privacy Budget for PrivateDiff Algorithm across Different Dataset.}
    \label{fig:rlpd}
\end{figure*}

\end{document}